%% file: HR_v3.tex
\documentclass{article}

\usepackage{import}
\import{}{preamble.tex}

\usepackage{setspace}
\spacing{1.5}

\usepackage{authblk}

\date{}
\title{Holistic Robust Data-Driven Decisions}
\author[1]{Amine Bennouna}
\author[2]{Bart P.G.\ Van Parys}
\author[1]{Ryan Lucas}
\affil[1]{Massachusetts Institute of Technology}
\affil[2]{CWI Amsterdam}

\begin{document}

\maketitle

\begin{abstract}
  The design of data-driven formulations for machine learning and decision-making with good out-of-sample performance is a key challenge.
  The observation that good in-sample performance does not guarantee good out-of-sample performance is generally known as \textit{overfitting}.
  Practical overfitting can typically not be attributed to a single cause but is caused by several factors simultaneously.
  We consider here three overfitting sources: (i) statistical error as a result of working with finite sample data, (ii) data noise, which occurs when the data points are measured only with finite precision, and finally, (iii) data misspecification in which a small fraction of all data may be wholly corrupted.
  Although existing data-driven formulations may be robust against one of these three sources in isolation, they do not provide holistic protection against all overfitting sources simultaneously.
  We design a novel data-driven formulation that guarantees such holistic protection and is computationally viable.
  Our distributionally robust optimization formulation can be interpreted as a novel combination of a Kullback-Leibler and L\'evy-Prokhorov robust optimization formulation.
   In the context of classification and regression problems, we show that several popular regularized and robust formulations naturally reduce to a particular case of our proposed novel formulation. Finally, we apply the proposed HR formulation to two real-life applications and study it alongside several benchmarks: (1) training neural networks on healthcare data, where we analyze various robustness and generalization properties in the presence of noise, labeling errors, and scarce data, (2) a portfolio selection problem with real stock data, and analyze the risk/return tradeoff under the natural severe distribution shift of the application. 
\end{abstract} 


\section{Introduction} \label{sec:introduction}

In this paper, we study data-driven decision-making in the context of stochastic optimization.
Let $\cX$ be a set of decisions and $\txi$ a random variable representing an uncertain scenario realizing in a set
$\Sigma$. For a given scenario $\xi \in \Sigma$ of the uncertainty, and a decision $x \in \mathcal{X}$, the loss incurred for decision $x$ in scenario $\xi$ is denoted here as $\loss(x, \xi) \in \Re$. 
The random variable $\txi$ is distributed according to a probability measure $\Pb$ in the set of probability measures $\mathcal{P}$ over $\Sigma$.
A decision with minimal expected loss can be found as the solution to the stochastic optimization problem
\begin{equation}\label{eq: stochastic opt}
  \min_{x \in \mathcal{X}}\, \mathbb{E}_{\Pb}[\loss(x,\txi)].
\end{equation}
Stochastic optimization covers a wide variety of problems, such as decision-making under uncertainty and regression in machine learning.
We will assume here that the set $\mathcal X$ and $\Sigma$ are compact and the loss function $\ell$ is continuous which guarantees 
that the minimum in Equation (\ref{eq: stochastic opt}) indeed exists.

\begin{example}[Machine Learning]\label{exp: ML}
  Consider covariates $(X,Y) \in \Re^n\times \mathcal{Y}$ for a set of possible outputs $\mathcal{Y}$ and $n\in \integ$, a set of parameters $\Theta$, and a loss function $L: \Theta \times \Re^n\times \mathcal{Y} \rightarrow \Re$. Assume the data $(\tilde{X},\tilde{Y})$ follows an out-of-sample probability distribution $\Pb$. A myriad of popular machine learning problems attempt to find parameters as minimizers in
  \begin{equation*}
    \min_{\theta \in \Theta}\, \Eb_{\Pb}[L(\theta,\tilde X,\tilde Y)].
  \end{equation*}
  The objective function is denoted as the out-of-sample error of a regressor associated with the parameter $\theta\in \Theta$.
  For example, least squares linear regression corresponds to the loss $L(\theta, X,Y) = (Y - \theta^\top X)^2$ with $\theta$ the coefficient of the considered linear models.
  When training a neural network, $\theta$ corresponds to the weights of the network, and $L(\theta,X,Y)$ the loss incurred when evaluating a network with weights $\theta$ on data point $(X,Y)$.
\end{example}

\begin{example}[Newsvendor problem]
  Consider a newsvendor problem in which the decision maker decides on an inventory $x \in \cX \subset \Re_+$ in the face of uncertain upcoming demand $\tilde d \in \Re_+$ following a distribution $\Pb$. Any excess inventory leads to a cost $h>0$ per unit, and any demand that is not satisfied leads to a cost $b>0$ per unit. The newsvendor problem can be formulated as the stochastic optimization problem 
  \begin{equation*}
    \min_{x\in \cX} \, \Eb_{\Pb}[b(\tilde d-x)^++h(x-\tilde d)^+].
  \end{equation*}
\end{example}

The distribution $\Pb$ of the random variable $\txi$  is in many practical situations unknown. Rather, the decision-maker can only observe historical data points $\xi_1,\ldots,\xi_T \in \Sigma^T$. The goal is to construct a solution approximating the optimal solution of Problem \eqref{eq: stochastic opt} using only the data samples and without access to the data distribution. As we assume the order of the data points to be irrelevant here, the observed data can be compactly represented through its empirical distribution
\begin{equation}\label{eq: empirical distribution.}
  \Pemp{T} \defn \textstyle\frac{1}{T}\sum_{t\in[T]} \delta_{\xi_t},
\end{equation}
where $\delta_{\xi_t}$ is the point mass distribution at $\xi_t$.
Typically, one constructs an approximation or proxy for the unknown expectation of the loss (or the out-of-sample loss) $\Eb_{\Pb}[\loss(x,\txi)]$ for each decision $x$ as a function of the empirical distribution; say $\hat{c}(x,\Pemp{T})$. We denote such an approximation or proxy $\hat{c}$ here as a cost \textit{predictor}. The cost $\hat{c}(x,\Pemp{T})$ is then called the in-sample loss of the decision $x$.
The resulting data-driven counterpart to Problem \eqref{eq: stochastic opt} is then
\begin{equation}\label{eq: min predictor}
  \min_{x\in \cX} \, \hat{c}(x,\Pemp{T}).
\end{equation}
Perhaps the most natural such approximation involves substituting the out-of-sample loss for the empirical loss $\hat{c}(x,\Pemp{T}) = \Eb_{\Pemp{T}}[\loss(x,\txi)] = \frac{1}{T}\sum_{t\in[T]} \loss(x,\xi_t)$. This yields the so called sample average approximations (SAA) in the context of decision-making and empirical risk minimization (ERM) in the context of machine learning.

However, substitution of a naive cost predictor for the unknown out-of-sample cost often results in decisions exhibiting poor generalization or out-of-sample performance: the in-sample cost $\hat{c}(x,\Pemp{T})$ may considerably underestimate the actual out-of-sample cost $\Eb_{\Pb}[\loss(x,\txi)]$ and lead to poor decisions \citep{smith2006optimizer}. This phenomenon is broadly known as \textit{overfitting}. Constructing predictors which safeguard against overfitting is one of the fundamental challenges of learning and decision-making with data.

\subsection{Robustness}
To safeguard against overfitting, \textit{robust} predictors resulting in reliable decisions are desirable. A robust predictor $\hat{c}$ is here a predictor which verifies an out-of-sample guarantee of the form 
\begin{equation}\label{eq: OOS Guarantee}
\hat{c}(x,\hat{\Pb}_T) \geq \Eb_{\Pb}[\loss(x,\txi)]
\end{equation}
with ``high probability'' on the data generation, for all $x\in \cX$. This property guarantees that the \textit{in-sample} cost ($\hat{c}(x,\hat{\Pb}_T)$), estimated with data, is an upper bound on the \textit{out-of-sample} cost ($\Eb_{\Pb}[\loss(x,\txi)]$) with high probability. Hence, the predictor provides a conservative|rather than an overly optimistic|estimation of the out-of-sample cost. In other words, the out-of-sample cost is at least as good as what was estimated in-sample. As the predictor is intended to be used for optimization and solving the stochastic optimization \eqref{eq: stochastic opt}, then the guarantee \eqref{eq: OOS Guarantee} must also be verified in its prescribed optimal solution $\hat{x}^\star \in \argmin \hat{c}(x,\hat{\Pb}_T)$. This is the case when the guarantee holds uniform in the decision, that is, 
$
\hat{c}(x,\hat{\Pb}_T) \geq \Eb_{\Pb}[\loss(x,\txi)]
$, $\forall x \in \cX$ with high probability. Uniform guarantees ensure that subsequently minimizing the estimated loss indeed results in minimizing an upper bound on the out-of-sample loss. 

It is easy to construct overly conservative upper bounds on the out-of-sample cost. Indeed, a trivial predictor predicting large or even infinite cost would trivially verify the guarantee \eqref{eq: OOS Guarantee}. However, such overly conservative predictors would likely lead to poor decisions. Hence, merely verifying an out-of-sample guarantee does not ensure good decisions. We will seek therefore \textit{efficient} predictors which constitute the \textit{smallest} upper bound on the out-of-sample loss|in a sense we will make more precise.

\subsection{Sources of Overfitting}
\label{sec:sources-overfitting}

The first important step toward constructing formulations which are robust against overfitting and generalize well is to understand precisely what we seek to robustify against (``\textit{robustness to what?}''). In other words, what are the potential sources of overfitting which can cause data-driven formulations to generalize poorly? We discuss here what we believe to be the three most common sources which can cause overfitting and challenge generalization. We will point out later that most existing robust predictors in the literature protect against some but not (at least not efficiently)
all sources simultaneously. 

\subsubsection{Statistical Error}
Perhaps the most prevalent and fundamental source of overfitting comes from the fact that as the data set is finite, our knowledge of the true data distribution is necessarily limited. The cost prediction $\hat{c}(x,\Pemp{T})$ is therefore inevitably only an approximation of the true out-of-sample loss $\Eb_{\Pb}[\loss(x,\txi)]$ and fluctuates with the randomness of the finite sample.
When the randomness of the sample---and therefore of the empirical distribution $\Pemp{T}$---causes the prediction $\hat{c}(x,\Pemp{T})$ to underestimate the out-of-sample loss, the minimization problem \eqref{eq: min predictor} might indeed  suggest seemingly good solutions $x$ in-sample, i.e., with low predicted cost $\hat{c}(x,\Pemp{T})$, while performing poorly out-of-sample. For example, one source of statistical error in classification tasks is data imbalance. Even when the true out-of-sample distribution of the data is balanced, by the random nature of the sample one class may count considerably more data points than the others. This results in classifiers based on ERM to mistakenly reduce the in-sample classification error on this majority class at the expense of the others.

There is considerable work on how to design robust formulations against overfitting due to statistical error \citep{bertsimas2018data, bertsimas2018robust,lam2019recovering}.
One particular example of such robust predictors is the Kullback-Leibler divergence distributionally robust optimization predictor (KL-DRO)
\begin{equation}\label{eq: KL predictor}
  \hat{c}^r_{\mathrm{KL}}(x,\Pemp{T}) \defn
  \max\{\Eb_{\Pb'}[\loss(x,\txi)] \; : \; \Pb'\in \cP, \; \KL(\Pemp{T} || \Pb') \leq r \} \quad \forall x \in \cX,
\end{equation}
for $r>0$, where
$
\KL(\mu || \nu) 
\defn 
\int \log  \left(\tfrac{\d\mu}{\d\nu}(\xi)\right) \, \d\mu(\xi)
$ for all $\mu\ll\nu \in \cP$, and $+\infty$ otherwise,
is the Kullback–Leibler divergence. Informally, when the observed samples are independent realizations of the out-of-sample distribution $\Pb$, then $\Pb \in \{\Pb'\in \cP\; : \; \KL(\Pemp{T} || \Pb') \leq r\}$ with high probability\footnote{This is only true when the distributions are of finite support. In the general case however, it still holds that $\Eb_{\Pb}[\loss(x,\txi)] \leq \hat{c}_{\mathrm{KL}}(x,\Pemp{T})$ with high probability \citep{van2021data}.}. Hence, the KL-DRO predictor  minimizes a worst-case expectation over a set of distributions which includes the out-of-sample distribution $\Pb$ with high probability and hence is unlikely to underestimate the out-of-sample cost. \citet{van2021data} prove that the KL-DRO predictor efficiently guards against overfitting caused by statistical error and in fact guarantees that the probability that the out-of-sample cost is underestimated decays exponentially fast in the number of collected data points $T$. Furthermore, the KL-DRO predictor can be evaluated as a convex minimization problem \citep{love2015phi,van2021data}.

\citet{lam2019recovering, duchi2021statistics} finally show that when this probability is desired to decay subexponentially, 
robustness against statistical error can be obtained by considering the sample variance penalization predictor (SVP)
\begin{equation*}\label{eq: SVP predictor}
  \hat{c}_{\mathrm{SVP}}(x,\Pemp{T}) \defn
  \Eb_{\Pemp{T}}[\loss(x,\txi)] + \lambda \sqrt{\Var_{\Pemp{T}}[\loss(x,\txi)]} \quad \forall x \in \cX,
\end{equation*}
for $\lambda\geq 0$, where $\Var_{\Pemp{T}}[\loss(x,\txi)]$ is the empirical variance of the loss.
In fact, \citet{van2021data, bennouna2021learning} prove in a precise sense that the KL-DRO and SVP predictors are efficient predictors for robustness against statistical error. In some sense, both predictors optimally balance efficiency while guaranteeing a certain level of robustness against statistical error.

However, overfitting in practice is not caused by statistical error alone.
In fact, statistical error is perhaps the most mild source of overfitting among those we will discuss as it diminishes with the samples size at rate $\mathcal O(\tfrac{1}{\sqrt{T}})$. That is, for any decision $x$ and bounded and continuous loss function $\xi\mapsto\ell(x, \xi)$ we have\footnote{Note that this bound can be extended to $\min_{x\in\cX}\Eb_{\Pemp{T}}[\loss(x,\txi)] - \min_{x\in \cX}\Eb_{\Pb}[\loss(x,\txi)]$ under careful complexity assumptions on the decision set $\cX$, such as finite VC dimension \citep{vapnik1999nature,maurer2009empirical}.}
\[
  \Eb_{\Pemp{T}}[\loss(x,\txi)] - \Eb_{\Pb}[\loss(x,\txi)] = \mathcal O\left(\sqrt{\tfrac{\Var[\loss(x,\txi)]}{T}}\right)
\]
when the data points are independently sampled from the out-of-sample distribution $\Pb$.
Hence, statistical error is in general more prominent in settings with scarce data or high variance of the loss.
We next discuss two sources of overfitting which require robust predictors even in a large data regime where the amount of statistical error is negligible.

\subsubsection{Noise}
\label{sec:noise}

In practice the observed data sample $\{\xi_1, \dots, \xi_T\}$ may often not be a realization of $\txi$ following the out-of-sample distribution $\Pb$, but rather a realization of a noisy random variable $\txi + \Tilde{n}\in\Sigma$.
As the distribution of the noise $\Tilde{n}$ is not known, collecting more training data does not eliminate this potential source of overfitting.

Clearly, when no assumptions are made regarding the size of the noise $\Tilde{n}$ then it may completely swamp the signal $\txi$ and nothing interesting can be said further. Assume however that the noise realizes in a known bounded set $0\in\cN$. An example of such set is the epsilon norm ball $\cB(0,\epsilon) = \{\noise \; : \; \|\noise\| \leq \epsilon \}$, which models norm bounded noise up to size $\epsilon\geq 0$. The cost predictors and their associate decisions may be guarded against such adversarial noise by considering an inflated loss function 
\(
\loss^{\cN}:\cX\times\Sigma\to \Re, ~(x,\xi) \mapsto \max\set{\loss(x,\xi-\noise)}{\noise \in \cN, \, \xi-\noise\in \Sigma}
\)
rather than the loss function $\ell$ directly.
Recent works by \citet{madry2017towards}
indeed consider robust predictors of the form
\begin{align}\label{eq: nosie robust predictor}
  \hat{c}_{\mathrm{R}}(x,\Pemp{T}) \defn & \max\set{\textstyle\frac 1T\sum_{t\in[T]}\loss(x,\xi_t -\noise_t)}{\noise_t \in \cN,~\xi_t-\noise_t\in \Sigma ~~\forall t\in [T]}
  =  \textstyle\frac{1}{T}\sum_{t\in [T]} \loss^{\cN}(x,\xi_t).
\end{align}
The underlying intuition is that when the statistical error is negligible (i.e., the out-of-sample distribution is the empirical distribution of the noiseless sample), we are guaranteed that
$\Pb \in \{\sum_{t\in [T]} \delta_{\xi_t-\noise_t}/T : \; \noise_t\in \cN,~\xi_t-\noise_t\in \Sigma ~~ \forall t\in [T]\}$.
We remark hence that the robust predictor differs from the ERM predictor only in that an inflated loss function $\ell^{\cN}$ is considered instead of the loss function $\ell$ itself.
The robust formulation is practical to the extent we have access to an oracle with which we can (approximately) evaluate the inflated loss function $\ell^{\cN}$. Sometimes this inflated loss can be computed explicitly for commonly considered loss functions and norm bounded noise as we will illustrate in Section \ref{sec: application}. In some settings evaluating the inflated loss function $\ell^{\cN}$ can be hard as discussed by \citet{sinha2017certifying}. Efficient approximation methods however may still be available as discussed by \citet{madry2017towards,bertsimas2021robust}.
In the remainder of the paper we will assume that an oracle is available which can evaluate the inflated loss function $\ell^{\cN}$.

Perhaps the most well known instance of the robust approach discussed here is the popular LASSO predictor for linear regression (see Example \ref{exp: ML}) which can be written as
\[
\textstyle \sqrt{\frac{1}{T}\sum_{t\in [T]} L(\theta,Y_t,X_t)} + \frac{\lambda}{\sqrt{T}} \|\theta\|_1  
=
   \max_{}  \set{\sqrt{\frac{1}{T}\sum_{t\in[T]} L(\theta,X_t-n_t,Y_t)}}{\|\noise_t\|_2 \leq \lambda ~~\forall t\in [T]},
\]
where $L(\theta,X,Y) = (Y-\theta^\top X)^2$,
for all $\theta \in \Theta$ and $\lambda \geq 0$ \citep[Theorem 1]{xu2009robustness}.
Hence, in this perspective, LASSO can be interpreted as to protect precisely against particular norm bounded noise.
 
It is important to note that protection against one form of overfitting does not necessarily translate to protection against all others. As an extreme case, consider the KL-DRO formulation \eqref{eq: KL predictor} which as we pointed out in the previous section protects efficiently against statistical error. Its associated ambiguity set will in the presence of any amount of noise fail to contain the distribution $\Pb$ when it is a continuous distribution; see also Figure \ref{fig:distributional-shift}. Hence, the KL-DRO predictor can severely underestimate the out-of-sample cost under the presence of noise, even though it verifies out-of-sample guarantees in the clean data setting. This somewhat surprising observation perhaps helps explain why the Kullback-Leibler divergence is not as popular in practice in favor of other divergence metrics such as the Wasserstein distance (with norm transportation cost)
\begin{equation}\label{eq: Wasserstein}
   W(\mu,\nu) = \inf_{\gamma \in \Gamma(\mu,\nu)} \int \norm{\xi-\xi'} \,\d\gamma(\xi,\xi') \quad \forall \mu,\nu \in \cP,
\end{equation}
for a given norm $\|\cdot\|$
and where the set $\Gamma(\mu,\nu)$ denotes the collection of all joint measures on $\Sigma\times \Sigma$
with marginals $\mu$ and $\nu$ \citep{ambrosio2003lecture} (which we also call the set of couplings between $\mu$ and $\nu$). However, the W-DRO robust predictor, defined here as
\(
  \hat{c}_{\mathrm{W}}(x,\Pemp{T}) \defn \sup\{\Eb_{\Pb'}[\loss(x,\txi)] \; : \; \Pb'\in \cP, \;  W(\Pemp{T} || \Pb') \leq \epsilon \}
\)
for $\epsilon>0$ and all $x\in \mathcal X$, protects only against certain types of data noise, rather than offering holistic protection against statistical error and noise simultaneously. Indeed,
\citet{gao2016distributionally, wang2022mean} show that when the event set $\Sigma$ is convex and the loss function $\ell$ is concave in its second argument, then
\[
    \hat{c}_{\mathrm{W}}(x,\Pemp{T}) \defn \max\set{\textstyle\frac{1}{T}\sum_{t\in[T]} \loss(x,\xi_t - \noise_t)}{\textstyle \frac 1T\sum_{t\in[T]}\norm{\noise_t} \leq \epsilon, ~\xi_t-n_t\in \Sigma} \quad \forall x \in \cX, \; \forall t\in [T].
  \]
Hence, W-DRO protects here precisely against all noise which averaged over the samples remains bounded in norm. 
Nevertheless, W-DRO does verify out-of-sample statistical guarantees \citep{esfahani2015data}, and therefore also provides protection again statistical error when the radius is chosen sufficiently large. However, its statistical guarantees are a byproduct of its noise protection rather than an explicit consequence of its design. Indeed, it has been observed empirically in deep learning that while type-$\infty$ W-DRO (adversarial training) provides robustness to noisy inputs, it suffers a worse overfitting gap than ERM \citep{rice2020overfitting}. The previous observations indicate that W-DRO does not \textit{efficiently} protect against statistical error, as its ambiguity set yielding such statistical guarantees can be exceedingly large.

\begin{figure}[t]
  \centering
\includegraphics[width=0.5\textwidth]{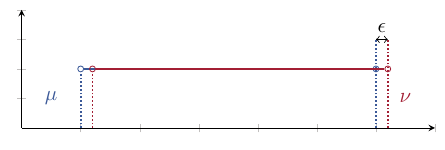}
  \caption{Two distributions $\mu$ and $\nu$ which are equivalent up to a small distributional shift $\epsilon>0$ are very dissimilar in terms of their entropic divergence, i.e., $\KL(\mu, \nu)=\infty$.}
  \label{fig:distributional-shift}
\end{figure}

The main characteristic of the noise considered here is that, in some sense, it remains small as it realizes in a compact set $\cN$.
The last source of overfitting we identify will neither be small nor diminish with increasing sample size.

\subsubsection{Misspecification}
\label{sssec:misspecification}

In most typical data sets, a hopefully small fraction of all data may be wholly corrupted.
This can be the case when errors in the data collection occurred or when fake data made its way into the data set.
We remark that this source of overfitting is quite distinct from the previously discussed overfitting source.
In the previous setting, the noise $n\in \cN$ remains bounded and consequently all samples carry at least some amount of information concerning the out-of-sample distribution. Here, however, we will assume that a small fraction $\alpha \ll 1$ of all data points $\xi_t\in \Sigma$ for $t\in [T]$ may carry no information at all.
This specific cause of overfitting has been studied since the pioneering work of \citet{tukey1958bias} and \cite{huber1981robust} in the context of robust statistics.
More recent work by \citet{diakonikolas2019robust} has renewed interest in this source of overfitting  in particular in the context of high-dimensional learning problems \citep{diakonikolas2019recent} and adversarial machine learning \citep{goldblum2022dataset}. In particular, in the context of adversarial machine learning the misspecification described here is also known as \textit{data poisoning}, with a specific instance being label flipping \citep{zhu_transferable_2019}.

In this context, the robust predictor for $x \in \cX$
\begin{equation}\label{eq: predictor misspecification}
  \hat{c}_{\mathrm{M}}(x,\Pemp{T}) \defn \max \set{\textstyle\frac{1}{T}\sum_{t\in[T]} \loss(x,\xi_t-\noise'_t)}{\sum_{t\in[T]}\one{\noise'_t\neq 0}/T\leq \alpha,~\xi_t-n'_t\in \Sigma}
\end{equation}
seems appropriate. The underlying intuition is that as it is assumed here that there is no statistical error (informally we again assume here that $\hat \Pb_T$ becomes the out-of-sample distribution of the corrupted samples) we are guaranteed that
$\Pb \in \tset{\sum_{t\in[T]} \delta_{\xi_t-\noise'_t}/T}{\sum_{t\in[T]}\one{\noise'_t\neq 0}/T\leq \alpha,~\xi_t-n_t\in \Sigma~~\forall t\in [T]}$. Therefore, we have
\(
  \Eb_{\Pb}[\loss(\hat{x}_{M}(\Pemp{T}),\txi)] \leq \hat{c}(\hat{x}_{M}(\Pemp{T}),\Pemp{T}))
\)
where $\hat{x}_{M}(\Pemp{T}) \in \argmin_{x\in \cX} \hat{c}_{M}(x,\Pemp{T})$, and this upper bound can be shown to be tight.
Hence, the predicted cost of the optimal decision $\hat{x}_{M}(\Pemp{T})$ exceeds its unknown out-of-sample cost and hence no overfitting has taken place.



Each of these three discussed sources of overfitting may have a certain varying importance in practice. For example, noise will be the main source of overfitting in settings where abundant but highly noisy data is available. Statistical error will come to dominate in settings where scarce but high quality data is available.
It is therefore instrumental to adapt the amount of protection offered against each of these three distinct sources to the appropriate level in a particular application.
Existing robust formulations typically protect specifically against only one of these aspects.
As already pointed out before, KL-DRO and SVP for instance protects efficiently against statistical error while LASSO and W-DRO protects efficiently against noise.
An important question that arises is therefore
\begin{center}
  \textit{Can we construct a predictor robust simultaneously against statistical error, noise and misspecification?}
\end{center}
This is precisely the objective of this paper. We seek to construct a \textit{holistic} robust predictor which protects against all sources of overfitting to any desirable degree.
In addition to (i) robustness, we seek robust formulations which (ii) are tractable (iii) and are efficient in that they are not overly conservative.

\subsection{Contributions}

We study in this paper how to construct a robust predictor tailored precisely to these three sources of overfitting, and their simultaneous and combined effect on data-driven decisions. We do so through a DRO framework, and seek to design a tailored ambiguity set that captures exactly the shift in the data caused by statistical error, noise and corruption simultaneously.

As a first building block towards our holistic robust predictor, we introduce the convex pseudo divergence metric
$\LP_{\cN}(\Pemp{T},\Pb') = \inf_{\gamma \in \Gamma(\Pemp{T},\Pb')} \int \indic(\xi-\xi' \not\in\cN) \,\d\gamma(\xi,\xi')$, for a given noise set $\cN$ and corruption fraction $\alpha \in [0,1]$.
The pseudo divergence $\LP_{\cN}$ generalizes the well-studied L\'evy-Prokhorov (LP) metric on probability distributions \citep{prokhorov1956convergence}; as we detail in Section \ref{sec:rel-DRO}. We show that the LP metric captures \textit{precisely} the distribution shift due to the combination of noise and corruption. Formally, given the observed empirical distribution of the corrupted samples $\Pemp{T}$, the ambiguity set $\{\Pb'\in \cP, \, \LP_{\cN}(\Pemp{T},\Pb') \leq \alpha\}$ captures exactly the set of possible out-of-sample distributions when the noise realizes in the set $\cN$ and less than $\alpha T$ data points are corrupted. 

We study the resulting LP-DRO formulation using this new ambiguity set, which provides efficient robustness to noise and corruption, and prove it admits an exact reformulation as combination of the $\CVaR$ and worse case loss. In particular, we point out that this result provides a first exact tractable reformulation for the robust optimization problems associated with the L\'evy-Prokhorov metric studied by \citet{erdougan2006ambiguous}, a special case of our LP-DRO predictor.

Building on these results, we then tackle the problem of protecting against statistical error on top of noise and corruption. The celebrated Sanov's theorem in large deviation theory \citep[Theorem 6.2.10]{dembod1996large} indicates that the KL divergence captures precisely the shift between an empirical distribution and its true generating distribution. Hence, we design a novel ambiguity set by inflating the LP ball by the KL divergence resulting in our novel holistic robust DRO predictor ($\HR$)
\begin{align*}
  \hat{c}^{\cN,\alpha,r}_{\HR}(x,\Pemp{T})\defn 
  \max \{ 
  \Eb_{\Pb'}[\loss(x,\txi)]
  & \; : \; 
    \Pb'\in \cP,\;\Qb' \in \cP, \;
    \LP_{\cN}(\Pemp{T},\Qb') \leq \alpha, \;
    \KL (\Qb'||\Pb') \leq r
    \},
\end{align*}
with parameters $\cN$, $\alpha \in [0,1]$ and $r \geq 0$.
These parameters set the desired robustness offered by our predictor against each source of overfitting separately.
We prove that when noise realizes in a set $\cN$ and less than a fraction $\alpha$ of data points are corrupted, the $\HR$ predictor provides uniform robustness out-of-sample guarantees: with probability larger than $1-e^{-rT + O(1)}$, $\hat{c}^{\cN,\alpha,r}_{\HR}(x,\Pemp{T}) \geq \Eb_{\Pb}[\loss(x,\txi)]$ \textit{uniformly} in $x \in \cX$, where $\Pemp{T}$ is the empirical distribution of the \textit{corrupted} samples (Theorem \ref{thm: HD robustness guarantee}). In particular, such uniform guarantees do not depend on the complexity of the decision set $\cX$, a crucial novel property in high-dimensional contexts. We further show that $\HR$ provides \textit{efficient} robustness: any other predictor that verifies such robustness out-of-sample guarantee is necessarily uniformly more conservative than the $\HR$ predictor (Theorem \ref{thm:efficiency-hd}). In particular, this proves theoretically that the new tailored ambiguity set matches exactly the shift caused by the considered sources of overfitting, and as a result, provides the exact amount of required conservatism. 

Computing the HR predictor is a challenging problem, as typical in DRO formulations, involving optimization over continuous distributions. 
We prove that for any bounded continuous loss function, the supremum in the HR-DRO formulation is attained in distributions of finite \textit{known} support. Subsequently, the HR predictor can be computed as a finite convex optimization problem with only conic and linear constraints (Theorem \ref{thm: HD finite formulation}). Moreover, using a dual formation of the predictor as a minimization problem (Theorem \ref{thm: dual rep HD}), we show that the predictor can be optimized efficiently and the problem $\min_{x\in \cX} \hat{c}^{\cN,\alpha,r}_{\HR}(x,\Pemp{T})$
is as hard as optimizing the inflated loss function $\min_{x\in \cX} \Eb_{\hat{\Pb}_T}[\loss^{\cN}(x,\txi)]$. This is in contrast with some popular DRO approaches, where tractability typically requires somewhat stringent conditions on the loss function \citep{wang2021sinkhorn}.

In Section \ref{sec: application}, we investigate the practicality of these insights with extensive experiments in real-life applications, beyond the theoretical setting. In particular, we chose challenging applications which violate, to various degrees, our assumptions on the nature of the three sources of overfitting. We also measure robustness through application-driven metrics. 

We experiment with HR in three distinct applications: (i) training neural networks for a healthcare classification task, (ii) selecting an investment portfolio using real stock data, and (iii) performing linear regression and classification with synthetic data. For each application, we extensively evaluate HR's robustness, benchmark its performance against classical methods, and analyze the influence of individual parameters on robustness outcomes. Additionally, we study the effectiveness of validation in selecting practically reasonable values for the noise set $\cN$, corruption fraction $\alpha$ and radius $r$.

\paragraph{Notations} 
Random variables are denoted with a tilde in this paper. For example, $\txi$ denotes the random variable with realizations $\xi \in \Sigma$.
We denote with $\mathcal B(\Sigma)$ the set of all Borel measurable subsets of the topological space $\Sigma$. $A^\into$ (or $int(A)$) and $\bar{A}$ (or $cl(A)$) denote respectively the interior and the closure of set $A$.
For a given measure $\mu \in \cP$, we denote its support by
$\supp(\mu) \defn \bigcap \{ A = \bar{A} \in \cB(\Sigma) \; : \; \mu(A^c) =0 \}
=
\{ \xi \in \Sigma \; : \; \exists U \in \cB(\Sigma), \; \xi \in U^\into, \; \mu(U^\into)>0 \}$ for all measure $\mu$.
We abuse the notation of measures by denoting $\mu(\xi) \defn \mu(\{\xi\})$ for all measure $\mu$ and $\xi \in \Sigma$.
For all $K \in \integ$, we denote $[K] = \{1,\ldots,K\}$.

\section{Robustness Against Noise and Misspecification}\label{sec: corruption}

Our focus in this section is robustness against noise and misspecification only. The underlying assumption is that the number $T$ of available data points $\{\xi_1,\dots, \xi_T\}$ is sufficiently large so that the effect of statistical error is negligible. This assumption, made by a plethora of work on constructing data-driven formulations robust against noise \citep{madry2017towards}, will be relaxed in Section \ref{sec:HR-robustness} where we account for statistical error on top of corruption.
In what follows, we propose a novel distributionally robust predictor which we prove protects efficiently against noise and misspecification. 
This predictor will then serve as a first building block towards holistic robustness.

\subsection{The LP-DRO predictor}
To construct our robust predictor, we seek to build an ambiguity set capturing precisely the considered overfitting sources.
We consider a setting of noise in a given set $\cN$ in full generality. For intuition purposes, the reader can consider the special case of $\cN$ being a norm ball for example, which is typically used in practice.
We associate with the set $\cN$ the convex pseudo divergence metric $\LP_{\cN}:\cP\times\cP\to \Re_+$ defined as
\begin{equation}\label{eq: LP pseudo-metric}
  \LP_{\cN}(\mu,\nu) \defn \inf_{\gamma \in \Gamma(\mu,\nu)} \int \indic(\xi-\xi' \not\in\cN) \,\d\gamma(\xi,\xi') \quad \forall \mu,\,\nu \in \cP.
\end{equation}
The considered metric is hence a generalized optimal transport metric associated with a particular transport cost function which is only sensitive to whether or not the transport distance remains bounded in the set $\cN$ or not. We consider the associated LP-DRO predictor
\begin{equation}\label{eq: LP-DRO predictor}
  \hat{c}^{\cN,\alpha}_{\LPtxt}(x,\Pemp{T})\defn 
  \max \{ 
  \Eb_{\Pb'}[\loss(x,\txi)]
  \; : \; 
  \Pb'\in \cP, \; 
  \LP_{\cN}(\Pemp{T},\Pb') \leq \alpha
  \}
  \quad \forall x\in \cX,
\end{equation}
with parameters $\cN$ and $\alpha \in [0,1]$. As visualized in Figure \ref{fig:LP-DRO} the LP-DRO predictor takes the supremum over all distributions close in probability.

\begin{figure}
\vspace{0.3em}
  \centering
  \begin{tikzpicture}
    \node (image) at (0,0) {
      \includegraphics[height=7cm]{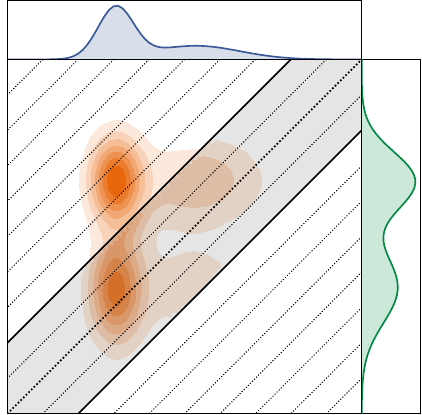}
    };
    \node at (1, 1.25) {$\cN$};
    \draw[latex-latex] (1.35, 2.5) -- (1.35, 0.1);
    \node[rotate=-90] at (2.9, 0.3) {{\color{green}$\mu$}};
    \node at (-1.6, 3) {{\color{blue}$\nu$}};
    \node at (-2.4, -0.3) {{\color{orange}$\gamma$}};
  \end{tikzpicture}
  \caption{Two distributions $\mu$ and $\nu$ satisfy the constraint $\LP_{\cN}(\mu,\nu)\leq \alpha$ if there is a coupling $\gamma$ so that at most $\alpha$ mass of the coupling is assigned to the outside of the cylinder strip associated with the noise set $\cN$.
  }
  \label{fig:LP-DRO}
\end{figure}

Intuitively, the ball $\{\Pb'\in \cP \; : \; \LP_{\cN}(\Pemp{T},\Pb') \leq \alpha\}$ contains all the distributions that can be obtained from shifting $\Pemp{T}$ by subtracting noise bounded in the set $\cN$ and misspecifying a random arbitrary fraction $\alpha$ of all data points. In fact, the distribution $\gamma$ in Equation \eqref{eq: LP pseudo-metric} can be interpreted as a coupling moving the distribution $\mu$ into the distribution $\nu$ ($\hat{\Pb}_T$ into $\Pb'$ in \eqref{eq: LP-DRO predictor}) and the objective function as the probability that a sample has been moved beyond the set $\cN$. Hence, $\LP_{\cN}(\Pemp{T},\Pb') \leq \alpha$ ensures that there exists a coupling that moves $\hat{\Pb}_T$ into $\Pb'$ by perturbing less than a fraction $\alpha$  of the data by beyond the set $\cN$. This is the misspecified part of the data. The remaining $1-\alpha$ fraction (or more) is perturbed by noise bounded in the set $\cN$. This is the noisy part of the data. 
The LP ambiguity set is therefore by construction expected to capture precisely our considered data alteration|noise and misspecification.
  We formalize this in the following theorem. Proofs of this section are deferred to Appendix \ref{App: proof of sec LP}.

Let $\txi$ be distributed as $\Pb$, the true out-of-sample distribution. Define the corrupted random observation as
$\txi^c = (\txi + \tn)\indic(\tc=0) + \txi_0 \indic(\tc=1) \in \Sigma$, where $\tn$ is a noise random variable, $\tc$ is a binary random variable indicating whether the sample is corrupted or not, and $\txi_0$ is a corrupted observation following an arbitrary distribution. Note that $\txi$, $\tn$, $\tc$ and $\txi_0$ can all be correlated and adversarial---we in particular do not impose any independence structure.
\begin{theorem}[Corruption set characterization]\label{thm: LP-DRO robustness}
Assume the noise $\tn$ realizes in the set $\cN$ and the misspecification probability is less than $\alpha$, i.e., $\Prob(\tilde c=1) < \alpha$.
For a given distribution $\Pb^c\in\cP$ of the corrupted sample $\txi^c\in \Sigma$, the set of possible out-of-sample distributions $\Pb$ of $\txi$ is 
$
  \set{\Pb'\in \cP}{\LP_{\cN}(\Pb^c, \Pb') < \alpha}.
$
\end{theorem}

\begin{corollary}[Robustness against noise and misspecification]\label{cor: LP OOS guarantee}
    Assume the noise $\tn$ realizes in the set $\cN$ and the misspecification probability is less than $\alpha$, i.e., $\Prob(\tilde c=1) < \alpha$. Let $\Pb$ be the out-of-sample distribution of clean observations and $\Pb^c$ the distribution of the corrupted observations $\txi$. We have
    $
    \hat{c}^{\cN,\alpha}_{\LPtxt}(x,\Pb^c) \geq \Eb_{\Pb}[\loss(x,\txi)], \; \forall x \in \cX.
    $
\end{corollary}

Corollary \ref{cor: LP OOS guarantee} ensures that the LP-DRO predictor provides a robustness guarantee against noise and misspecification, whenever the noise realizes in $\cN$ and the misspecification is at frequency up to $\alpha$. The following theorem shows that the LP-DRO predictor provides efficient robustness: any other ``robust'' predictor is necessarily more conservative.

\begin{theorem}[Efficient robustness]\label{thm: LP optimality}
    Any predictor $\hat{c}$ that verifies the guarantee of Corollary \ref{cor: LP OOS guarantee} for all out-of-sample distribution $\Pb$ and noise limited to $\cN$ and less than a fraction $\alpha$ of data points misspecified, 
    is necessarily uniformally more conservative than the LP-DRO predictor, that is
    $
    \hat{c}(x,\Qb) \geq \hat{c}^{\cN,\alpha}_{\LPtxt}(x,\Qb), \; \forall \Qb \in \cP, \; \forall x \in \cX.
    $
\end{theorem}


The LP-DRO predictor characterized in Equation \eqref{eq: LP-DRO predictor} involves maximizing over probability measures which may in generally be computational hard.
The following result indicates however that the supremum defining our LP-DRO predictor can be computed almost in closed form.

\begin{theorem}\label{thm: LP-DRO expression for empirical}
  Let $x\in \cX$, $\alpha \in [0,1]$ and a noise set $\cN$ be given. Denote with $\{\xi_1,\ldots,\xi_K\}$ the support of $\Pemp{T}$, and consider an ordering such that $\loss^\cN(x, \xi_{[1]})\leq \dots \leq \loss^\cN(x, \xi_{[K]})$. Introduce for all $k\in [K]$ the auxiliary points $\xi'_k \in \xi_k-\argmax_{\noise \in \cN,\, \xi_k-\noise} \loss(x,\xi_k - \noise)$ and $\xi_\infty \in \argmax_{\xi \in \Sigma} \loss(x,\xi)$.
  We have
  \begin{align*}
    \hat{c}^{\cN,\alpha}_{\LPtxt}(x,\Pemp{T})
    = &~ (1-\alpha)\CVaR_{\Pemp{T}}^\alpha(\loss^{\cN}(x,\txi))+\alpha\max_{\xi\in \Sigma}\loss(x,\xi)
    =  \Eb_{\Pemp{T}^{\cN,\alpha}}[\loss(x,\txi)]
  \end{align*}
  where a worst-case distribution $\Pemp{T}^{\cN,\alpha}\in \cP$, $\LP_{\cN}(\Pemp{T},\Pemp{T}^{\cN,\alpha}) \leq \alpha$ can be found as
  \begin{equation}
    \label{eq:wc-distribution-noise-misspecification}
    \textstyle \Pemp{T}^{\cN,\alpha} =  \sum_{k=p+1}^{T} \delta_{\xi'_k} \Pemp{T}(\xi_k)+(1-\alpha-\sum_{k=p+1}^{T}\Pemp{T}(\xi_k)) \delta_{\xi'_{p}} + \alpha \delta_{\xi_\infty}
  \end{equation}
  where $p$ the smallest index so that $1-\alpha-\sum_{k=p+1}^{T}\Pemp{T}(\xi_k)$ is strictly positive.
\end{theorem}

\begin{figure}
  \centering
  \import{figures/}{cvar.tex}
  \caption{Illustration of the LP-DRO expression of Theorem \ref{thm: LP-DRO expression for empirical}. The circles represent the observed corrupted samples ordered by increasing inflated loss $\loss^\cN(x,\xi_{[1]}) \leq  \ldots \leq \loss^\cN(x,\xi_{[T]})$, with $p = \lceil \alpha T \rceil$. The filled part represent the $1-\alpha$ fraction with highest inflated loss $\loss^\cN$. The LP-DRO predictor corresponds to the total loss of this $1-\alpha$ fraction of the samples with highest loss $\loss^\cN$ (which is $(1-\alpha)\CVaR$) plus $\alpha$ times the worst-case scenario. The adversary replaces the $\alpha$ fraction of the samples with lowest inflated loss $\loss^\cN$ with the worst-case loss.}
  \label{fig:LP-DRO illustration}
\end{figure}

We prove a more general version of Theorem \ref{thm: LP-DRO expression for empirical} in Appendix \ref{sec:gener-theor-refthm} where $\Pemp{T}$ is not required to be an empirical distribution.
Theorem \ref{thm: LP-DRO expression for empirical} shows that we can compute the LP-DRO predictor explicitly and provides an intuitive interpretation on how noise and misspecification interact when both are considered simultaneously, which we illustrate in Figure \ref{fig:LP-DRO illustration}.
First, the loss function considered here is the $\cN$-inflated counterpart $\loss^{\cN}$ of the loss function $\ell$
as is the case for predictors robust against noise $\hat{c}_{\mathrm{R}}(x,\Pemp{T})$ defined in \eqref{eq: nosie robust predictor}.
Second, the conditional value-at-risk of the inflated loss function above the $1-\alpha$ quantile is considered, which corresponds to the predictor $\hat{c}_{\mathrm{M}}(x,\Pemp{T})$ robust against misspecification stated in Equation \eqref{eq: predictor misspecification}.
Finally, the LP-DRO predictor considers the $1-\alpha$ percent worst $\cN$-inflated loss scenarios  and substitutes the ignored $\alpha$ percent lowest loss scenarios with the worst-case loss $\max_{\xi\in \Sigma}\loss(x,\xi)$.
This result is rather intuitive, as the misspecification is adversarial it will tend to target an $\alpha$ fraction of low loss scenarios and corrupt them into worst-case scenarios. The remaining $1-\alpha$ fraction of samples is then perturbed adversarially by bounded noise resulting in the $\cN$-inflated loss; see also Figure \ref{fig:LP-DRO illustration}.

  \begin{remark}[Robust Statistics and Outliers]
    From Figure \ref{fig:LP-DRO illustration} and the previous discussion it can be remarked that the LP-DRO formulation assigns more influence to those noise realizations associated with a large inflated loss while at the same time suppresses the influence of those realizations with a small inflated loss.
    This runs counter to the general strategy of robust methods introduced by \cite{tukey1958bias,huber1981robust} which try to identify the $\alpha$ fraction of the data points which are misspecified outliers and consequently suppress their influence. We remark however that outlier identification is only then a sound strategy if structural conditions on the unknown distribution such as normality \citep{huber1981robust} or bounded variance \citep{diakonikolas2019recent} are imposed. Critically, here we make no such assumptions and hence there is also no concept of statistical outlier as opposed to recent work by \cite{nietert2023OutlierRobustWassersteinDRO}. A detailed study of this seeming difference in paradigm between DRO and robust statistics can be found in \cite{chan2024DistributionalRobustnessRobust}.
  \end{remark}

Theorem \ref{thm: LP-DRO expression for empirical} implies that the LP-DRO formulation can be tractably computed provided we can evaluate the inflated loss function $\loss^{\cN}$. Hence, our novel LP-DRO formulation is computationally not harder to solve than the robust formulations proposed by \citet{madry2017towards}. In particular, when the loss function $\loss(x,\xi)$ is convex in the decision $x$ for all $\xi$, the LP-DRO predictor is convex in $x$ as well.

Finally, our LP-DRO robust predictor can be interpreted as a generalization of the predictors $\hat{c}_{\mathrm{R}}(x,\Pemp{T})$ and $\hat{c}_{\mathrm{M}}(x,\Pemp{T})$ which protect against both noise and misspecification, simultaneously.
\begin{corollary}
  \label{corollary:interpretation}
  When $\alpha T$ is integer, any $x\in \cX$ and noise set $\cN$, we have
  \begin{equation*}
  \displaystyle \hat{c}^{\mathcal N, \alpha}_{\LPtxt}(x,\Pemp{T}) = \max \set{\displaystyle\sum_{t\in[T]} \frac{\loss(x,\xi_t-\noise_t-\noise'_t)}{T}}{\sum_{t\in[T]}\frac{\one{\noise'_t\neq 0}}{T}\leq \alpha,~\noise_t \in \mathcal N, ~ \xi_t-\noise_t-\noise'_t\in \Sigma~~\forall t\in [T]}.
\end{equation*}
\end{corollary}

\subsection{Relationship with other DRO predictors}
\label{sec:rel-DRO}
We note in this section that the LP-DRO predictor is a generalization of several popular DRO predictors with metrics such as type-$\infty$ Wasserstein and total variation. We first exhibit its close relationship with the L\'evy-Prokhorov (LP) metric, on which basis we coined the LP-DRO predictor its name.

\paragraph{L\'evy-Prokhorov Metric}
The pseudo-metric $\LP_{\cN}$ is intimately related to the classical L\'evy-Prokhorov metric 
$$\LPdist(\mu,\nu) = \inf\{\epsilon>0 \; : \; \nu(A)\leq \mu(A^\epsilon) + \epsilon\quad \forall A \in \cB(\Sigma)\}$$
for $\epsilon>0$, for all $\mu$ and $\nu$ where $A^\epsilon \defn \{ \xi \in A \; | \exists \xi' \in A, \; \|\xi'-\xi \| < \epsilon \}$. The L\'evy-Prokhorov metric is perhaps best known for its topological properties. Indeed, weak convergence of probability distributions is equivalent to convergence in the L\'evy-Prokhorov metric. Following \citet{strassen1965existence}, we have that $\set{\nu\in \cP}{\LP_{\mathcal B(0, \epsilon)}(\mu,\nu) \leq \epsilon} =  \set{\nu\in \cP}{\LPdist(\mu,\nu) \leq \epsilon}$. The LP-DRO predictor hence comprises (for $\cN=\mathcal B(0, \epsilon)$ and $\alpha = \epsilon$) the L\'evy-Prokhorov robust cost predictor
$
  \sup \{ 
  \Eb_{\Pb'}[\loss(x,\txi)]
  \; : \; 
  \Pb'\in \cP, \; 
  \LPdist(\Pemp{T},\Pb') \leq \epsilon
  \}
$
studied by \citet{erdougan2006ambiguous} who propose an approximate evaluation. Hence, even in the restricted context of L\'evy-Prokhorov robust optimization the result stated in Theorem \ref{thm: LP-DRO expression for empirical} appears to be novel.

\paragraph{Type-$\infty$ Wasserstein Metric}
    The type-$p$ Wasserstein distance family for $p\geq 1$ between two distributions $\mu$ and $\nu$ is defined as $W_p(\mu, \nu)=\left(\inf \set{\int \norm{\xi-\xi'}^p \d \gamma(\xi, \xi')}{\gamma\in \Gamma(\mu,\nu)}\right)^{1/p}$.
    \citet[Proposition 3, Proposition 5]{givens1984class} show that we have in the limit
    \[
      \lim_{p\to\infty} W_p(\mu, \nu) = W_\infty(\mu, \nu) \defn \inf\{\epsilon>0 \; : \; \nu(A)\leq \mu(A^\epsilon)\quad \forall A \in \cB(\Sigma)\}.
    \]
    Following \citet{strassen1965existence}, we have that $\set{\nu\in \cP}{\LP_{\mathcal B(0, \epsilon)}(\mu,\nu) \leq 0} =  \set{\nu\in \cP}{W_\infty(\mu,\nu) \leq \epsilon}$.
    The LP-DRO predictor hence comprises (for $\cN=\mathcal B(0, \epsilon)$ and $\alpha = 0$) the type-$\infty$ Wasserstein robust cost predictor
    $
    \sup \{ 
    \Eb_{\Pb'}[\loss(x,\txi)]
    \; : \; 
    \Pb'\in \cP, \; 
    W_\infty(\Pemp{T},\Pb') \leq \epsilon
    \}
    $
    studied recently by \citet{bertsimas2022two, xie2019tractable} in the context of two-stage robust optimization and by \citet{nguyen2020distributionally} in the context of conditional estimation.

\paragraph{Total Variation}
    The total variation distance between any distributions $\mu$ and $\nu$ is defined as
    \[
      TV(\mu, \nu) = \inf\{\epsilon>0 \; : \; \nu(A)\leq \mu(A)+\epsilon \quad \forall A \in \cB(\Sigma)\}.
    \]
    Following again \citet{strassen1965existence}, we have that $\set{\nu\in \cP}{\LP_{\{0\}}(\mu,\nu) \leq \epsilon} =  \set{\nu\in \cP}{TV(\mu,\nu) \leq \epsilon}$. The LP-DRO predictor hence comprises (for $\cN=\{0\}$ and $\alpha = \epsilon$) the total variation robust cost predictor
$
  \sup \{ 
  \Eb_{\Pb'}[\loss(x,\txi)]
  \; : \; 
  \Pb'\in \cP, \; 
  TV(\Pemp{T},\Pb') \leq \epsilon
  \}.
  $
  Total variation robust formulations were proven effective by \citet{rahimian2019identifying} to identify critical scenarios in stochastic optimization models. With this perspective, Theorem \ref{thm: LP-DRO expression for empirical} can be seen as an extension of prior reformulations of TV-DRO \citep{shapiro2017distributionally,jiang2018risk}.

\section{Holistic Robustness}\label{sec:HR-robustness}

The previous section shows that the LP-DRO predictor defined in Equation \eqref{eq: LP-DRO predictor} provides robustness efficiently against noise and misspecification. On the other hand, as we mentioned in Section \ref{sec:introduction}, the KL-DRO predictor provides efficient robustness against statistical error as pointed out recently by \citet{duchi2021statistics, van2021data, bennouna2021learning}. However, none of the discussed robust predictors protects against all sources of overfitting simultaneously.

The question hence becomes whether a predictor can be constructed that is simultaneously robust against all the three sources of overfitting.
A natural idea is to combine the KL and LP ambiguity sets which separately capture robustness against statistical error, noise and misspecification, respectively.

\subsection{A Holistic Robust Predictor}\label{sec:HR-pred}

Consider the following data generation process in which sample data is generated by independent draws from the true out-of-sample distribution $\Pb$.
However, this sample data is not observed directly by the decision-maker.
Rather, each sample is first corrupted by noise realizing in $\cN$.
Finally, the decision maker only gets to observe data in which up to a fraction $\alpha$ of these noisy samples have been substituted by arbitrary fake data. Remark that as opposed to Section \ref{sec: corruption}, the statistical error is here \textit{not} negligible. We do not observe the full distribution of the corrupted uncertainty but rather must learn its distribution, and infer the clean distribution from samples. 

We construct a holistically robust predictor against statistical error, noise and misspecification.
Let us first provide intuition on the construction. The goal is to construct an ambiguity set containing the true out-of-sample distribution $\Pb$ with high probability, given the observed empirical distribution $\hat{\Pb}_T$ of the corrupted sample. Let $\hat{\Qb}$ be the empirical distribution of the unobserved sample, from the true out-of-sample distribution $\Pb$.
Sanov's theorem \citep[Theorem 6.2.10]{dembod1996large} indicates that this empirical distribution (intuitively) verifies\footnote{This is not true for  a continuous distribution $\Pb$. We state this here only for the purpose of intuition. The proof requires more elaborate arguments, see Appendix \ref{App: proof of sec det corr}.} $\KL(\hat{\Qb}||\Pb) \leq r$ with high probability $1-\exp(-rT+o(T))$.
As samples from of the empirical distribution $\hat{\Qb}$ are perturbed into $\Pemp{T}$ with noise in $\cN$ and a fraction up to $\alpha$ is misspecified, Theorem \ref{thm: LP-DRO robustness} ensures that we must have $\LP_{\cN}(\hat{\Pb}_T,\hat{\Qb}) \leq \alpha$. Hence, the out-of-sample distribution $\Pb$ is in the set $\tset{\Pb' \in \cP}{\exists\hat{\Qb} \in \cP, \;
  \LP_{\cN}(\Pemp{T}, \hat{\Qb}) \leq \alpha, \;
  \KL (\hat{\Qb}||\Pb') \leq r}$.
Figure \ref{fig:deterministic corruption} illustrates the hypothesized data generation process visually.
  
Following this intuition, for a given noise support $\cN$, a misspecification fraction $\alpha\in [0,1]$ and a desired statistical robustness parameter $r$, we consider the holistic robust predictor 
\begin{align}\label{eq: HD predictor}
  \hat{c}^{\cN,\alpha,r}_{\HR}(x,\Pemp{T})
  =
  \sup \{ 
  \Eb_{\Pb'}[\loss(x,\txi)]
  \; : \; 
  \Pb'\in \cP,\; \hat{\Qb} \in \cP, \;
  \LP_{\cN}(\Pemp{T}, \hat{\Qb}) \leq \alpha, \;
  \KL (\hat{\Qb}||\Pb') \leq r
  \} \quad \forall x \in \cX.
\end{align}

\begin{figure}
  \centering
  \import{figures/}{post-sample.tex}
  \caption{When learning with corrupted data, data points are first sampled independently from $\Pb$.
    Sampling from $\Pb$ itself results in statistical errors which however remain bounded with high probability $1-\exp(-rT+o(T))$ by $r$ as measured in by the KL divergence between $\Pb$ and $\hat{\mathbb{Q}}$.
    Subsequently, an adversary corrupts these samples with noise in $\cN$ and misspecification at frequency at most $\alpha$. Theorem \ref{thm: LP-DRO robustness} guarantees that the distance between $\hat{\mathbb{Q}}$ and $\Pemp{T}$ is a most $\alpha$ as measured by the pseudometric $\LP_{\cN}$.}
  \label{fig:deterministic corruption}
\end{figure}

We formally prove that the HR predictor enjoys robustness in the presence of noise and misspecification. Proofs of this section are in Appendix \ref{App: proof of sec det corr}.

\begin{theorem}[Uniform Robustness Out-of-sample Guarantee]
    \label{thm: HD robustness guarantee}
    Let $\Pemp{T}$ be the observed empirical distribution of corrupted samples. Assume that the noise realizes in a compact set in the interior of $\cN$ and the corrupted portion of the data is less than $\alpha T$. 
    We have
    \begin{equation*}
      \Prob\left(\exists x \in \cX,\, \hat{c}^{\cN,\alpha,r}_{\HR}(x,\Pemp{T}) < \Eb_{\Pb}[\loss(x,\txi)] \right) \leq  \exp(-rT +O(1)).
    \end{equation*}
\end{theorem}

The previous result establishes that the HR predictor enjoys a uniform asymptotic out-of-sample guarantee and hence safeguards against the three sources of overfitting simultaneously. In particular, the robustness out-of-sample guarantee is uniform in the decision set $\cX$, an important propriety for optimization as it implies $\Prob\left(\hat{c}^{\cN,\alpha,r}_{\HR}(\hat{x}^\star_{\HR},\Pemp{T}) \geq \Eb_{\Pb}[\loss(\hat{x}^\star_{\HR},\txi)]\right) \geq 1- \exp(-rT +O(1))$, where $\hat{x}^\star_{\HR}$  minimizes  the $\HR$ predictor. In classical statistical learning literature, such uniform upper bounds are typically obtained through restricting the complexity decision set (or hypothesis class) $\cX$, such as its VC dimension \citep{vapnik1999nature}. The then obtained bounds depend on such complexity measures, which may render them impractical in high-dimensional applications. 
In the proof of Theorem \ref{thm: HD robustness guarantee}, we exhibit a finite sample result as well in which we show that the $O(1)$ term is independent of the complexity of the set $\cX$.
We in fact show that for all $T\geq 1$, we have $\Prob\left(\exists x \in \cX,\, \hat{c}^{\cN^\delta,\alpha+\delta,r}_{\HR}(x,\Pemp{T}) < \Eb_{\Pb}[\loss(x,\txi)] \right) \leq \exp(-rT + m(\Sigma,\delta)\log \left( \tfrac{4}{\delta} \right))$ for any $\delta>0$, where $\cN^\delta = \cN + \cB(0,\delta)$ and $m(\Sigma, \delta)$ denotes the internal covering number of the support set $\Sigma$.


In the proof of Theorem \ref{thm: HD robustness guarantee}, we also shows that the ambiguity set of the HR predictor contains the out-of-sample distribution with high probability. This shows that the LP metric smooths the KL ambiguity sets overcoming some of its limitations \citep{liu2023smoothed}. Indeed, while KL-DRO is proven to be optimal in a certain sense when data is clean \citep{van2021data}, the KL ambiguity set never contains the out-of-sample distribution when this distribution is continuous. The HR ambiguity sets circumvents this issue by use of the optimal transport metric, and contains the out-of-sample distribution with high probability in the more general setting of corrupted observed data.

Several DRO formulations\footnote{But not all; for example, the KL divergence, as illustrated in Figure \ref{fig:distributional-shift}.} can verify the out-of-sample guarantee when their ambiguity sets are sufficiently inflated, i.e., when their radius is large enough. However, this often leads to overly conservative predictions. This raises an important question: how conservative is a given DRO formulation in verifying the out-of-sample guarantee?

The following result demonstrates that the HR predictor achieves \textit{efficient} robustness. Specifically, among all predictors (i.e., mappings from data to predictions) that satisfy the out-of-sample guarantee, the HR predictor yields the \textit{least conservative} cost estimates.

  \begin{theorem}[Efficient Robustness]
    \label{thm:efficiency-hd}
    Suppose $0 \in \interior(\cN), \cl(\interior(\cN)) = \cN$, $\alpha>0$, $r>0$ and $\xi \to \loss(x,\xi)$ continuous and bounded for all $x\in \cX$. 
    Let $\hat{c}$ be a predictor that verifies the out-of-sample guarantee 
    $$
    \limsup_{T \to \infty} \frac{1}{T} \log \Prob \left(\exists x \in \cX, \, \hat{c}(x,\hat{\Pb}_T) < \Eb_{\Pb}[\loss(x,\txi)]\right) \leq -r,
    $$
    for all out-of-sample distributions $\Pb$ and adversary which misspecifies less than a fraction $\alpha$ of training data points and perturbs the testing data with noise bounded in a compact set in the interior of $\cN$. Then, $\hat{c}$ is necessarily uniformly more conservative than $\hat{c}^{\cN,\alpha,r}_{\HR}$, that is, we have
    $
    \hat{c}(x,\hat{\Pb}) 
    \geq
    \hat{c}^{\cN,\alpha,r}_{\HR}(x,\hat{\Pb}),
    $
    for all $x \in \cX$ and observed empirical distribution $\hat{\Pb}$.
\end{theorem}

The previous results show that the $\HR$ predictor provides precisely the desired robustness, and is in fact an efficient robust predictor.
However, it is not clear whether the $\HR$ predictor is tractable as it is characterized as a maximum over probability measures. In the following result, we show that computing the predictor can be reduced to solving a tractable optimization problem. Moreover, we show that the supremum in Equation \eqref{eq: HD predictor} admits an optimal solution with finite \textit{known} support, which characterizes the worst-case adversary.
This is rather a surprising result, as the predictor characterized in Equation \eqref{eq: HD predictor} considers the worst-case distribution among all potentially continuous distributions on $\Sigma$.

\begin{theorem}[Finite Primal Reduction]
  \label{thm: HD finite formulation}
  Let $\{\xi_1,\ldots,\xi_K\}$ be the support of $\Pemp{T}$. For all $x\in \cX$, the HR predictor \eqref{eq: HD predictor} admits the representation
  \begingroup
  \allowdisplaybreaks
  \begin{align}
    \label{eq:finite-formulation-hd}
    \hat{c}^{\cN,\alpha,r}_{\HR}(x,\Pemp{T}) = &                                                          \left\{
                                                         \begin{array}{rl}
                                                           \max & \sum_{k\in[K]} p'_{k}\loss^{\cN}(x,\xi_k) + p'_{K+1} \max_{\xi \in \Sigma} \loss(x,\xi) \\[0.5em]
                                                           \st &  p'\in \Re^{K+1}_+, \, \hat q\in \Re^{K+1}_+, \, s \in \Re^{K}_+,\\[0.5em]
                                                                &  \sum_{k\in[K+1]} p'_{k} = 1,\quad \sum_{k\in[K+1]} \hat q_{k}=1\\[0.5em]
                                                                & \sum_{k\in[K+1]} \hat q_k \log \left( \frac{\hat q_k}{p'_{k}} \right) \leq r,\\[0.5em]
                                                                & \sum_{k\in[K]} s_k\leq \alpha, \\[0.5em]
                                                                & \hat q_k + s_k = \Pemp{T}(\xi_k) \quad \forall k \in [K].
                                                         \end{array}\right.
  \end{align}
  \endgroup
\end{theorem}
Notice that the nonlinear constraint in the maximization problem \eqref{eq:finite-formulation-hd} can be represented as an exponential cone constraint by introducing the auxiliary variable $t \in \Re^K$ with $\sum_{k\in [K]} t_k\leq r$ and the auxiliary constraints $(-t_k,\Pemp{T}(\xi_k),q_k) \in \mathcal{K}_{\mathrm{exp}}$ for all $k \in [K]$ where $\mathcal{K}_{\mathrm{exp}}$ is the exponential cone. Hence, problem \eqref{eq:finite-formulation-hd} can be solved using off-the-shelf exponential cone optimization solvers \citep{dahl2021primal}.

Given an optimal solution in the optimization problem \eqref{eq:finite-formulation-hd} we can explicitly construct an optimal solution in the optimization problem \eqref{eq: HD predictor} as well. To that end recall that for all $k\in [K]$ the points $\xi'_k \in \xi_k-\argmax_{\noise \in \cN, \,\xi_k - \noise\in\Sigma} \loss(x,\xi_k - \noise)$ and $\xi_\infty \in \argmax_{\xi \in \Sigma} \loss(x,\xi)$.
We can associate with any feasible $(\hat q, s, p')$ in \eqref{eq:finite-formulation-hd} the distributions
$\hat \Qb =  \textstyle \sum_{k\in [K]} \hat q_k \delta_{\xi'_k} + \hat q_{K+1} \delta_{\xi_\infty}$,
$\Pb' =  \textstyle \sum_{k\in [K]} p'_k \delta_{\xi'_k} + p'_{K+1} \delta_{\xi_\infty}$ and 
$\gamma = \textstyle \sum_{k\in K} \hat q_k \delta_{\xi_k,\xi_k'} + s_k \delta_{\xi_k, \xi_\infty}$.
The constraints $\hat q_k+s_k=\Pemp{T}(\xi_k)$ for all $k\in [K]$, $\sum_{k\in [K+1]}p'_k=1$ and $\sum_{k\in [K+1]}\hat q_k=1$ guarantee that $\gamma\in \Gamma(\Pemp{T}, \hat \Qb)$ while $\int \indic(\xi-\xi' \not\in\cN) \,\d\gamma(\xi,\xi')\leq \sum_{k\in[K]} s_k \leq \alpha$ guarantees $\LP_\cN(\Pemp{T}, \hat \Qb)\leq \alpha$. Furthermore, we have \(  \KL (\hat{\Qb}||\Pb') = \sum_{k\in[K+1]} \hat q_k \log(\tfrac{\hat q_k}{p'_k})\leq r \) and $\Eb_{\Pb'}[\loss(x,\txi)]=\sum_{k\in [K]}p'_k \loss(x,\xi_k') +p'_{K+1}\loss(x,\xi_\infty)=\sum_{k\in [K]}p'_k\loss^\cN(x, \xi_k)+p_{K+1}'\max_{\xi\in\Sigma}\loss(x, \xi)$. Hence, for all feasible solutions in \eqref{eq:finite-formulation-hd} we can associate finitely supported distributions feasible in the maximization problem \eqref{eq: HD predictor} characterizing the $\HR$ predictor attaining the same objective value.
A feasible solution $(p', q', s)$ in problem \eqref{eq:finite-formulation-hd} corresponds to the solution
\begin{equation*}
  \Pb'=\textstyle \sum_{k\in [K]} p'_k \delta_{\xi'_k}+p_{K+1}'\delta_{\xi_\infty}\quad {\rm{and}} \quad
  \hat \Qb =  \textstyle \sum_{k\in[K]} \hat q_k \delta_{\xi'_k}+\hat q_{K+1}\delta_{\xi_\infty}.
\end{equation*}
feasible in the maximization problem \eqref{eq: HD predictor} characterizing the $\HR$ predictor; see the proof of Theorem \ref{thm: HD finite formulation}.

Although Theorem \ref{thm: HD finite formulation} provides a tractable ways to evaluate the $\HR$ predictor, finding an optimal decision in  $\min_{x\in \cX} \hat{c}^{\cN,\alpha,r}_{\HR}(x,\Pemp{T})$ still requires solving a saddle point problem which may yet again be slightly awkward to solve in practice.
In the following, we provide an equivalent tractable dual minimization representation of the HR predictor.

\begin{theorem}[Dual Formulation]\label{thm: dual rep HD}
  Let $\{\xi_1,\ldots,\xi_K\}$ be the support of $\Pemp{T}$. For all $x\in \cX$, the HR predictor \eqref{eq: HD predictor} admits for all $r>0$ the dual representation
  \begin{align}
    \hat{c}^{\cN,\alpha,r}_{\HR}(x,\Pemp{T}) & = \left\{
    \begin{array}{rl}
      \inf & \sum_{k\in[K]} w_k \Pemp{T}(\xi_k) + \lambda (r-1) + \beta \alpha+ \eta\\[0.5em]
      \st &  w\in \Re^K,\; \lambda \geq 0, \; \beta \geq 0,\; \eta \in \Re, \\[0.5em]
           & w_k \geq \lambda\log \left( \frac{\lambda}{\eta - \loss^{\cN}(x,\xi_k)}\right),~ w_k \geq  \lambda \log \left( \frac{\lambda}{\eta - \max_{\xi \in \Sigma} \loss(x,\xi)}\right) -\beta \quad \forall k \in [K],\\[0.5em]
           & \eta \geq \max_{\xi \in \Sigma} \loss(x,\xi).
    \end{array}\right.
  \end{align}  
\end{theorem}

In Appendix \ref{App: Dual HD}, we prove this dual formulation in the more general case when $\Pemp{T}$ may be any possibly continuous distribution. 
Hence, minimizing the $\HR$ predictor $\min_{x\in \cX} \hat{c}^{\cN,\alpha,r}_{\HR}(x,\Pemp{T})$ using the dual formulation is tractable whenever the inflated loss function $\loss^{\cN}(x,\xi)$ can be evaluated efficiently and hence the HR formulation is not harder to solve than the robust formulation $\min_{x\in \cX} \Eb_{\hat{\Pb}_T}[\loss^\cN(x,\txi)]$.

Finally, we prove that the HR predictor can be interpreted as a KL-DRO predictor (protecting against statistical error) applied to a perturbed distribution (guarding against noise and misspecification).

\begin{lemma}
  \label{lemma:hr:primal:reduction}
  Consider the worst-case distribution
  \(
  \Pemp{T}^{\cN,\alpha} \in \arg \max \tset{\Eb_{\Pb'}[\loss(x,\txi)]}
  {\Pb'\in \cP, ~ \LP_{\cN}(\Pemp{T},\Pb') \leq \alpha}
  \)
  defined in the context of Theorem \ref{thm: LP-DRO expression for empirical}.
  For all $x\in \cX$, the HR predictor \eqref{eq: HD predictor} admits the decomposition
  $$\hat{c}^{\cN,\alpha,r}_{\HR}(x,\Pemp{T}) = \hat{c}^r_{\mathrm{KL}}(x,\Pemp{T}^{\cN,\alpha}).$$
\end{lemma}

This results implies an alternative way of computing the HR predictor.
In fact, from Lemma \ref{lemma:hr:primal:reduction} and \citet[Proposition 5]{van2021data} the following even more compact reduced formulation follows.

\begin{lemma}
  \label{lemma:hr:dual:reduction}
  Consider the worst-case distribution
  \(
  \Pemp{T}^{\cN,\alpha} \in \arg \max \tset{\Eb_{\Pb'}[\loss(x,\txi)]}
  {\Pb'\in \cP, ~ \LP_{\cN}(\Pemp{T},\Pb') \leq \alpha}
  \)
  defined in the context of Theorem \ref{thm: LP-DRO expression for empirical}.
  For all $x\in \cX$ and $r>0$, the HR predictor \eqref{eq: HD predictor} admits the dual representation
  \begin{equation}
    \label{eq:univariate-dual}
    \hat{c}^{\cN,\alpha,r}_{\HR}(x,\Pemp{T}) =  \min \set{\textstyle \alpha - \exp(-r) \exp\left( \int \log\left(\alpha-\ell(x, \xi)\right) \d \Pemp{T}^{\cN,\alpha}(\xi)\right)}{\textstyle \alpha\ge\max_{\xi \in \Sigma} \loss(x,\xi)}.
  \end{equation}
  Furthermore, the minimization problem (\ref{eq:univariate-dual}) admits an minimizer $\alpha^\star$ which is bounded as
  \begin{equation}
    \label{eq:sandwich:alpha}
    \max_{\xi \in \Sigma} \loss(x,\xi)\leq \alpha^\star \leq \frac{\max_{\xi \in \Sigma} \loss(x,\xi) -e^{-r}\, \int \ell(x, \xi) \d \Pemp{T}^{\cN,\alpha}(\xi) }{1-e^{-r}}.
  \end{equation}
\end{lemma}  

We remark that the HR-DRO predictor is characterized in Lemma \ref{lemma:hr:dual:reduction} as the solution to a univariate convex optimization problem which can be efficiently solved using a simple bisection search starting from Equation (\ref{eq:sandwich:alpha}) whereas a worst-case distribution $\Pemp{T}^{\cN,\alpha}$ can be trivially found based on Equation (\ref{eq:wc-distribution-noise-misspecification}) after having simply sorted the values $(\ell^\cN(x, \xi_1), \dots, \ell^\cN(x, \xi_K))$ where here $(\xi_1,\dots, \xi_K)$ denotes again the support of the empirical distribution $\Pemp{T}$.

\subsection{Oblivious Adversaries}\label{sec: random corr}

The adversaries considered in the previous section are called \textit{adaptive} by \cite{ben1994power}; the data is first sampled from the out-of-sample distribution and then corrupted. 
An alternative corruption model can be considered by reversing the order of sampling and corruption; the out-of-sample distribution is first corrupted by an adversary, and then data is sampled from the corrupted distribution. This corruption model is associated with oblivious adversaries \citep{ben1994power}. More formally, let $\txi$ be a random variable following the true out-of-sample distribution of the uncertainty $\Pb$.
Each observed data point will be a sample, not from $\txi\in \Sigma$, but from a corrupted data source
$\txi^c = (\txi + \Tilde{n}) \indic(\tc =0) + \txi_0  \indic(\tc =1)\in\Sigma$, where $\Tilde{n}$ is a noisy random variable, $\tc$ is a binary random variable modeling whether the sample is misspecified, and $\txi_0$ is an arbitrary random variable representing the misspecified observation. We observe a finite data sample from this corrupted distribution. These adversaries are generally weaker than the previously considered adaptive adversaries \citep{ben1994power,zhu2022generalized, blanc2022power}.
Although the focus of this paper will be on adaptive adversaries, we briefly indicate in this section that inverting the order of LP and KL in the HR predictor's ambiguity set naturally leads to robustness to oblivious adversaries.

Let us first provide intuition on our construction. The goal is to construct an ambiguity set containing the true out-of-sample distribution $\Pb$ with high probability, given the observed empirical distribution $\hat{\Pb}_T$ of the corrupted sample.
Assume that the noise $\tn$ realizes in $\cN$ and the probability of a sample being corrupted is no more than the fraction $\alpha$ (i.e., $\Prob(\tilde{c}=1) \leq \alpha$).
Following Theorem \ref{thm: LP-DRO robustness}, this means that the corrupted data could in effect have been sampled by the adversary from any distribution $\Qb$ verifying $\LP_{\cN}(\Qb, \Pb)\leq \alpha$. Now $\hat{\Pb}_T$ is an empirical distribution sampled from the distribution $\Qb$. Following Sanov's theorem \citep[Theorem 6.2.10]{dembod1996large}, we have\footnote{Again, this is only for the purpous of intuition, but does not hold theoretically.} $\KL(\Pemp{T}||\mathbb Q)\leq r$ with probability $1-\exp(-rT+o(T))$. Hence the out-of-sample distribution $\Pb$ is expected to be in the set $\{\Pb'\in \cP \; : \;  \exists \Qb' \in \cP, \;
  \KL (\Pemp{T}||\Qb') \leq r, \;
  \LP_{\cN}(\Qb',\Pb') \leq \alpha  \}$ with high probability $1-\exp(-rT+o(T))$.
Figure \ref{fig:on-sample corruption} illustrates this data generation process visually.
Following this intuition, we consider the holistic robust predictor for oblivious adversaries (HRo)
\begin{equation}\label{eq: HRo predictor}
  \hat{c}^{\cN,\alpha,r}_{\HRo}(x,\Pemp{T})
  \defn
  \max \{ 
  \Eb_{\Pb'}[\loss(x,\txi)]
  \; : \; 
  \Pb'\in \cP,\, \Qb' \in \cP, \;
  \KL (\Pemp{T}||\Qb') \leq r, \;
  \LP_{\cN}(\Qb',\Pb') \leq \alpha 
  \} \quad \forall x\in \cX.
\end{equation}

\begin{figure}[t]
  \centering
  \import{figures/}{on-sample.tex}
  \caption{In the data generation process with oblivious advesaries, an adversary can corrupt the data generation distribution away from $\Pb$ towards any $\Qb$ within distance $\alpha$ as measured by our pseudo metric $\LP_\cN$. Sampling from $\Qb$ itself results in statistical errors which however remain bounded with high probability $1-\exp(-rT)$ by $r$ as measured in by the KL divergence between $\Pemp{T}$ and $\mathbb Q$.}
  \label{fig:on-sample corruption}
\end{figure}

We prove next prove that the $\HRo$ predictor indeed provides robustness in the oblivious adversary setting.
\begin{theorem}[Robustness Out-of-sample Guarantee]\label{thm: HR robustness}
  Let $\Pemp{T}$ be the observed empirical distribution of corrupted samples $\txi^c$ by an oblivious advesary. Assume that the noise $\tn$ realizes in $\cN$ and the probability of a sample being corrupted is no more than $\alpha$, i.e., $\Prob(\tc=1) \leq \alpha$.
  For all $x\in \cX$, we are disappointed, i.e., the event
  $
     \Eb_{\Pb}[\loss(x,\txi)] > \hat{c}^{\cN,\alpha,r}_{\HRo}(x,\Pemp{T})
   $
   occurs, with probability at most $\exp(-rT +o(T))$.
\end{theorem}

The proof of the previous result is in Appendix \ref{App: proofs of random corr}. We remark that Theorem \ref{thm: HR robustness} is weaker than its counterpart for adaptive adversaries found in Theorem \ref{thm: HD robustness guarantee}. Indeed, the result here is holds only asymptotically as $T$ tends to infinity whereas the guarantee in Theorem \ref{thm: HD robustness guarantee} holds for finite samples as well. Furthermore, the guarantee in Theorem \ref{thm: HD robustness guarantee} is pointwise rather than uniform in $x\in \cX$. Technically, this is due to a lack of a finite sample counterpart of the result by \cite[Exercise 4.5.5]{dembod1996large} to the oblivious adversaries. Uniform guarantees can hence be only guaranteed when $\cX$ is finite by a union bound argument or $\cX$ compact by way of an appropriate discretization method.

  We will also not derive here an efficiency guarantee as done in Theorem \ref{thm:efficiency-hd}.
  Rather, we do show here that a worst case distribution is attained in a fine support which we characterize. 
We prove a more general result for general distributions in Appendix \ref{App: proofs of random corr}.

\begin{theorem}[Finite Primal Reduction]\label{thm: HR finite formulation}
  Let $\{\xi_1,\ldots,\xi_K\}$ be the support of $\Pemp{T}\in \cP$. For all $x\in \cX$, the HR predictor admits for all $r>0$ the finite primal representation
  \begin{align}
    \label{eq:finite-formulation-hro}
    \hat{c}^{\cN,\alpha,r}_{\HRo}(x,\Pemp{T})=&                                                          \left\{\begin{array}{rl}
                                                               \max & \sum_{k\in[K]} p'_{k}\loss^{\cN}(x,\xi_k) + p'_{K+1} \max_{\xi \in \Sigma} \loss(x,\xi) \\[0.5em]
                                                               \st &  p'\in \Re^{K+1}_+, \, q'\in \Re^{K+1}_+, \, s \in \Re^{K}_+,\\[0.5em]
                                                                    &  \sum_{k\in[K+1]} p'_{k} = 1,\quad \sum_{k\in[K+1]} q'_{k} = 1,\\[0.5em]
                                                                    & \sum_{k\in[K]} \Pemp{T}(\xi_k) \log \left( \frac{\Pemp{T}(\xi_k)}{q'_{k}} \right) \leq r,\\[0.5em]
                                                                    & \sum_{k\in[K]} s_k\leq \alpha, \\[0.5em]
                                                               &  q'_k=p'_k +s_k \quad \forall k \in [K].
                                                                                                                       \end{array}\right.
  \end{align}
\end{theorem}

Given an optimal solution in the optimization problem \eqref{eq:finite-formulation-hro} we can explicitly construct an optimal solution in the optimization problem \eqref{eq: HRo predictor} as well. To that end consider again for all $k\in [K]$ the points $\xi'_k \in \xi_k-\argmax_{\noise \in \cN, \, \xi_k-\noise\in \Sigma} \loss(x,\xi_k - \noise)$ and let $\xi_\infty \in \argmax_{\xi \in \Sigma} \loss(x,\xi)$.
We can associate with any feasible $(q', s, p')$ in \eqref{eq:finite-formulation-hro} the distributions
$\Qb' =  \textstyle \sum_{k\in[K]} q'_k \delta_{\xi_k}+q_{K+1}'\delta_{\xi_\infty}$, 
$\Pb' =  \textstyle \sum_{k\in [K]} p'_k \delta_{\xi'_k} + p'_{K+1} \delta_{\xi_\infty}$ and 
$\gamma = \textstyle \sum_{k\in K} p'_k \delta_{\xi_k,\xi_k'} + s_k \delta_{\xi_k, \xi_\infty}+q_{K+1}\delta_{\xi_\infty, \xi_\infty}$.
The constraints $q'_k=p'_k+s_k$ for all $k\in [K]$, $\sum_{k\in [K+1]}p'_k=1$ and $\sum_{k\in [K+1]} q'_k=1$ guarantee that $\gamma\in \Gamma(\Qb', \Pb')$ while $\int \indic(\xi-\xi' \not\in\cN) \,\d\gamma(\xi,\xi')\leq \sum_{k\in[K]} s_k \leq \alpha$ guarantees $\LP_\cN(\Qb', \Pb')\leq \alpha$.
Furthermore, we have \(  \KL (\Pemp{T}||\Qb') = \sum_{k\in[K]} \Pemp{T}(\xi_k) \log ( \tfrac{\Pemp{T}(\xi_k)}{q'_{k}} ) \leq r \) and $\Eb_{\Pb'}[\loss(x,\txi)]=\sum_{k\in [K]}p'_k \loss(x,\xi_k') +p'_{K+1}\loss(x,\xi_\infty)=\sum_{k\in [K]}p'_k\loss^\cN(x, \xi_k)+p_{K+1}'\max_{\xi\in\Sigma}\loss(x, \xi)$. Hence, for all feasible solutions in \eqref{eq:finite-formulation-hro} we can associate finitely supported distributions feasible in the maximization problem \eqref{eq: HRo predictor} characterizing the $\HRo$ predictor attaining the same objective value.

A similar dual formulation is also verified for the $\HRo$ predictor, useful for optimization.

\begin{theorem}[Dual Formulation]\label{thm: dual rep HR}
  Let $\{\xi_1,\ldots,\xi_K\}$ be the support of $\Pemp{T}\in\cP$. For all $x\in \cX$, the $\HRo$ predictor \eqref{eq: HRo predictor} admits for all $r>0$ the dual representation
  \begin{equation*}
    \hat{c}^{\cN,\alpha,r}_{\HRo}(x,\Pemp{T})=
    \left\{
    \begin{array}{rl}
      \inf & \sum_{k\in[K]} w_k \Pemp{T}(\xi_k) + \lambda (r-1) + \beta \alpha + \eta\\[0.5em]
      \st & w\in \Re^K, \; \lambda \geq 0, \; \beta \geq 0, \; \eta \in \Re,\\[0.5em]
           & w_k \geq \lambda \log \left( \frac{\lambda}{\eta - \loss^{\cN}(x,\xi_k)}\right), ~ w_k \geq \lambda \log \left( \frac{\lambda}{\eta - \max_{\xi \in \Sigma} \loss(x,\xi) +\beta}\right) \quad \forall k \in [K], \\[0.5em]
           & \eta \geq \max_{\xi \in \Sigma} \loss(x,\xi).
    \end{array}
    \right.
  \end{equation*}
\end{theorem}

In Appendix \ref{App: Dual HR}, we prove this dual formulation in a slightly more general case when $\Pemp{T}$ may be any possibly continuous distribution. Hence, minimizing the $\HRo$ predictor $\min_{x\in \cX} \hat{c}^{\cN,\alpha,r}_{\HRo}(x,\Pemp{T})$ using the dual formulation is also tractable whenever the inflated loss function $\loss^{\cN}(x,\xi)$ can be evaluated efficiently and hence the $\HRo$ formulation is not harder to solve than the robust formulation $\min_{x\in \cX} \Eb_{\hat{\Pb}_T}[\loss^\cN(x,\txi)]$ by \citet{madry2017towards}.

\subsection{Special Cases}
\label{sec:special-cases}

In this section, we will point out that for several special parameter settings, the HR formulation is equivalent and collapses to well-known robust formulations.

As a first example, in linear regression, both Ridge (or Tikhonov) regularization and Sample Variance Penalization \citep{maurer2009empirical} naturally emerge as special cases of HR when applied to the $L_1$ regression loss. Similarly, in linear classification, HR applied to the hinge loss gives rise to soft-margin SVM \citep{cortes1995support} and $\nu$-SVMs \citep{scholkopf2000new}. These derivations are presented and discussed in detail in Appendix \ref{App: Linear-application}.

More generally, when no robustness against statistical error is desired ($r=0$), it follows from Equations \eqref{eq: HD predictor} that
\begin{equation}
  \label{eq:hr r=0}
\hat{c}^{\cN,\alpha,0}_{\HR}(x,\Pemp{T}) = \hat{c}^{\cN,\alpha}_{\LPtxt}(x,\Pemp{T})= (1-\alpha)\CVaR_{\Pemp{T}}^\alpha(\loss^{\cN}(x,\txi))+\alpha\max_{\xi\in \Sigma}\loss(x,\xi).
\end{equation}
That is, the HR predictor naturally reduce to the LP-DRO predictor which was the topic of Section \ref{sec: corruption}.
When additionally no robustness against misspecification is desired ($r=0$, $\alpha =0$), the HR predictor is equivalent to
\begin{align}
    \label{eq:hr r=0 alpha=0}
\hat{c}^{\cN,0,0}_{\HR}(x,\Pemp{T})=\hat{c}_{\mathrm{R}}(x,\Pemp{T})=&                  \textstyle\sum_{k\in[K]} \hat{\Pb}_T(\xi_k) \loss^{\cN}(x,\xi_k).
  \end{align}
This is the robust formulation advanced by \cite{madry2017towards}---adversarial training---which we discussed in Section \ref{sec:noise}. It can also be viewed as Wasserstein type-$\infty$ as discussed in Section \ref{sec:rel-DRO}.

When no robustness against misspecification is desired ($\alpha =0$), the HR predictor is equivalent to
\begin{align}
  \hat{c}^{\cN,0,r}_{\HR}(x,\Pemp{T})=
  \left\{
  \begin{array}{rl}
    \sup & \textstyle\sum_{k\in[K]} p_{k}\loss^{\cN}(x,\xi_k) + p_{K+1} \max_{\xi \in \Sigma} \loss(x,\xi)\\[0.5em]
    \st & p\in \Re^{K+1}_+, \sum_{k\in[K]} \Pemp{T}(\xi_k) \log ( \tfrac{\Pemp{T}(\xi_k)}{p_{k}} ) \leq r.
  \end{array}
  \right.
\end{align}
In particular, if additionally no robustness against noise is desired ($\cN = \{0\}$), the HR formulation is recognized as the classical KL-DRO \eqref{eq: KL predictor}.
Instead of using the empirical loss associated with the observed data points, HR predictors consider the worst-case inflated loss $\ell^\cN$ and reweigh each sample according to the variable $p$ subject to a likelihood constraint.
Additionally, a safeguard against unobserved events in $\Sigma$ is provided by assigning a weight $p_{K+1}$ to outcomes with worst-case loss $\max_{\xi \in \Sigma} \loss(x,\xi)$. Once a worst-case event has been observed, i.e., $\max_{k\in [K]}\loss^{\cN}(x,\xi_k)= \max_{\xi \in \Sigma} \loss(x,\xi)$, then it follows immediately from the monotonicity of the logarithm that without loss of optimality we may assume $p_{K+1}=0$
and hence
\begin{align}
  \label{eq:hr alpha=0}
  \hat{c}^{\cN,0,r}_{\HR}(x,\Pemp{T})=
  \left\{
  \begin{array}{rl}
    \sup & \textstyle\sum_{k\in[K]} p_{k}\loss^{\cN}(x,\xi_k)\\[0.5em]
    \st & p\in \Re^{K}_+, \; \sum_{k\in[K]} \Pemp{T}(\xi_k) \log ( \tfrac{\Pemp{T}(\xi_k)}{p_{k}} ) \leq r.
  \end{array}
  \right.
\end{align}
This final formulation is recognized as the robust likelihood formulation advanced by \citet{wang2016likelihood}.

  \subsection{Hyper-parameters selection}
  \label{ssec:hyper-param-select}

Our HR formulation involves three parameters, $\cN$, $\alpha$, and $r$, each targeting a distinct source of overfitting. Intuitively, when overfitting stems from the three described overfitting sources with varying levels, these three parameters are necessary to offer efficient robustness as described in Theorem \ref{thm:efficiency-hd}. That being said, it would be unreasonable to assume that practical data aligns \textit{exactly} with the data corruption model outlined in Section \ref{sec:sources-overfitting}.
The flexibility of our method enables, nevertheless, tailoring our ambiguity set to capture the likely out-of-sample distributions better.
However, increased flexibility comes with increased practical complexity, as selecting reasonable values for hyperparameters can be daunting.

In a first approach, one could select the hyperparameters $\cN$, $\alpha$, so there is no reason to believe that the data may suffer noise and contamination outside the considered levels. Similarly,  the hyperparameter $r$ can be chosen as a function of the data size and a desired disappointment probability based on Theorem \ref{thm: HD robustness guarantee}. This dogmatic approach comes with the benefit of offering guaranteed out-of-sample performance but can be quite conservative in practice.

An alternative approach is to tailor the hyperparameters based on validation or cross-validation on a validation set. In this practical approach, we are clearly in a state of ``sin'' as none of the theoretical guarantees offered in Theorem \ref{thm: HD robustness guarantee} apply. This approach, which we take in Section \ref{sec: application}, is nevertheless quite successful in practice. As a guide to practitioners, we leave a few recommendations on how to apply this approach successfully. These recommendations stem from extensive numerical experimentation. In the following section, we indeed present three numerical applications where we experiment in-depth HR with data-driven hyper-parameter selection and study the impact of each parameter on model performance.


\paragraph{Practical Hyperparameters Selection}
First, choosing the noise set $\cN$ as the norm ball $\{\noise \; : \; \|\noise\| \leq \epsilon \}$ is a reasonable choice in most settings with continuous data. Discrete sets should be considered, however, when the data is categorical. 

Tuning all three hyperparameters based on validation or cross-validation on a validation set yields superior results, but increases computational cost, requiring training $k_\epsilon \times k_\alpha \times k_r$ models, where $k_\epsilon$, $k_\alpha$, and $k_r$ are the number of candidate values for each parameter. We recommend this approach only when computational resources permit, allowing full utilization of the flexibility offered by our ambiguity set. A reduced HR formulation with only two parameters, $r$ and $\epsilon$, and setting $\alpha = 0$, seems to practically capture much of HR's advantages by addressing only two distinct sources of overfitting: statistical error and data tampering. We recommend this approach as the default.
In a resource-constrained setting, we finally recommend fixing $\epsilon$ to a small, non-zero, fraction of the data point amplitude and tuning only $r$.

\section{Applications}\label{sec: application}

In this section, we investigate the practicality our proposed method in real-life applications. In particular, we chose challenging applications which violate, to various degrees, our assumptions on the nature of the three sources of overfitting. We also measure robustness through application-driven metrics rather than precisely the out-of-sample guarantees for which HR is provably efficient (Theorem \ref{thm:efficiency-hd}).

  We test HR alongside key benchmarks across three main applications. The first application is a computer vision task in healthcare, specifically training deep learning models for brain tumor classification. Here, we examine the robustness of models to data scarcity, potential medical errors (mislabeling), and variations in MRI types (noisy images).
  The second application is investment portfolio selection with actual stock data, 
  where we analyze the volatility of learned portfolios under a particularly acute, hard-to-characterize distribution shift.
  Here, we examine in depth the impact of each HR parameter. 
  The third application is linear regression and classification with synthetic data. Here, we visualize the effects of HR's robustness parameters and the resulting worst-case distributions. Due to space limitations, this application is referred to Appendix \ref{App: Linear-application}.
  To experiment with both types of HR, with oblivious and adaptive adversaries, we apply the former in the first application and the latter in the second and third.

\subsection{Application to Neural Networks with Healthcare Data}

Our first application is training robust neural networks for medical imaging, where the high-stakes nature of medical diagnostics demands not only high accuracy but also robust and reliable models. In this case, we focus on the brain tumor classification problem \citep{msoud_nickparvar_2021}\footnote{\url{https://www.kaggle.com/datasets/masoudnickparvar/brain-tumor-mri-dataset}}, where the task involves identifying and categorizing brain tumors from MRI images by both malignancy (malignant or benign) and tumor type, resulting in four classes in total. There are 5712 labeled images available for training and 1311 for testing. Figure \ref{fig:sample_scans} presents sample images from the dataset.

\begin{figure}[H]
    \centering    \includegraphics[width=1\linewidth]{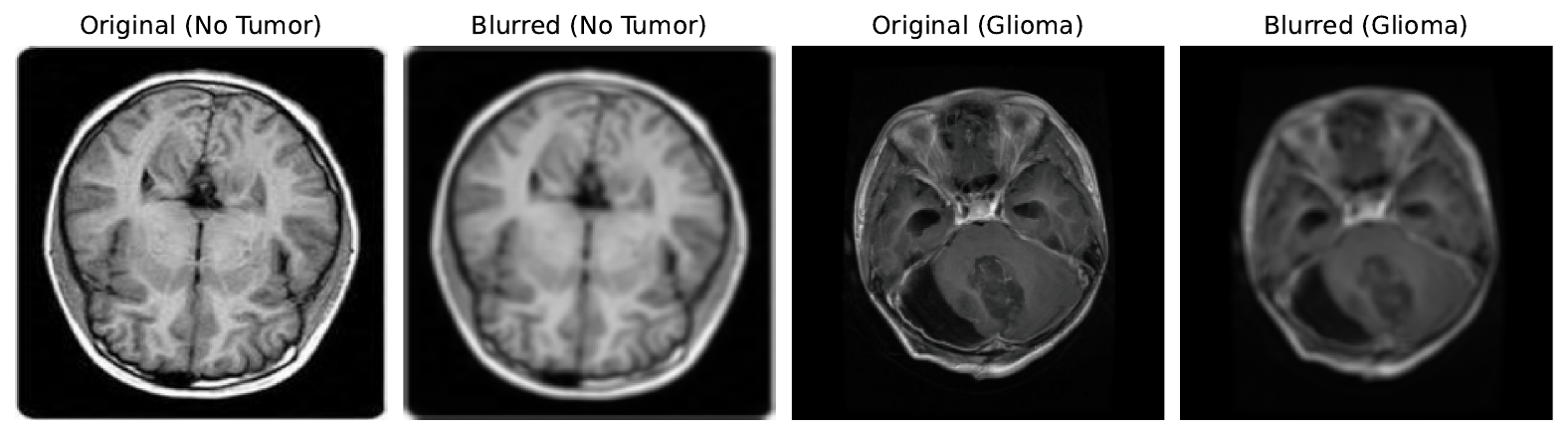}
    \caption{Sample brain scans from the dataset, both clean images and their blurred counterparts.}
    \label{fig:sample_scans}
\end{figure}

In medical imaging applications, the three sources of overfitting discussed in this paper arise naturally. Statistical error is influenced by data scarcity|the smaller the dataset, the higher is the effect of statistical error. Noise can emerge, for example, when training on MRI images from one hospital and deploying the model on images from another facility, potentially with differing resolutions or image characteristics. Finally, misspecification may result from medical errors or mislabeled MRIs.

\paragraph{Experimental Set-Up.}
To evaluate model robustness against each source of overfitting, we control their prevalence through the following:

\begin{enumerate}
    \item  \textbf{Statistical error:} Models are trained on only 25\% of the available training data to simulate the effects of limited data availability.

    \item \textbf{Noise:} Gaussian blur is applied to the testing images, simulating naturally occurring noise rather than adversarial perturbations. Figure \ref{fig:sample_scans} illustrates the resulting testing images.

    \item \textbf{Misspecification:} Label corruption is simulated by randomly flipping 10\% of the training labels.
\end{enumerate}

We choose here a noise set $\cN = \cB_2(0,\epsilon)$ as the norm two ball, and compare several DRO approaches: (i) \textbf{ERM/SAA} ($\alpha = 0, r = 0, \epsilon = 0$), (ii) \textbf{Type-$\infty$ W-DRO} ($\alpha=0$, $r=0$, $\epsilon>0$) with \( \epsilon\) chosen in \(\{0.05, 0.1, 0.15, 0.2\} \), (iii) \textbf{KL-DRO} ($\alpha = 0, r > 0, \epsilon = 0$) with $r$ chosen in \(\{0.05, 0.1, 0.15, 0.2\} \), (iv) \textbf{TV-DRO} ($\alpha > 0, r = 0, \epsilon = 0$) with $\alpha$ chosen in \(\{0.05, 0.1, 0.15, 0.2\} \), and finally (v) \textbf{HR-DRO} ($\alpha > 0, r > 0, \epsilon > 0$) with \( r \in \{0.05, 0.1\} \), \( \alpha \in \{0.05, 0.1\} \), \( \epsilon \in \{0.05, 0.1\} \).

Hyperparameters are selected in a purely data-driven manner using a holdout validation set consisting of 20\% of the (subsampled) training data. The parameters that achieve the best performance on this validation set are chosen. The whole experimental procedure, including subsampling, noise and label flipping, is repeated over 5 trials to obtain statistically significant results.

\paragraph{Training Neural Networks with HR.} 
In classification problems, training neural networks is minimizing the cross-entropy loss \(\ell(\theta, (X,Y)) = -\sum_{c=1}^C Y^{(c)} \log f^{(c)}_{\theta}(X)\), where \(C = 4\) (corresponding to our four output classes). Here, \(f_{\theta}(X) \in \mathbb{R}^C\) represents the neural network with parameters \(\theta \in \Theta\), producing softmax output probabilities for an input image \(X\) and its one-hot encoded class label \(Y \in \mathbb{R}^C\). The architecture we use follows a standard design in vision tasks: two convolutional layers for feature extraction, followed by fully connected layers to classify inputs into the four classes.

To train robust models, 
we train the neural network by minimizing the HR loss \eqref{eq: HRo predictor} (instead of ERM), i.e., solving \(\min_{\theta \in \Theta} \hat{c}^{\cN,\alpha,r}_{\HRo}(\theta,\Pemp{T})\) based on the cross-entropy loss $\ell$. To do this, we need to compute gradients for each batch of data, given a network with parameters \(\theta\), which we then use for backpropagation. This gradient is computed using the primal finite formulation \eqref{eq:finite-formulation-hro} of the HR predictor, combined with Danskin’s theorem, yielding
\[
\nabla_{\theta} \hat{c}^{\cN,\alpha,r}_{\mathrm{HR}}(\theta,\Pemp{T}) = \sum_{i=1}^T p'^\star_i \nabla_{\theta} \ell(\theta, (X_i',Y_i)) + p'^\star_{T+1}  \nabla_{\theta} \ell(\theta, (X_\infty',Y_\infty))
\]
where \( p^\star \) denotes the optimal solution of the finite formulation of Theorem \ref{thm: HR finite formulation}. Here, the terms \( X_i' \approx \arg\max_{\|X-X_i\|_2 \leq \epsilon} \ell(\theta, (X,Y_i)) \) represent ``adversarial examples'', and also a worst-case sample \((X_\infty', Y_\infty) = \arg\max_{(X_i', Y_i)} \ell(\theta, (X_i', Y_i)) \approx \arg\max_{(X,Y) \in \Sigma} \ell(\theta, (X,Y))\) for a given batch of data $\{X_i,Y_i\}_{i=1\ldots T}$.
Generating these adversarial examples has been widely studied in adversarial training (AT), with several efficient methods available. Here, we use the popular Projected Gradient Descent (PGD) method of \cite{madry2017towards}, which effectively performs a few steps of gradient ascent. 
With these adversarial examples, we can efficiently compute the HR loss by plugging the computed inflated loss \(\ell^{\cN}\) into the finite primal conic optimization problem \eqref{eq:finite-formulation-hro}, obtaining an optimal solution \(p'^\star\). Compared to standard AT, training with the HR loss incurs only a slight computational overhead: solving an additional exponential cone optimization problem of size \(O(T)\) in each batch iteration, where \(T\) is the batch size.
For example, over 100 epochs, HR took an average of only $2\%$ more computation time than regular adversarial training (7.10s vs 6.96s per epoch).


\paragraph{Results: Accuracy and Loss.}
Figure \ref{fig:results} shows the testing loss and accuracy of each DRO formulation. HR significantly outperforms other methods in terms of testing loss and accuracy. Notably, even though Wasserstein DRO provides better testing accuracy than ERM, it exhibits the highest testing loss among all approaches. This may be due to the well-documented ``robust overfitting'' phenomenon in deep learning \cite{rice2020overfitting}, which finds that adversarial training procedures tend to overfit to the adversarial examples on which they are trained. Remarkably, HR corrects this overfitting by the addition of the KL ball, accounting for statistical error, on top of the optimal transport ball, accounting for noise.

While accuracy is typically the primary metric of interest in these applications, a high loss can indicate poor reliability. In the following analysis, we demonstrate the benefits of HR in the context of confidence calibration as a measure of robustness and reliability.

\begin{figure}[H]
    \centering    \includegraphics[width=1\linewidth]{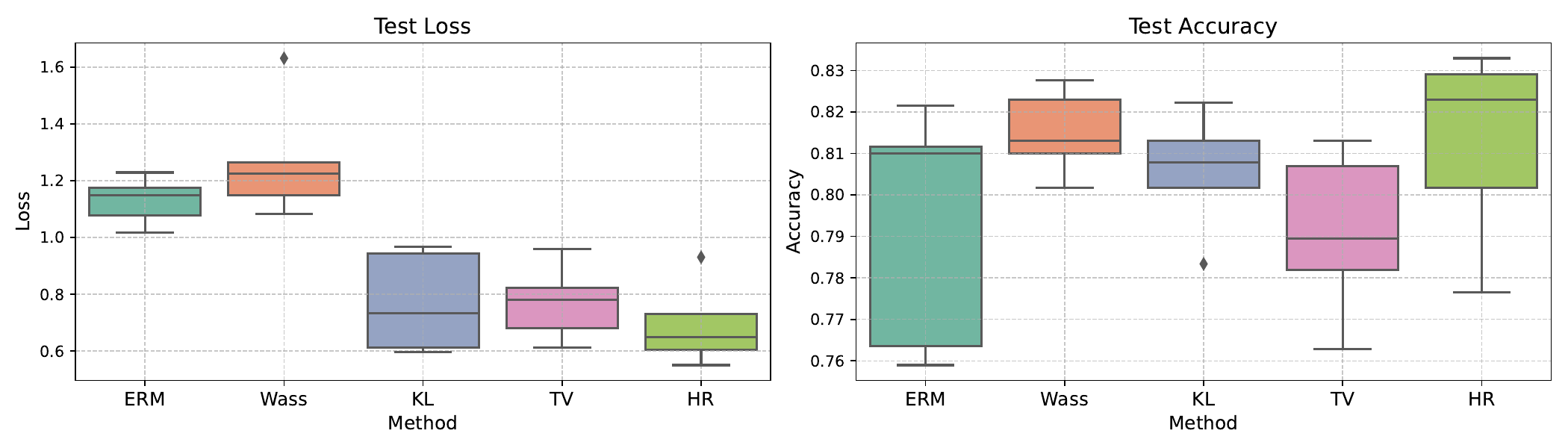}
    \caption{Testing Loss and Accuracy across 5 trials.}
    \label{fig:results}
\end{figure}

\paragraph{Results: Calibration as a Practical Measure of Robustness.}
Calibration, a popular evaluation metric in deep learning \cite{pmlr-v70-guo17a}, assesses how well a model's confidence in an outcome (predicted probabilities of class) aligns with the actual likelihood of that outcome. 
In a well-calibrated model, high confidence should correspond to correct predictions, while lower confidence should reflect uncertainty in a prediction. 
Figure \ref{fig:calibration} visualizes calibration performance across DRO approaches, showing the relationship between model confidence (horizontal axis) and actual accuracy (vertical axis) at each confidence level. These calibration curves  \citep{pmlr-v70-guo17a} provide insight into model reliability. The diagonal dashed line represents perfect calibration, where confidence matches accuracy. ERM and Wasserstein DRO display overconfidence, with their curves lying mostly below the diagonal.
In contrast, KL and HR-DRO approaches closely align with the diagonal, indicating well-calibrated predictions where reported confidence closely matches actual accuracy, thus providing more reliable probability estimates.
Hence, Wasserstein provides high accuracy, but tends to be overconfident in incorrect predictions, while KL has lower accuracy but is typically more reliable (cautious). HR combines the best in both, by providing the highest accuracy while being reliable.


\begin{figure}[H]
    \centering
    \includegraphics[width=1\linewidth]{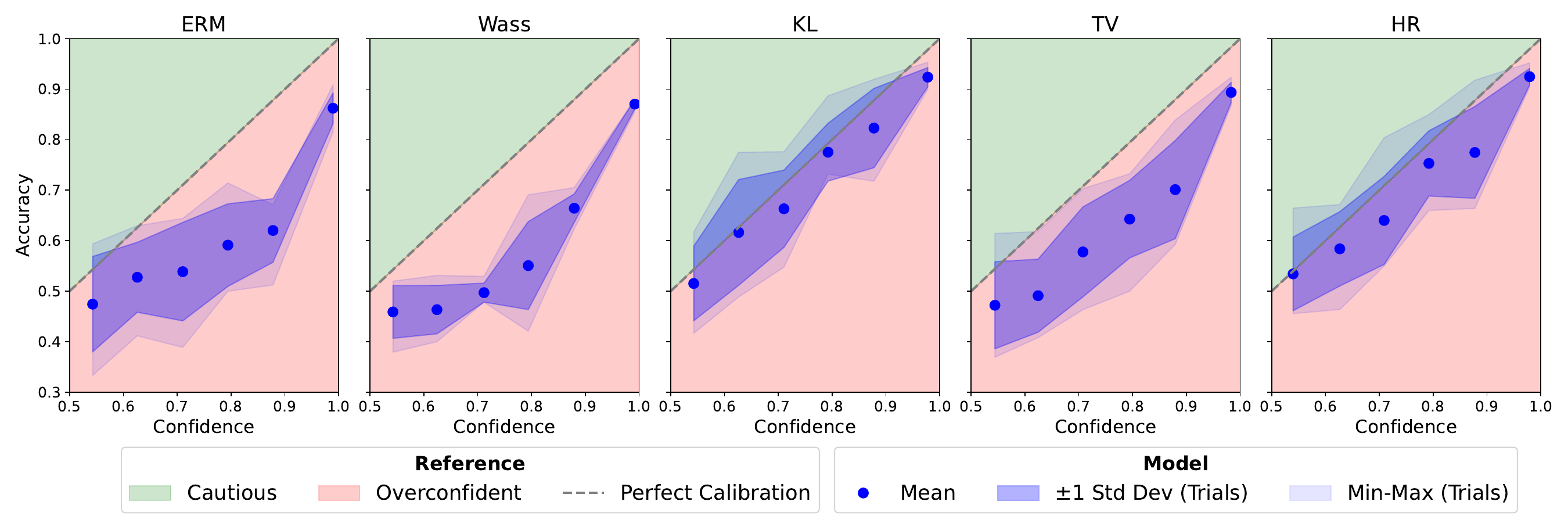}
    \caption{Calibration of DRO approaches. The horizontal axis denotes the confidence of the prediction, while the vertical axis denotes the accuracy at that given confidence level.}
    \label{fig:calibration}
\end{figure}


\subsection{Application to the Portfolio Selection with Real Stock Data}

In this experiment we investigate the \textit{flexibility} of our developed ambiguity set. The intuition is that HR, allows for flexible ambiguity sets which can adapt to a given distribution shift by selecting appropriate hyperparameters parameters through validation.
We use HR, along with other DRO and classical portfolio selection benchmarks, to decide on investment strategies using historical stock data. We then evaluate the models
with two metrics, namely expected return, and risk through the standard deviation of returns.

\subsubsection{Experimental Set-up.}
\paragraph{The Portfolio Selection Problem.} We are provided here with a set of assets $A$ and observe the returns $(\xi_{a,i})_{i\in [T]}$ of each asset $a\in A$ for $T$ times-steps. The goal is to design a portfolio $x \in \cX := \{x \in \Re_+^A \; : \; \sum_{a \in [A]} x_a = 1\}$  enjoying maximum expected return $\Eb[\langle x, \txi \rangle]$ given unknown future random returns $\txi$. The loss function can therefore be taken here as $\loss(x,\xi) = - \langle x, \xi \rangle$ for all $\xi \in \Re^{|A|}$ and $x \in \cX$.

\paragraph{Data.} Our data consists of historical returns of top 100 stocks of the S\&P 500 by market capitalization\footnote{S\&P 500 stocks were selected on \DTMdate{2024-04-07} and market capitalization measured on \DTMdate{2024-05-23}.} among those that have complete return data\footnote{More precisely ``closing price''. This amounts to 243 stocks.} from \DTMdate{1990-01-01} to \DTMdate{2024-03-31}|our complete data range.
Appendix \ref{stock-tickers} lists explicitly the chosen stocks.

Our target is \textit{yearly} returns.
To limit the data size, we record yearly returns every \textit{quarter} in our data $\xi \in \Re^{100\times T}$, where $T=114$ is the number of quarters. That is, $\xi_{a,t} \in \Re$ is the return of stock $a$ when investing on the last trading day of quarter $t$ and selling one year later.

We divide our data into three sets: training (68 quarters, from \DTMdate{1995-09-30} to \DTMdate{2012-09-30}), validation (18 quarters, from \DTMdate{2013-09-30} to \DTMdate{2018-03-31}), and testing (20 quarters, from \DTMdate{2019-03-31} to \DTMdate{2024-03-31}) respectively.
Notice that we keep a 4-quarter gap between each set to avoid information overlap (data leakage).

\paragraph{Experiment.} HR and each benchmark (see Section \ref{sec: benchmarks}) are used to select an investment portfolio as follow:
first, in each model class we train $p=2000$ models associated with different hyperparameter selections on the training set. The investment vector $x^\star$ for all these models associated with different hyper-parameters selections is then evaluated through its expected returns ($\Eb[\langle x^\star, \xi \rangle]$) and risk (standard deviation, $\sqrt{\Var[\langle x^\star, \xi \rangle]}$) on the validation set. We then select those models associated with hyperparameters which are on the Pareto frontier;  that is, models for which there is no other hyperparameter with both better return and risk. The models associated with the Pareto optimal hyperparameters are retrained on the full non-testing data (that is training, 4 quarters of gap and validation). Return and risk of the obtained investment decisions are then evaluated on the test data (after a 4-quarters gap). Considering all Pareto optimal models in each model class simultaneously allows us to judge the performance the model class independently from any specific return-risk tradeoff.

\subsubsection{HR and Benchmarks}\label{sec: benchmarks}
We compare our HR formulation to two DRO formulations and two classical portfolio selection formulations:
\begin{itemize}
    \item \textbf{Type-1 Wasserstein DRO} where we solve the problem $\inf_{x\in\cX}\sup\{ \Eb_{\Pb'}[\loss(x,\txi)] \; : \; \Pb' \in \cP, \; W(\hat{\Pb}_T,\Pb') \leq \epsilon\}$, with parameter $\epsilon\geq 0$, using the 1-Wasserstein distance \eqref{eq: Wasserstein} suggested in \cite{esfahani2015data}. In this particular setting, W-DRO reduces to a linear optimization problem \citep[eq 27]{esfahani2015data}.
    \item \textbf{KL DRO} where we solve the problem $\inf_{x\in\cX}\sup\{ \Eb_{\Pb'}[\loss(x,\txi)] \; : \; \Pb' \in \cP, \; \KL(\hat{\Pb}_T||\Pb') \leq r\}$, with parameter $r\geq 0$. The problem here reduces to an exponential cone optimization problem \citep{van2021data}.
    \item \textbf{Mean-CVaR} where we solve the problem $\inf_{x\in \cX} \Eb_{\hat{\Pb}_T}[\loss(x,\txi)] + \rho \CVaR_{\hat{\Pb}_T}^\gamma [\loss(x,\txi)]$, with parameter $\rho \geq 0$. The Mean-CVaR model can be reduced to a linear optimization problem \citep{rockafellar2000optimization}. Here, we fix $\gamma = 20 \%$. 
    \item \textbf{Markowitz} where we solve the problem 
    $\inf_{x\in \cX} \Eb_{\hat{\Pb}_T}[\loss(x,\txi)] + \rho \Var_{\hat{\Pb}_T}[\loss(x,\txi)]$ \citep{markowitz1952portfolio}. This Mean-Variance problem reduces to a second-order optimization problem.
\end{itemize}

To formulate the HR-DRO model, we take the noise set to be the 1-norm ball $\cN = \cB_1(0,\epsilon)$, $\epsilon\geq 0$. The associated inflated loss is given explicitly as $\ell^\cN(x,\xi) = \ell^\epsilon(x,\xi) = -\langle x,\xi \rangle + \epsilon \| x \|_\infty$, and we take $\max_{\xi \in \Sigma} \loss^{\cN}(x,\xi) = \max_{t \in [T]} \loss^{\cN}(x,\xi_t)$. Hence, following Theorem \ref{thm: dual rep HD} the HR portfolio optimization problem reduces to

\begin{equation*}
  \left\{
    \begin{array}{rl}
      \inf &  \frac 1T \sum_{t \in [T]} w_t + \lambda (r-1) + \beta \alpha + \eta\\[0.5em]
      \st & x\in \Re^A_+, \; w\in \Re^T, \; \lambda \geq 0, \; \beta \geq 0, \; \eta \in \Re,\\[0.5em]
      & \sum_{a \in [A]} x_a = 1, \quad
      \ell_t \geq - \langle x, \xi_t \rangle + \epsilon \|x\|_{\infty}, \; \forall t\in [T], \\[0.5em]
           & w_t \geq \lambda \log \left( \frac{\lambda}{\eta - \ell_t}\right), ~ w_t \geq \lambda \log \left( \frac{\lambda}{\eta - \max_{t'\in [T]}\ell_t'}\right)-\beta \quad \forall t \in [T], \\[0.5em]
           & \eta \geq \max_{t\in [T]}\ell_t
    \end{array}\right.
\end{equation*}
and therefore requires the solution of an exponential cone problem \citep{dahl2021primal}. In this application we chose to use only the parameters $r$ and $\epsilon$, and set $\alpha=0$ as recommended in Section \ref{ssec:hyper-param-select}. For further information see also Appendix \ref{sec: hyperparams}.

\subsubsection{Results}
\paragraph{Return-Risk performance; Parameters chosen through validation.}

Figure \ref{fig: 20240331-3} displays the testing return and risk for all models associated with the Pareto optimal hyperparameters selected using validation. HR mostly dominates other baselines with better risk-return performances, across all the spectrum, with the exception of few extreme low-risk points, where CVaR and Markowitz reach slightly better returns. Importantly, HR is able to reach (while performing well) a wider range of risk-return values. This is in contrast to Wasserstein which reaches mostly the high return but high risk points (upper right side of the figure), and the CVaR/Markowitz/KL which reach mostly the low risk but low return points (bottom left side). This highlights the flexibility of the new introduced ambiguity set compared to the classical benchmarks. The experiments also highlights that the increase in hyper-parameters, while benefiting flexibility, does not impede on validation which allows to successfully choose well performing parameter combinations. We recall that all approaches were trained with the same total number of hyper-parameter choices.

To ensure the significance of our observations, we repeat the whole experiment with different data sets (where dates are shifted back by 1, 2 and 3 quarters) and observe similar results; see Appendix \ref{app:other-performance}.

\paragraph{Achievable Return-Risk Performance.} 
To investigates further the flexibility of the novel ambiguity set, beyond hyper-parameter selection, we plot in Figure \ref{fig: 20240331-1} all the achievable testing risk-return performances by each approach with all the considered 2000 hyper-parameters. HR's Pareto frontier dominates benchmarks, except for a single Wasserstein point achieving return, with high risk. As observed in validation, HR is more flexible, achieving almost all the spectrum of risk-return frontier compared to bechmarks which achieve either high returns with high risks, or low returns with low risks. Hence, the observed flexibility in the previous experiment is not a by-product of validation, but an inherent property of the new ambiguity set.
HR hence provides the decision maker with more choice in terms of robustness-expected performance trade-off. See also Appendix \ref{app:other-performance} for similar plots on other data sets.

\begin{figure}
    \begin{minipage}[b]{0.45\textwidth}
        \includegraphics[width=\textwidth]{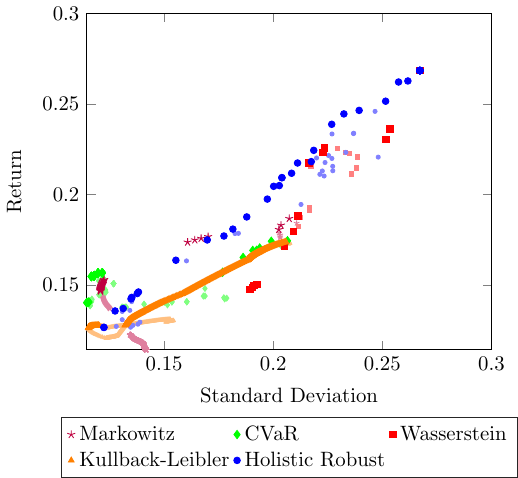}
        \caption{Testing return-risk with hyper-parameters chosen using validation. Points in the Pareto frontier are highlighted.}
        \label{fig: 20240331-3}
    \end{minipage}
    \hfill
    \begin{minipage}[b]{0.45\textwidth}
        \includegraphics[width=\textwidth]{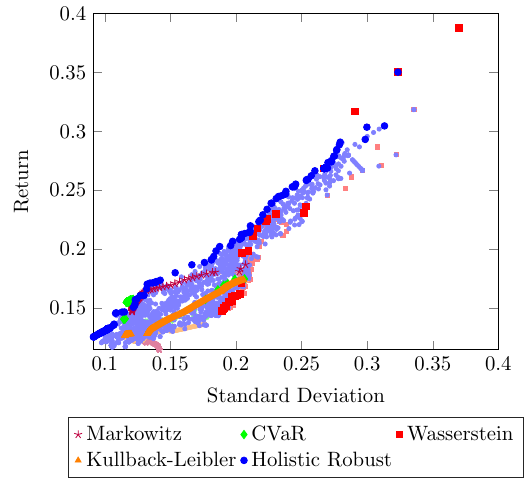}
        \caption{Full achievable testing return-risk with all hyper-parameters. Points in the Pareto frontier are highlighted.}
        \label{fig: 20240331-1}
    \end{minipage}
\end{figure}

\paragraph{Risk Tolerance Investment.} 
To test HR and benchmarks in a realistic investment scenario, we run the following experiment. We set a risk tolerance (bound on the standard deviation) and chose, for each approach, the hyper-parameter that achieves the highest return, while respecting the risk tolerance in \textit{validation}. If no such parameter exists\footnote{Which rarely happens in our experiments.}, we chose the parameter with the lowest risk. We then evaluate this investment in the test set and rank the approaches in terms of their Sharpe Ratio---defined here as average return divided by risk---a popular metric for the quality of investment. 
We repeat this experiment for 10 risk tolerances in a uniform discretization of $[0.05, 0.4]$, and 4 data sets (each shifted by 1,2 or 3 quarters from our original data-set). Table \ref{tab:rank_sr_rv} reports the average ranking of each approach, as well as the average risk tolerance violation \textit{in the test set}. The results show that CVaR, KL, and Markowitz are typically conservative (as expected from results of Figure \ref{fig: 20240331-3} and \ref{fig: 20240331-1}) and respect the risk tolerance more often. However, they largely under perform in terms of Sharpe ratio compared to HR and Wasserstein. Overall, HR most often outperforms other benchmarks with some datasets where HR and Wasserstein rank closely. HR has consistently smaller risk tolerance violation (by roughly 30\%) than Wasserstein. Table \ref{tab:test-violation-formula-2}, \ref{tab:validation-violation-formula-1}, \ref{tab:validation-violation-formula-2} in Appendix \ref{app:other-performance} show details on risk tolerance violation of each approach.


\begin{table}[htbp]
\centering
\setlength{\tabcolsep}{3pt} 
\begin{tabular}{lccccc}
\toprule
Dataset \hspace{0.7cm} & CVaR & Wasserstein & KL & Markowitz & HR \\
\midrule
Original & 3.5 \,(\scriptsize\textcolor{blue}{+1.2}\normalsize) & 2.2 \,(\scriptsize\textcolor{blue}{+4.8}\normalsize) & 4.0 \,(\scriptsize\textcolor{blue}{+1.9}\normalsize) & 3.2 \,(\scriptsize\textcolor{blue}{+2.0}\normalsize) & \textbf{2.1} \,(\scriptsize\textcolor{blue}{+4.5}\normalsize) \\
Shifted 1Q & 3.0 \,(\scriptsize\textcolor{blue}{+1.2}\normalsize) & 2.6 \,(\scriptsize\textcolor{blue}{+5.2}\normalsize) & 3.5 \,(\scriptsize\textcolor{blue}{+2.0}\normalsize) & 3.6 \,(\scriptsize\textcolor{blue}{+2.4}\normalsize) & \textbf{2.3} \,(\scriptsize\textcolor{blue}{+4.3}\normalsize) \\
Shifted 2Q & 3.4 \,(\scriptsize\textcolor{blue}{+1.4}\normalsize) & 2.5 \,(\scriptsize\textcolor{blue}{+4.9}\normalsize) & 3.4 \,(\scriptsize\textcolor{blue}{+1.8}\normalsize) & 3.4 \,(\scriptsize\textcolor{blue}{+1.9}\normalsize) & \textbf{2.3} \,(\scriptsize\textcolor{blue}{+4.5}\normalsize) \\
Shifted 3Q & 3.5 \,(\scriptsize\textcolor{blue}{+1.7}\normalsize) & \textbf{2.2} \,(\scriptsize\textcolor{blue}{+5.0}\normalsize) & 3.6 \,(\scriptsize\textcolor{blue}{+1.8}\normalsize) & 3.4 \,(\scriptsize\textcolor{blue}{+1.7}\normalsize) & 2.3 \,(\scriptsize\textcolor{blue}{+4.3}\normalsize) \\
\bottomrule
\end{tabular}
\caption{Average Sharpe ratio rankings across risk thresholds and ({\color{blue}average risk tolerance violation in the testing set, scaled by $10^2$}). Tolerance violation is defined as the positive part of the testing standard deviation minus the tolerance. Each dataset is shifted back by 0, 1, 2, or 3 quarters from the original dataset.}
\label{tab:rank_sr_rv}
\end{table}

\paragraph{Influence of each HR parameter.} To investigate the effect of each parameter in terms of expected performance and robustness, we plot the testing return, risk and Sharpe ratio of HR as a function of the two hyper-parameters we use ($r,\epsilon$) in Figures \ref{fig:3-in-1-contour}. The first plot (starting from the left) shows that high $\epsilon$ with low $r$ leads to better returns, while the second shows that high $r$ with low or medium $\epsilon$ achieves better risk. While these figures show that one parameter is sufficient to achieve good performance in each metric independently, the third plot, with the ratio of risk and return, shows that both are necessary, as neither parameter is sufficient to reach a good Sharpe ratio, and both must be set to be non-zero for the best performance. Figure \ref{fig:contour-plot-concatenated} in Appendix \ref{app:other-performance} shows a more exhaustive plot including $\alpha$.

\begin{figure}[h]
  \centering
  \includegraphics[width=0.9\linewidth]{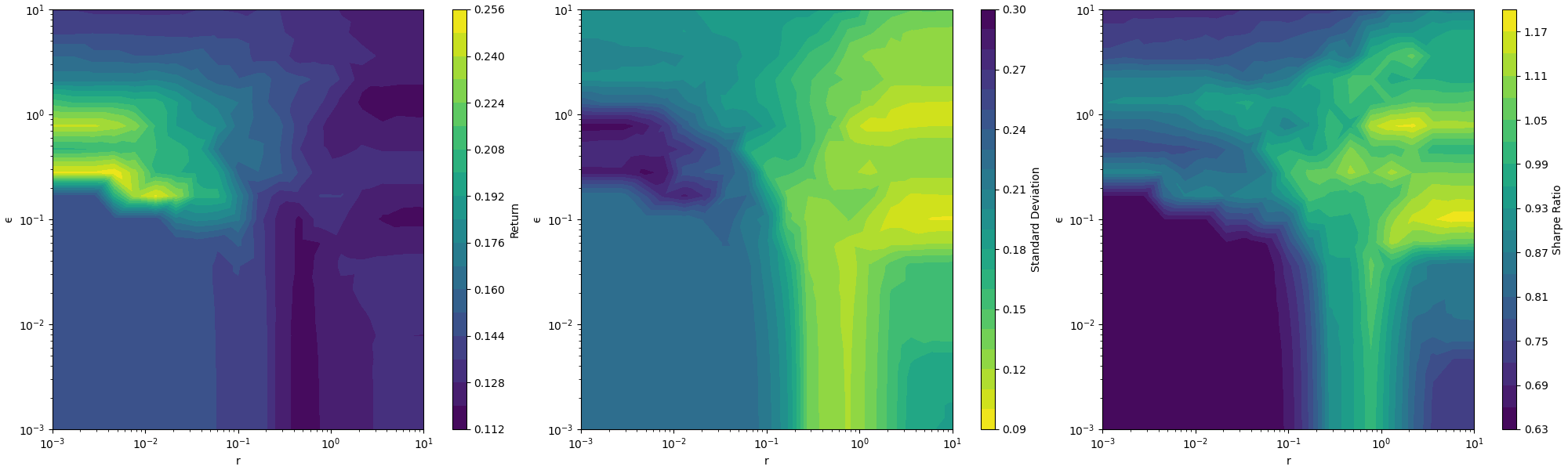} 
  \caption{From left to right: Testing return, standard deviation, and Sharpe ratio as a function of HR parameters, in log-scale. Brighter is better. Trained on data from \DTMdate{1995-03-31} to \DTMdate{2017-09-30} and tested on data from \DTMdate{2018-09-30} to \DTMdate{2023-09-30}. }
  \label{fig:3-in-1-contour}
\end{figure}

\section{Conclusion}

  Taking a step back, we would like to take the time here to emphasize two key methodological points advanced in this work.
  The first point concerns the importance of designing decision methods to meet a priori chosen robustness criteria rather than deriving robustness properties post-hoc. By articulating a specific data model and imposing a desired statistical performance, we are able to measure the robustness of a considered decision method objectively.
  The second point we want to reiterate here is that rather than merely meeting desired robustness criteria, an efficient decision method should do so by incurring as little conservatism as possible. These two key points are surprisingly at odds with how most robust decision methods were developed in the literature. Indeed, outside the few specific cases discussed in this work, we are not aware in which---if in any---natural statistical framework a general $f$-divergence or optimal transport robust decision method offers efficient robust protection.
  Since the release of the first version of this paper, HR has been employed in training deep learning models and was shown to achieve state-of-the-art generalization when facing data corruption in classic deep learning benchmarks \citep{pmlr-v202-bennouna23a}; a feat which we ascribe to the methodological advance by HR along the two aforementioned points.

The methodology presented in this paper can be generalized to settings in which other sources of over-fitting, such as, for instance, the noisy observation model of \citet{van2024efficient, farokhi2023distributionally}, by considering ambiguity sets using an appropriate combination of statistical distances. \cite{liu2023smoothed} shows, for instance, that the L\'evy-Prohorov distance, which serves in this work as a distance describing a distributional shift caused by adversarial noise and contamination, can be substituted for a general transportation distance without the loss of computational tractability.
This modularity of our proposed methodology is particularly beneficial in practical settings where domain-specific corruption must be considered.

Finally, the ambiguity set considered in this paper can be interpreted as a nonparametric minimal confidence set for the unknown out-of-sample distribution. Indeed, \citet{chan2024distributional} shows that by simply intersecting our ambiguity set with a location family of parametric distributions recovers the optimal statistical properties enjoyed by the classical robust statistics developed by \cite{huber1981robust}. This parametric perspective presents a road map on how to bridge the gap between modern robust optimization methods and robust statistical procedures.

\newpage

\bibliography{References}

\newpage
\appendix


\section{Omitted proofs of Section \ref{sec: corruption}}\label{App: proof of sec LP}
\subsection{Definitions and results on the conditional value-at-risk}
We define the scaled conditional value-at-risk of a random variable $C(\txi)$ at a quantile level $\alpha$ as
\begin{align}
  \label{eq:cvar-max-char}(1-\alpha)\CVaR_{\Pb}^\alpha(C(\txi)) \defn &~ \sup \set{\Eb_{\Pb}[C(\txi) R(\txi)]}{\Eb_{\Pb}[R(\txi)] = 1-\alpha, ~R:\Sigma\to [0,1]}\\
  \label{eq:cvar-min-char}= &~\min\, \set{ \beta(1-\alpha)+\Eb_{\Pb}[ \max(C(\txi)-\beta,0)] }{\beta\in \Re};
\end{align}
see for instance \citet{eichhorn2005polyhedral}. Denote the quantile of a random variable $C(\txi)$ with distribution $\Pb$ at level $\alpha$ as
$
\tau_{\alpha} \defn \inf\{ \tau \; : \; \Pb(C(\txi) \leq \tau) \geq \alpha \}.
$
We have from \citet[Theorem 1]{rockafellar2000optimization} that $\tau_{\alpha}$ is a minimizer in \eqref{eq:cvar-min-char} and hence
\begin{align}
  (1-\alpha)\CVaR_{\Pb}^\alpha(C(\txi)) =&~ \tau_{\alpha}(1-\alpha)+\Eb_{\Pb}[ \max(C(\txi)-\tau_{\alpha},0)]\nonumber\\
 \label{eq:cvar-solution} = &~\Eb_{\Pb}[C(\txi)\indic(C(\txi)> \tau_{\alpha})] + \tau_{\alpha}(\Pb(C(\txi) \leq \tau_{\alpha}) -\alpha).
\end{align}

\subsection{Proof of Corollary \ref{cor: LP OOS guarantee}, Theorem \ref{thm: LP optimality}, and Theorem \ref{thm: LP-DRO robustness}}
We first prove Theorem \ref{thm: LP-DRO robustness}.

\begin{proof}[Proof of Theorem \ref{thm: LP-DRO robustness}]
 Denote with $\mathcal C$ the set of possible distributions of the random variable $\txi$ and let
  $
  \mathcal C' = \set{\Pb'\in \cP}{\LP_{\cN}(\Pb^c, \Pb')\ < \alpha}.
  $
  We first show that $\mathcal C\subseteq \mathcal C'$.
  Consider any $\Pb\in \mathcal C$. We have
  \begin{align*}
    \LP_{\cN}(\Pb^c,\Pb)
    &=
      \inf_{\gamma \in \Gamma(\Pb^c,\Pb)} \textstyle\int \indic(\xi-\xi' \not\in \cN) \,\d\gamma(\xi,\xi') \\
    &\leq
      \inf\{\Prob(\txi_1-\txi_2\not\in \cN) \; : \; \text{$\txi_1$ and $\txi_2$ r.v.\ with respective distributions $\Pb^c$ and $\Pb$}\} \\
    &\leq
      \Prob((\txi + \tn)\indic(\tc=0) + \txi_0 \indic(\tc=1)- \txi  \not\in \cN) \\
    &\leq
      \Prob((\txi + \tn)-\txi\not \in \cN)\Prob(\tc=0) + \Prob(\tc=1)\\
    &< \Prob(\tn\not\in\cN)\Prob(\tc=0) + \alpha = \alpha.
  \end{align*}
  and hence $\Pb\in \mathcal C'$.
  
  We now show that $\mathcal C'\subseteq \mathcal C$ as well. Consider any $\Pb\in \mathcal C'$ and hence $\LP_{\cN}(\Pb^c,\Pb) < \alpha$. Consider random variables $\txi^c\sim \Pb^c$ and $\txi\sim \Pb$  with a joint distribution $\bar{\gamma}$ such that 
  $\int \indic(\xi-\xi' \not\in \cN)\,\d\bar{\gamma}(\xi,\xi') < \alpha$.
  This joint distribution exists as 
  $
  \inf\{ \int \indic(\xi-\xi' \not\in \cN)\,\d\gamma(\xi,\xi') \; : \; \gamma \in \Gamma(\Pb^c,\Pb) \} = \LP_{\cN}(\Pb^c,\Pb) < \alpha$.
  Set $\tc \defn \indic(\txi^c - \txi\ \not \in \cN)$ and $\tn \defn \indic(\tc=0) \cdot (\txi^c - \txi) + \indic(\tc=1) n$, for some $n\in \cN\neq \emptyset$, and $\txi_0 \defn \txi^c $. It follows that $\txi^c = (\txi+\tn)\cdot \indic(\tc=0) + \txi_0 \cdot \indic(\tc=1)$. Furthermore, we have $\tn\in \cN$ almost surely and $\Prob(\tc =1) = \int \indic(\xi-\xi' \not\in \cN)\,\d\bar{\gamma}(\xi,\xi') < \alpha$.
  Hence, we have that  $\Pb^c\in\mathcal C$ as well.

\end{proof}

Corollary \ref{cor: LP OOS guarantee} follows immediately. Let us now prove Theorem \ref{thm: LP optimality}.

\begin{proof}[Proof of Theorem \ref{thm: LP optimality}]
  Consider an arbitrary $\Pb^c \in \cP$, $x \in \cX$ and $\Pb \in  \set{\Pb'\in \cP}{\LP_{\cN}(\Pb^c, \Pb') < \alpha}$. Then Theorem \ref{thm: LP-DRO robustness} ensures that there exists an adversary, with the prescribed power corrupting $\Pb$ into $\Pb^c$. Suppose $\hat{c}$ verifies the out-of-sample guarantee. Hence, we must have $\hat{c}(x,\Pb^c) \geq \Eb_{\Pb}[\loss(x,\txi)]$. As $\Pb \in  \set{\Pb'\in \cP}{\LP_{\cN}(\Pb^c, \Pb') < \alpha}$ is arbitrary we must have in fact that $\hat{c}(x,\Pb^c) \geq \sup \{\Eb_{\Pb'}[\loss(x,\txi)] \; : \; \Pb' \in \cP, \; \LP_{\cN}(\Pb^c, \Pb') < \alpha \} = \sup \{\Eb_{\Pb'}[\loss(x,\txi)] \; : \; \Pb' \in \cP, \; \LP_{\cN}(\Pb^c, \Pb') \leq \alpha \} = \hat{c}^{\cN,\alpha}_{\LPtxt}(x,\Pb^c)$.
\end{proof}

\subsection{Generalization and proof of Theorem \ref{thm: LP-DRO expression for empirical}}
\label{sec:gener-theor-refthm}

\begin{theorem}\label{thm: LP-DRO expression}
  For all $x\in\cX$, $\Qb \in \cP$, noise set $\cN$ and $\alpha\in [0,1]$, the following equality holds
  \begin{equation}
    \label{eq: LP-DRO expression}
    \max \set{\Eb_{\Pb'}[\loss(x,\txi)]}{\Pb' \in \cP, \;\LP_\cN(\Qb,\Pb') \leq \alpha}
    = 
    (1-\alpha)\CVaR_{\Qb}^\alpha(\loss^\cN(x,\txi))
    +
    \alpha\max_{\xi\in \Sigma}\loss(x,\xi).
  \end{equation}
  Let $\xi_{\infty} \in \argmax_{\xi \in \Sigma} \loss(x,\xi)$ and define a selection $S(\xi) \in \xi- \argmax_{\noise \in \cN,\,\xi - \noise\in\Sigma } \loss(x,\xi - \noise)$. There exists an optimal $\Pb^\star$ in \eqref{eq: LP-DRO expression} and a coupling $\gamma^\star \in \Gamma(\Qb,\Pb')$ with $\int \indic(\xi-\xi'  \not \in \cN) \,\d\gamma^\star(\xi,\xi') \leq \alpha$ for which we have
  \[
     \Pb^\star(\set{S(\xi)}{\xi \in \supp \Qb} \cup \{ \xi_{\infty}\})=1 ~~ {\rm{and}} ~~
     \gamma^\star(\set{(\xi,S(\xi))}{\xi \in \supp \Qb} \cup \set{(\xi,\xi_{\infty})}{\xi \in \supp \Qb})=1.
  \]
\end{theorem}
\begin{proof}
  We can rewrite the LHS of Equation \eqref{eq: LP-DRO expression} as
  \begin{equation}\label{proof eq: sup with int}
    \sup 
    \left\{ \Eb_{\Pb'}[\loss(x,\txi)] 
      \; : \; 
      \Pb' \in \cP, \;
      \gamma \in \Gamma(\Qb,\Pb'), \;
      \int \indic(\xi-\xi' \not \in \cN) \,\d\gamma(\xi,\xi') \leq \alpha 
    \right\}
  \end{equation}
  We will prove the equality by first constructing a feasible solution to this problem that attains the RHS term in Equation \eqref{eq: LP-DRO expression}. Subsequently, we  then show than any other feasible solution attains an objective smaller than the RHS.

  \textbf{Constructing a feasible solution attaining the RHS.} 
  We start by providing some intuition on the construction that will follow. A feasible solution consists of a coupling $\gamma$ that moves mass from the distribution $\Qb$ into the distribution $\Pb'$. For $(\xi,\xi') \in \Sigma \times \Sigma$, the quantity $\gamma(\xi,\xi')$ represents the amount of mass moved from $\xi$, in the distribution $\Qb$, to $\xi'$ in the distribution $\Pb'$. The maximization problem seeks to move this mass in a way so to maximize the expected loss for the new mass distribution $\Pb'$ under the constraint that at most $\alpha$ mass is moved by more than accounted for in $\cN$. Hence, we are allowed to move $\alpha$ mass of points arbitrary far and $1-\alpha$ with distance bounded in $\cN$. In order to maximize the expectation, we naturally move the $\alpha$ mass of points (that can be moved arbitrarily) into the worst-case event in $\Sigma$ (that maximizes the loss $\loss$), which we denote later by $\xi_{\infty}$. The remaining mass is  moved to the worst-case scenario within the noise set $\cN$.

  Let $\xi_{\infty} \in \argmax \{ \loss(x,\xi) \; : \; \xi \in \Sigma\}$.
  Consider a Bernoulli random variable $\tzeta \in \{0,1\}$ independent of $\txi$ with distribution $\mathbb{Z}$ such that $\mathbb{Z}(\tzeta =1) = (\Qb(\loss^\cN(x,\txi) \leq \tau_{\alpha}) -\alpha)/\Qb(\loss^\cN(x,\txi) = \tau_{\alpha})$ and  $\mathbb{Z}(\tzeta =0) = 1- \mathbb Z(\tzeta =1)$
  for
  $
  \tau_{\alpha} = \inf\{ \tau \in \Re\; : \; \Qb(\loss^\cN(x,\txi) \leq \tau) \geq \alpha \}
  $
  the $\alpha$-quantile of the inflated loss function.
  We denote by $\bar{\Qb}$ the product distribution $\Qb\otimes \mathbb Z$.
  Define the mapping $S_{\cN,\alpha}: \Sigma \times \{0,1\} \to \Sigma$ as
  $$
  S_{\cN,\alpha}(\xi,\zeta)=
  \begin{cases}
    \xi-\argmax \{ \loss(x,\xi - \noise) \; : \; \noise \in \cN,~ \xi-\noise\in \Sigma \}
    \quad 
    &\text{if} \; \loss^\cN(x,\xi) > \tau_\alpha, \\
    \xi-\argmax \{ \loss(x,\xi - \noise) \; : \; \noise \in \cN,~ \xi-\noise\in \Sigma \}
    &\text{if} \; \loss^\cN(x,\xi) = \tau_\alpha \; \text{and} \; \zeta =1,\\
    \xi_{\infty}
    &\text{if} \; \loss^\cN(x,\xi) = \tau_\alpha \; \text{and} \; \zeta =0,\\
    \xi_{\infty} 
    \quad 
    &\text{if} \;  \loss^\cN(x,\xi) < \tau_\alpha. 
  \end{cases}
  $$
  
  Define now a point $(\Pb',\gamma)$ through 
  $$
  \Pb'(A) 
  \defn
  \bar{\Qb}(\{(\xi,\zeta) \in \Sigma \times \{0,1\} \; : \;  S_{\cN,\alpha}(\xi,\zeta) \in A\}) \quad  \forall A \in \cB(\Sigma),
  $$
  and 
  $$
  \gamma(\Pi) \defn \bar{\Qb}\left( \{ (\xi,\zeta) \in \Sigma \times \{0,1\} \; : \; (\xi,S_{\cN,\alpha}(\xi,\zeta)) \in \Pi \}\right)
  \quad \forall \Pi\in \cB(\Sigma\times\Sigma).
  $$
  By definition of $\Pb'$, it is clear that $\supp \Pb' \subseteq \set{\xi-\argmax_{\noise \in \cN, \xi - \noise\in \Sigma} \loss(x,\xi - \noise)}{\xi \in \supp \Qb} \cup \{ \xi_{\infty}\}$.
  We next show that the solution is feasible and its cost is exactly the RHS. 

  To show feasibility we need to have that $\gamma \in \Gamma(\Qb,\Pb')$. For all event $A\in \cB(\Sigma)$, we have 
  $
  \gamma(A\times \Sigma) 
  =
  \bar{\Qb}(\{(\xi,\zeta) \in \Sigma \times \{0,1\} \; : \; (\xi,S_{\cN,\alpha}(\xi,\zeta)) \in A \times \Sigma\}) 
  =
  \Qb(A) 
  $.
  Furthermore
  $
  \gamma(\Sigma \times A) 
  =
  \bar{\Qb}(\{(\xi,\zeta) \in \Sigma \times \{0,1\} \; : \; S_{\cN,\alpha}(\xi,\zeta) \in A) 
  = 
  \Pb'(A)
  $. Hence $\gamma \in \Gamma(\Qb,\Pb')$.
  It also remains to verify the last constraint. We have
  \begingroup
  \allowdisplaybreaks
  \begin{align*}
    \int \indic(\xi-\xi' \not \in \cN) \,\d\gamma(\xi,\xi')
    &=
      \gamma
      \left(
      \{(\xi,\xi') \in  \Sigma \times \Sigma \; : \; \xi-\xi'  \not \in \cN \}
      \right) \\  
    &= 
      \bar{\Qb}\left(
      \{
      (\xi,\zeta) \in \Sigma \times \{0,1\} \; : \; 
      \xi-S_{\cN,\alpha}(\xi,\zeta)\not \in \cN\} \right)\\
    &\leq  
      \bar{\Qb}\left(\loss^\cN(x,\txi) = \tau_\alpha \; \text{and} \; \tilde \zeta =0
      \;
      \text{or}
      \;
      \loss^\cN(x,\txi) < \tau_\alpha
      \right) 
      \numberthis \label{proof eq: penalisation 1} \\
    &=
      \bar{\Qb}(\loss^\cN(x,\txi) = \tau_\alpha \; \text{and} \; \tilde \zeta =0) 
      +
      \bar{\Qb}\left(\loss^\cN(x,\txi) < \tau_\alpha
      \right) \\
    & =
      \Qb(\loss^\cN(x,\txi) = \tau_\alpha)\mathbb{Z}(\tilde \zeta =0) 
      +
      \Qb\left(\loss^\cN(x,\txi) < \tau_\alpha
      \right) 
      \numberthis \label{proof eq: penalisation 2}\\
    &= \alpha 
  \end{align*}
  \endgroup
  where \eqref{proof eq: penalisation 1} is by definition of $S_{\cN,\alpha}$ and \eqref{proof eq: penalisation 2}  is justified by the independence between $\txi$ and $\tilde \zeta$. The ultimate equality is by definition of $\mathbb{Z}(\zeta =0)$.

  Let us now compute the cost of this solution.  We can decompose this cost as
  \begin{align*}
    \Eb_{\Pb'}[\loss(x,\txi')]  & = \Eb_{\gamma}[\loss(x,\txi')] = \Eb_{\bar{\Qb}}[\loss(x,S_{\cN,\alpha}(\txi,\tzeta))]\\
    &= 
      \Eb_{\bar{\Qb}}[\loss(x,S_{\cN,\alpha}(\txi,\tzeta)) \indic(\loss^\cN(x,\txi)> \tau_{\alpha})]
      +
      \Eb_{\bar{\Qb}}[\loss(x,S_{\cN,\alpha}(\txi,\tzeta)) \indic(\loss^\cN(x,\txi) = \tau_\alpha \; \text{and} \; \tzeta =1)]\\
    &\quad+
      \Eb_{\bar{\Qb}}[\loss(x,S_{\cN,\alpha}(\txi,\tzeta)) \indic(\loss^\cN(x,\txi) = \tau_\alpha \; \text{and} \; \tzeta =0, \; \text{or} \;\loss^\cN(x,\txi) < \tau_\alpha)]
      \numberthis \label{proof eq: penalisation 3}
  \end{align*}
  Let us now examine the second and third terms. By definition of $S_{\cN,\alpha}$ the second term can be written as
  \begin{align*}
    \Eb_{\bar{\Qb}}[\loss(x,S_{\cN,\alpha}(\txi,\tzeta)) \indic(\loss^\cN(x,\txi) = \tau_\alpha \; \text{and} \; \tilde \zeta =1)]
    &=
      \Eb_{\bar{\Qb}}[\loss^\cN(x,\txi) \indic(\loss^\cN(x,\txi) = \tau_\alpha \; \text{and} \; \tilde \zeta =1)]\\
    &=
      \tau_{\alpha}\Eb_{\bar{\Qb}}[ \indic(\loss^\cN(x,\txi) = \tau_\alpha \; \text{and} \; \tilde \zeta =1)]\\
    &= 
      \tau_{\alpha}\bar{\Qb}( \loss^\cN(x,\txi) = \tau_\alpha \; \text{and} \; \tilde \zeta =1)\\
    &=
      \tau_{\alpha}\Qb( \loss^\cN(x,\txi) = \tau_\alpha)\mathbb{Z}(\tilde \zeta =1)\\
    &= 
      \tau_{\alpha}
      (\Qb(\loss^\cN(x,\txi) \leq \tau_{\alpha}) -\alpha)
  \end{align*}
  where the last two equalities are by the independence between $\txi$ and $\tilde \zeta$ and by the definition of $\mathbb{Z}(\tilde \zeta =1)$. We now examine the third term. By definition of $S_{\cN,\alpha}$, we have
  \begingroup
  \allowdisplaybreaks
  \begin{align*}
    \Eb_{\bar{\Qb}}[\loss(x,S_{\cN,\alpha}(\txi,\tzeta)) &\indic(\loss^\cN(x,\txi) = \tau_\alpha \; \text{and} \; \zeta =0, \; \text{or} \;\loss^\cN(x,\txi) < \tau_\alpha)]\\
                                                              &= 
                                                                \max_{\xi\in \Sigma}\loss(x,\xi)
                                                                \Eb_{\bar{\Qb}}[\indic(\loss^\cN(x,\xi) = \tau_\alpha \; \text{and} \; \zeta =0, \; \text{or} \;\loss^\cN(x,\txi) < \tau_\alpha)]\\
                                                              &= 
                                                                \max_{\xi\in \Sigma}\loss(x,\xi)
                                                                \left[
                                                                \bar{\Qb}(\loss^\cN(x,\txi) = \tau_\alpha \; \text{and} \; \zeta =0) 
                                                                +
                                                                \bar{\Qb}(\loss^\cN(x,\txi) < \tau_\alpha)
                                                                \right]\\
                                                              &= 
                                                                \max_{\xi\in \Sigma}\loss(x,\xi)
                                                                \left[
                                                                \Qb(\loss^\cN(x,\txi) = \tau_\alpha )\mathbb{Z}(\zeta =0) 
                                                                +
                                                                \Qb(\loss^\cN(x,\txi) < \tau_\alpha)
                                                                \right] \\
                                                              &=
                                                                \max_{\xi\in \Sigma}\loss(x,\xi)
                                                                \left[
                                                                \Qb(\loss^\cN(x,\txi) = \tau_\alpha )
                                                                -
                                                                \Qb(\loss^\cN(x,\txi) \leq \tau_{\alpha}) +\alpha
                                                                +
                                                                \Qb(\loss^\cN(x,\txi) < \tau_\alpha)
                                                                \right] \\
                                                              &= 
                                                                \alpha\max_{\xi\in \Sigma}\loss(x,\xi)
  \end{align*}
  \endgroup
  where the two last equalities are by independence of $\tzeta$ and $\txi$ and definition of $\mathbb{Z}(\tilde \zeta =0)$. Plugging the values of the second and third term in \eqref{proof eq: penalisation 3}, we get exactly the RHS after observing Equation \eqref{eq:cvar-solution}.

  \textbf{Proving LHS$\leq$RHS.} We show that the cost of every feasible solution of the supremum problem \eqref{proof eq: sup with int} is smaller than the RHS. Let $\Pb'\in \cP$ and $\gamma \in \Gamma(\Qb,\Pb')$ be a feasible solution to the supremum problem \eqref{proof eq: sup with int}. Let $\gamma(\{ (\xi, \xi')\in \Sigma\times\Sigma \; :\; \xi-\xi' \not \in \cN \}) = \alpha'\leq \alpha$. As $\gamma$ has marginals $\Qb$ and $\Pb'$ we have
  \begin{equation*}
    \Eb_{\Pb'}[\loss(x,\txi')]
    =
      \Eb_{\gamma}[\loss(x,\txi')]=
      \Eb_{\gamma}[\loss(x,\txi') \indic(\txi-\txi' \in \cN)]
      +
      \Eb_{\gamma}[\loss(x,\txi') \indic(\txi-\txi'\not \in \cN)]
  \end{equation*}
  Let us bound each of the two terms. We start by the second term. We have
  \begin{align*}
    \Eb_{\gamma}[\loss(x,\txi') \indic(\txi-\txi' \not \in \cN)]
    \leq
    \gamma(\{ (\xi,\xi') \in \Sigma\times\Sigma \; :\; \xi'-\xi \not \in \cN \})\max_{\xi \in \Sigma} \loss(x,\xi)
    =
    \alpha'\max_{\xi \in \Sigma} \loss(x,\xi).
  \end{align*}
  We now turn to the first term. When $\txi-\txi' \in \cN$, we have $\loss^{\cN}(x,\txi) = \max_{\noise \in \cN,\,\txi - n\in \Sigma } \loss(x,\txi - n) \geq \loss(x,\txi')$.
  Hence,
  \begin{align*}
    \Eb_{\gamma}[\loss(x,\txi') \indic(\txi-\txi' \in \cN)]
    &\leq
      \Eb_{\gamma}[\loss^\cN(x,\txi) \indic(\txi-\txi' \in \cN)]\\
    &\leq
      \sup_{R:\Sigma\times\Sigma\to [0,1] }\{
      \Eb_{\gamma}[\loss^\cN(x,\txi) R(\txi,\txi')] \; : \; \Eb_{\gamma}[R(\txi,\txi')] = 1-\alpha' \}
  \end{align*}
  where the $\sup$ is taken over measurable functions.
  The second inequality follows from $\Eb_{\gamma}[\indic(\txi-\txi' \in \cN)] =1- \gamma(\{ (\xi, \xi')\in \Sigma\times\Sigma \; :\; \xi-\xi' \not \in \cN \}) = 1-\alpha'$.
  Furthermore,
  \begin{align*}
    & \Eb_{\gamma}[\loss(x,\txi') \indic(\txi-\txi' \in \cN)]\\
    & \leq \sup_{R:\Sigma\times\Sigma\to [0,1] }\{
      \Eb_{\gamma}[\loss^\cN(x,\txi) R(\txi,\txi')] \; : \; \Eb_{\gamma}[R(\txi,\txi')] = 1-\alpha'\}\\
    & = \sup_{R:\Sigma\times\Sigma\to [0,1] }\{
      \Eb_{\Qb}[\loss^\cN(x,\txi)  \Eb_{\gamma}[R(\txi,\txi')|\txi]] \; : \; \Eb_{\Qb}(\Eb_{\gamma}(R(\txi,\txi')|\txi)) = 1-\alpha'\}\\
    & = \sup_{R':\Sigma\to [0,1] }\{
      \Eb_{\Qb}(\loss^\cN(x,\txi) R'(\txi)) \; : \; \Eb_{\Qb}(R'(\txi)) = 1-\alpha'\}\\
    &= (1-\alpha')\CVaR_{\Qb}^{\alpha'}(\loss^\cN(x,\txi)).
  \end{align*}

  Here the first equality holds by the law of total expectation and the fact that $\loss^\cN(x,\txi)$ does not depend on $\txi'$.
  The second inequality follows from the change of variables $R'(\txi) = \Eb_{\gamma}(R(\txi,\txi')|\txi))$. That is, the random variable $R'$ denotes the conditional expectation of the random variable $R$ given $\txi$.
  The last equality is by definition of the conditional value-at-risk.
  Combining the bound on the first and second term, we get 
  \begin{align*}
    \Eb_{\Pb'}(\loss(x,\txi')) \leq  &
    (1-\alpha')\CVaR_{\Qb}^{\alpha'}(\loss^\cN(x,\txi))
    +
                                 \alpha'\max_{\xi\in \Sigma}\loss(x,\xi)\\
    \leq &\max_{\alpha'\leq \alpha}~
    (1-\alpha')\CVaR_{\Qb}^{\alpha'}(\loss^\cN(x,\txi))
    +
           {\alpha'}\max_{\xi\in \Sigma}\loss(x,\xi)\\
    \leq &
    (1-\alpha)\CVaR_{\Qb}^{\alpha}(\loss^\cN(x,\txi))
    +
           {\alpha}\max_{\xi\in \Sigma}\loss(x,\xi)
  \end{align*}
  which is exactly the RHS.
\end{proof}

\subsection{Proof of Corollary \ref{corollary:interpretation}}

\begin{proof}[Proof of Corollary \ref{corollary:interpretation}]
Consider an ordering of the observed samples $\xi_{[1]},\ldots,\xi_{[T]}$ by increasing inflated loss so that $\loss^\cN(x,\xi_{[1]}) \leq \loss^\cN(x,\xi_{[2]}) \leq \ldots \leq \loss^\cN(x,\xi_{[T]})$.
We have
\begingroup
\allowdisplaybreaks
\begin{align*}
  & \max \set{\textstyle\frac{1}{T}\sum_{t\in[T]} \loss(x,\xi_t-\noise_t-\noise'_t)}{\sum_{t\in[T]}\tfrac{\one{\noise'_t\neq 0}}{T}\leq \alpha,~\noise_t \in \mathcal N, ~ \xi_t-\noise_t-\noise'_t\in \Sigma~~\forall t\in [T]}\\
  = & \max \set{\textstyle\frac{1}{T}\sum_{t\in[T]} \loss^\cN(x,\xi_t-\noise'_t)}{\sum_{t\in[T]}\tfrac{\one{\noise'_t\neq 0}}{T}\leq \alpha,~\noise_t \in \mathcal N, ~ \xi_t-\noise'_t\in \Sigma~~\forall t\in [T]} \\
  = & \textstyle\sum_{t=\alpha T+1}^{T} \tfrac{\loss^\cN(x,\xi_{[t]})}{T}   +
      \alpha\max_{\xi\in \Sigma}\loss(x,\xi)\\
  = & \textstyle
      (1-\alpha)\CVaR_{\Pemp{T}}^{\alpha}(\loss^\cN(x,\txi))
      +
      \alpha\max_{\xi\in \Sigma}\loss(x,\xi)
      =   \hat{c}^{\mathcal N, \alpha}_{\LPtxt}(x,\Pemp{T}).
\end{align*}
\endgroup
The first equality follows from the definition of the inflated loss function as $\ell^\cN(x, \xi) = \max_{n\in \cN, \xi-n\in \Sigma}\ell(x, \xi-n)$.
The second equality use the fact that $\alpha T$ is integer. The penultimate equality follows from Equation \eqref{eq:cvar-solution}.
The final equality follow from Theorem \ref{thm: LP-DRO expression for empirical}.
\end{proof}

\section{Omitted proofs of Section \ref{sec:HR-pred}}\label{App: proof of sec det corr}
\subsection{Proof of Theorem \ref{thm: HD robustness guarantee}}\label{App: proof robustness HD}
We first prove the following key lemma.

\begin{lemma}
  Let $\hat{\Pb}_T$ be the empirical distribution of $T$ independent samples with distribution $\Pb$ on a compact set $\Sigma$. Then for all $\delta >0$ we have
  $$
  \Prob\left(
    \exists \Pb'\in \cP, \; \LP_{\cB(0,\delta)}(\hat{\Pb}_T, \Pb')\leq \delta,\, \KL(\Pb',\Pb) \leq r
  \right)
  \geq 
  1 - \left( \frac{4}{\delta} \right)^{m(\Sigma,\delta)} \exp(-rT)
  $$
  where $m(\Sigma, \delta) := \min \{ k \geq 0 \; : \; \exists \xi_1, \ldots, \xi_k \in \Sigma \; \st \; \cup_{i=1}^k \cB(\xi_i,\delta) \supseteq \Sigma \}$ denotes the internal covering number of the support set $\Sigma$.
\end{lemma}\label{lemma: Uniform-KL-inflated-with-LP}
\begin{proof}[Proof of Theorem \ref{thm: HD robustness guarantee}]
  We have
  \begin{align*}
    &\Prob(\exists \Pb'\in \cP \; \st \; \LP_{\cB(0,\delta)}(\hat{\Pb}_T,\Pb') \leq \delta, \; \KL(\Pb', \Pb) \leq r ) \\
    & = 
      \Prob(\exists \Pb'\in\cP \; \st \; \LPdist(\hat{\Pb}_T,\Pb') \leq \delta, \; \KL(\Pb', \Pb) \leq r \} ) \\
    &=
      \Prob( \hat{\Pb}_T \in \{ \hat \Pb \in \cP :  \exists \Pb' \in \cP \; \st \; \LPdist(\hat{\Pb},\Pb') \leq \delta, \; \KL(\Pb', \Pb) \leq r \} ) \\
    &= 
      1- \Prob( \hat{\Pb}_T \in \cA)
  \end{align*}
  where $\cA$ is defined as $\cA^c = \{ \hat \Pb \in \cP : \exists \Pb' \in \cP \; \st \; \LPdist(\hat \Pb,\Pb') \leq \delta, \; \KL(\Pb', \Pb) \leq r\}$. We will show using results from \citet{dembod1996large} that $\Prob(\hat{\Pb}_T \in \cA ) \leq (4/\delta)^{m(\Sigma,\delta)} \exp(-T \inf_{\Pb' \in \cA^\delta} \KL(\Pb',\Pb))$, where $\cA^\delta = \{ \Pb' \in \cP \; : \; \LPdist(\Pb',\Pb'') \leq \delta, \; \Pb'' \in \cA \}$ is the $\delta$-inflation of the set $\cA$, by the LP distance $\LPdist$. This last inequality immediately leads to the conclusion by remarking that $\Pb' \in \cA^\delta \implies \KL(\Pb',\Pb) >r$. Indeed, suppose that there exists $\Pb'' \in \cA^\delta$ with $\KL(\Pb'',\Pb) \leq r$. By definition of $\cA^\delta$, there exists $\Pb' \in \cA$ such that $\LPdist(\Pb'',\Pb') = \LPdist(\Pb',\Pb'') \leq \delta$. As we have both $\LPdist(\Pb',\Pb'') \leq \delta$ and $\KL(\Pb'',\Pb) \leq r$ which implies that $\Pb' \in \cA^c$, a contradiction.

  Let us now show that 
  $\Prob(\hat{\Pb}_T \in \cA) \leq (4/\delta)^{m(\Sigma,\delta)} \exp(-T \inf_{\Pb' \in \cA^\delta} \KL(\Pb',\Pb))$.
  Denote $\cB^{\cA}_{\LPtxt}(\Pb') := \{ \Pb'' \in \cP \; : \; \LPdist(\Pb'',\Pb') \leq \delta\}$ the LP ball inside $\cA$, which is compact as $\LPdist$ is continuous in the weak topology and $\Sigma$ is compact \cite{prokhorov1956convergence}. \citet[Exercise 4.5.5]{dembod1996large} established that for any set $\cA$  and $\delta>0$, we have for all $n \geq 1$ the upper bound
  $$
  \Prob(\hat{\Pb}_T \in \cA) \leq m_{\LPtxt}(\cA,\delta) \exp\left( -T \inf_{\Pb' \in \cA^\delta} \KL(\Pb', \Pb) \right)
  $$
  where $m_{\LPtxt}(\cA,\delta) = \min \{ k \geq 0 \; : \; \exists \Pb_1, \ldots, \Pb_k \in \cA \; \st \; \cup_{i=1}^k \cB_{\LPtxt}(\Pb_i,\delta) \supseteq \cA \}$ denotes the internal covering number of the set $\cA$ of interest with LP balls of radius $\delta$.
  \citet[Exercise 6.2.19]{dembod1996large} upper bounds the  covering number  of any set $\cA$ in terms of the  covering number of the event set $\Sigma$ as
  \(
  m_{\LPtxt}(\mathcal A, \delta)\leq m_{\LPtxt}(\mathcal P, \delta)\leq \left(\tfrac{4}{\delta}\right)^{m(\Sigma ,\delta)}
  \), for all $\delta >0$. 
\end{proof}

\begin{proof}[Proof of Theorem \ref{thm: HD robustness guarantee}]
Denote with $\cN'$ the compact set in the interior of $\cN$ where the noise realizes and $\alpha'<\alpha$ the fraction of data points misspecified.
There exists $\delta >0$ such that $\cN' + \cB(0,\delta) \subset \cN$ and $\alpha' + \delta \leq \alpha$. Denote with $\hat{\Qb}_T$ the empirical distribution of clean samples from $\Pb$ before corruption. Theorem \ref{thm: LP-DRO robustness} ensures that 
$\LP_{\cN'}(\hat{\Pb}_T,\hat{\Qb}_T) \leq \alpha'$. Furthermore, Lemma \ref{lemma: Uniform-KL-inflated-with-LP} ensures that 
$$
\Prob(\Pb \in \tset{\Pb''\in \cP}{\exists \Pb' \in \cP \; \st \; \LP_{\cB(0,\delta)}(\hat{\Qb}_T,\Pb') \leq \delta, \; \KL(\Pb',\Pb'') \leq r} ) \geq 1 - (4/\delta)^{m(\Sigma,\delta)} \exp(-rT).
$$
Combining these two results yields 
\begin{align*}
  & ~1 - (4/\delta)^{m(\Sigma,\delta)} \exp(-rT)\\
  \leq & ~ \Prob(\Pb \in \tset{\Pb''\in \cP}{\exists \Pb'\in \cP,\, \hat{\Qb} \in \cP \; \st \; \LP_{\cN'}(\hat{\Pb}_T,\hat{\Qb}) \leq \alpha', \; \LP_{\cB(0,\delta)}(\hat{\Qb},\Pb') \leq \delta, \; \KL(\Pb',\Pb'') \leq r }) \\
  = & ~\Prob(\Pb \in \tset{\Pb''\in \cP}{\exists \Pb'\in \cP \; \st \; \LP_{\cN' + \cB(0,\delta)}(\hat{\Pb}_T,\Pb') \leq \alpha + \delta, \; \KL(\Pb',\Pb'') \leq r} ). 
\end{align*}
The final equality follows from  $\mathcal U'\defn \tset{\Pb''\in \cP}{\exists \Pb'\in \cP,\, \hat{\Qb} \in \cP \; \st \; \LP_{\cN'}(\hat{\Pb}_T,\hat{\Qb}) \leq \alpha', \; \LP_{\cB(0,\delta)}(\hat{\Qb},\Pb') \leq \delta, \; \KL(\Pb',\Pb'') \leq r} = \tset{\Pb''\in \cP}{\exists \Pb'\in \cP \; \st \; \LP_{\cN' + \cB(0,\delta)}(\hat{\Pb}_T,\Pb') \leq \alpha + \delta, \; \KL(\Pb',\Pb'') \leq r}$.
Hence it follows that
\begin{align*}
    \Prob\left(
        \sup_{\Pb' \in \cU'} \Eb_{\Pb'}[\loss(x,\txi)] \geq \Eb_{\Pb}[\loss(x,\txi)], \; \forall x \in \cX
    \right)
    \geq 1 - (4/\delta)^{m(\Sigma,\delta)} \exp(-rT).
\end{align*}
we get then the desired result by noticing that $\hat{c}^{\cN,\alpha,r}_{\HR}(x,\Pemp{T}) \geq \sup_{\Pb' \in \cU'} \Eb_{\Pb'}[\loss(x,\txi)]$
as $\cU'$ is a subset of the ambiguity set of $\hat{c}^{\cN,\alpha,r}_{\HR}$.
\end{proof}

\subsection{Proof of Theorem \ref{thm:efficiency-hd}}
\begin{proof}[Proof of Theorem \ref{thm:efficiency-hd}]
    Suppose for the sake of contradiction that there exists a predictor $\hat{c}$ verifying the out-of-sample guarantee and such that there exists $x_0$ and $\hat{\Pb}_T \in \cP$ for some $T$ such that
    \begin{equation}\label{proof eq: conservativness}
    \hat{c}(x_0,\hat{\Pb}_T) < \hat{c}^{\cN,\alpha,r}_{\HR}(x_0,\hat{\Pb}_T).
    \end{equation}
    Denote $\cD = \{\xi_1,\ldots,\xi_T\}$ the data points which constitute the empirical distribution $\hat{\Pb}_T$. Without loss of generality, we may assume $\loss^\cN(x_0,\xi_1) \leq \ldots \leq \loss^\cN(x_0,\xi_T)$.
    Denote $\epsilon := \hat{c}^{\cN,\alpha,r}_{\HR}(x_0,\hat{\Pb}_T) - \hat{c}(x_0,\hat{\Pb}_T) >0$. 
    
    The goal is to construct a data generation process and a malicious adversary (with the prescribed power) for which $\hat{c}$ does not verify the out-of-sample guarantee thereby reaching a contradiction. We will first show that the inequality \eqref{proof eq: conservativness} is still verified with slightly perturbed parameters $\alpha$ and $\cN$. These perturbed parameters will be convenient to construct our malicious adversary.

    Recall that 
    $$
    \hat{c}^{\cN,\alpha,r}_{\HR}(x_0,\Pemp{T})
      =
      \sup \{ 
      \Eb_{\Pb'}[\loss(x_0,\txi)]
      \; : \; 
      \Pb',\hat{\Qb} \in \cP, \;
      \LP_{\cN}(\Pemp{T}, \hat{\Qb}) \leq \alpha, \;
      \KL (\hat{\Qb}||\Pb') \leq r
      \} \quad \forall x \in \cX.
    $$
    Remark that the function $\hat{c}^{\cN,\alpha,r}_{\HR}(x_0,\hat{\Pb}_T)$ is continuous in the parameters $\alpha$ and $r$. Indeed, both the pseudo divergence $\LP_\eta$ and the $\KL$ divergence are jointly convex in the optimization variables $\Pb'$ and $\hat{\Qb}$ and the objective function $\Eb_{\Pb'}[\loss(x_0,\txi)]$ is linear, which implies that $\hat{c}^{\cN,\alpha,r}_{\HR}(x_0,\hat{\Pb}_T)$ is concave in $(\alpha, r)\in \Re^2_+$. As concave functions are continuous in the interior of their domains, $\hat{c}^{\cN,\alpha,r}$ is continuous at $r>0$ and $\alpha>0$. Hence, we can consider $0<r'<r$ and $0<\alpha'<\alpha$, with $\alpha'$ rational so that $\hat{c}^{\cN,\alpha',r'}_{\HR}(x_0,\hat{\Pb}_T) \geq  \hat{c}^{\cN,\alpha,r}_{\HR}(x_0,\hat{\Pb}_T) - \epsilon/4$. Consider an integer $k \geq 1$ such that $Tk\alpha'$ is integer and denote $T' = kT$. Note that $\hat{\Pb}_T$ can also be interpreted as an empirical distribution of $T'$ data points (by duplicating all data points $k$ times) and our adversary can misspecify up to exactly $T'\alpha' \in \integ$ data points, which will be useful later on.

    We now build our data generating distribution $\Pb^\star$ and our adversary.
    Let us consider a worst-case corrupted version of our new dataset defined as
    $$
    \cD' = \{\xi'_i := \xi_i + \argmax_{\delta \in \cN, \xi_i + \delta \in \Sigma} \loss(x_0,\xi_i + \delta)\}_{i=1}^T \cup \{\xi'_{\infty} \in \argmax_{\xi \in \Sigma} \loss(x_0,\xi)\}
    $$
    which is well defined as $\xi \to \loss(x_0,\xi)$ is continuous and $\cN$ is compact. Each of the points $\xi'_i$ is a perturbed version of $\xi_i$ by adversarial noise, and $\xi'_\infty$ is a worst-case data point maximizing the loss. As the empirical distribution $\hat{\Pb}_T$ is finitely supported, from Lemma \ref{lemma:hr:primal:reduction} and Corollary \ref{corollary:interpretation}, it follows that we can consider the maximizers in the HD problem, which write as
    $$
    \hat{\Pb}_T^{\cN,\alpha'} = \sum_{i=\tau +1}^T \delta_{\xi_i'} \Pemp{T}(\xi_i) + (1-\alpha'-\sum_{i=\tau +1}^T \Pemp{T}(\xi_i)) \delta_{\xi'_\tau} + \alpha' \delta_{\xi'_\infty}
    $$
    with $\tau$ the smallest integer such that $1-\alpha' - \sum_{i=\tau+1}^T \Pemp{T}(\xi_i) >0$, and $\Pb^\star$ verifying $\supp(\Pb^\star) \subseteq \cD'$ such that
    $$
    \LP_{\cN}(\Pemp{T},\hat{\Pb}_T^{\cN,\alpha'}) \leq \alpha', \;
    \KL(\hat{\Pb}_T^{\cN,\alpha'},\Pb^\star) \leq r', \; 
    \hat{c}^{\cN,\alpha',r'}_{\HR}(x_0,\Pemp{T}) = \Eb_{\Pb^\star}[\loss(x_0,\txi)].
    $$
    We have $\supp(\hat{\Pb}_T^{\cN,\alpha'}) \subseteq \cD'$, and as $\alpha' T'$ is integer, it follows that $\hat{\Pb}_T^{\cN,\alpha'}(\xi')\in \frac{1}{T'}\integ$ for all $\xi' \in \cD'$. This implies that $\hat{\Pb}_T^{\cN,\alpha'}$ can also be seen as the empirical distribution of $T'$ points in $\Sigma$. 
    
    Finally, to complete building our adversary, we perturb slightly the support of $\Pb^\star$ and $\supp(\hat{\Pb}_T^{\cN,\alpha'})$. As $\xi \to \loss(x_0,\xi)$ is continuous and $\cl(\interior(\mathcal N))=\mathcal N$, we can find points $\xi''_i$ such that $\xi_i'' - \xi_i \in \interior(\cN)$ so that for $\bar{\Pb}^\star$ defined as  $\bar{\Pb}^\star(\xi_i'') = \Pb^\star(\xi'_i)$ for all $i$ and  $\bar{\Pb}^\star(\xi'_\infty) = \Pb^\star(\xi'_\infty)$, we have
    $\Eb_{\bar{\Pb}^\star}[\loss(x_0,\txi)] + \epsilon/4 \geq \Eb_{\Pb^\star}[\loss(x_0,\txi)] = \hat{c}^{\cN,\alpha',r'}_{\HR}(x_0,\hat{\Pb}_T)$. Here, we move each point $\xi'_i$ of $\cD'$ (the support of $\Pb^\star$) into a close point $\xi_i''$. Define $\cN' = \{0\} \cup \{\xi_i'' - \xi_i\} \subset \interior(\cN)$. Define similarly the distribution $\hat{\Pb}_T^{\cN',\alpha'}$ through $\hat{\Pb}_T^{\cN',\alpha'}(\xi''_i) = \hat{\Pb}_T^{\cN,\alpha'}(\xi'_i)$ for all $i$ and $\hat{\Pb}_T^{\cN',\alpha'}(\xi_\infty) = \hat{\Pb}_T^{\cN,\alpha'}(\xi_\infty)$. 
   The distributions $\hat{\Pb}_T^{\cN',\alpha'}$ and $\bar{\Pb}^\star$ preserve the properties of $\hat{\Pb}_T^{\cN,\alpha'}$ and $\Pb^\star$;
    as $\cN' \subset \cN$, we have $\LP_{\cN}(\Pemp{T},\hat{\Pb}_T^{\cN',\alpha'}) \leq \LP_{\cN}(\Pemp{T},\hat{\Pb}_T^{\cN,\alpha'}) \leq \alpha'$ and as the support of $\hat{\Pb}_T^{\cN',\alpha'}$ and $\bar{\Pb}^\star$ moved in exactly the same way, we have $\KL(\hat{\Pb}_T^{\cN',\alpha'},\bar{\Pb}^\star) = \KL(\hat{\Pb}_T^{\cN,\alpha'},\Pb^\star) \leq r$. Hence, our constructed distributions verify
    $$
    \LP_{\cN}(\Pemp{T},\hat{\Pb}_T^{\cN',\alpha'}) \leq \alpha', \;
    \KL(\hat{\Pb}_T^{\cN',\alpha'},\bar{\Pb}^\star) \leq r',
    $$
    and
    \begin{equation}\label{proof eq: disap-ineq}
    \Eb_{\bar{\Pb}^\star}[\loss(x_0,\txi)] \geq \hat{c}^{\cN,\alpha',r'}_{\HR}(x_0,\Pemp{T}) - \epsilon/4 \geq \hat{c}^{\cN,\alpha,r}_{\HR}(x_0,\Pemp{T}) - \epsilon/2 > \hat{c}(x_0, \hat{\Pb}_T).
    \end{equation}
    
    We are now ready to construct our data generation process and our adversary.
    Consider the data generation process of out-of-sample distribution $\bar{\Pb}^\star$ and denote $\hat{\Qb}_{t}^\star$ its empirical distribution for all $t \in \integ$. In particular, $\hat{\Qb}_{t}^\star$ is a random variable. As $\hat{\Pb}_T^{\cN',\alpha'}$ shares the same support as $\bar{\Pb}^\star$, and $\hat{\Pb}_T^{\cN',\alpha'}(\xi) \subset \frac{1}{T'} \integ$, the distribution $\hat{\Pb}_T^{\cN',\alpha'}$ can be seen as a potential realization of the empirical distribution of $\hat{\Qb}_{t}^\star$. As $\LP_{\cN'}(\hat{\Pb}_T,\hat{\Pb}_T^{\cN',\alpha'}) \leq \alpha'$, there exists an adversary following Corollary that can perturb $\hat{\Pb}_T^{\cN',\alpha'}$ into $\hat{\Pb}_T$ by noise limited to the set $\cN'$ and misspecfication of less than $\alpha'$. For this adversary, denote $\hat{\Pb}_t^\star$ the corrupted distribution from the clean samples $\hat{\Qb}_t^\star$ for all $t$.
    We have
    \begin{align*}
        \Prob\left(\hat{c}(x_0,\hat{\Pb}^\star_{mT'}) < \Eb_{\bar{\Pb}^\star}[\loss(x_0,\txi)] \right)
        \geq 
        \Prob\left(
        \hat{\Pb}^\star_{mT'} 
        =
        \hat{\Pb}_T
        \right)
        \geq
        \Prob\left(
        \hat{\Qb}_{mT'}^\star = \hat{\Pb}_T^{\cN',\alpha'}
        \right), \; \forall m \geq 1
    \end{align*}
    Indeed, for any data size $mT'$, if the empirical distribution from $\bar{\Pb}^\star$ realizes as $\hat{\Pb}_T^{\cN',\alpha'}$, then the considered adversary can corrupt $\hat{\Qb}^\star_{mT'} = \hat{\Pb}_T^{\cN',\alpha'}$ into 
    $\hat{\Pb}^\star_{mT'} 
    =
    \hat{\Pb}_T$. Then, inequality \eqref{proof eq: disap-ineq} ensures that 
    $
    \Eb_{\bar{\Pb}^\star}[\loss(x_0,\txi)] 
    > 
    \hat{c}(x_0, \hat{\Pb}_T) = \hat{c}(x_0,\hat{\Pb}^\star_{mT'})
    $.

    If suffices now to lower bound the probability of the event $\hat{\Qb}_{mT'}^\star = \hat{\Pb}_T^{\cN',\alpha'}$.
    As $\bar{\Pb}^\star$ is a discrete distribution supported on less than $T+1$ points, and as $\hat{\Pb}_T^{\cN',\alpha'}(\xi) \subset \frac{1}{T'} \integ$, Theorem 11.1.4 in \citet{cover1991information} implies
    $$
    \Prob(\hat{\Qb}_{mT'}^\star = \hat{\Pb}_T^{\cN',\alpha'})
    \geq 
    \frac{1}{(mT'+1)^{T+1}} 
    \exp\left(
    -mT'\KL(\hat{\Pb}_T^{\cN',\alpha'},\bar{\Pb}^\star)
    \right)
    \geq 
    \frac{1}{(mT'+1)^{T+1}} 
    \exp\left(
    -mT'r'
    \right), \; \forall m\geq 1.
    $$
    Taking $m\to \infty$ leads then to
    $$
    \limsup_{t \to \infty} \frac{1}{t} \log \Prob \left( \hat{c}(x_0,\hat{\Pb}^\star_t) < \Eb_{\Pb}[\loss(x_0,\txi)]\right) \geq -r' > -r,
    $$
    a contradiction with $\hat{c}$ veriying the out-of-sample guarantee for out-of-sample distribution $\bar{\Pb}^\star$ and the considered adversary.
\end{proof}

\subsection{Proof of Theorem \ref{thm: HD finite formulation}}

\begin{proof}[Proof of Theorem \ref{thm: HD finite formulation}]
  Following Theorem \ref{thm: dual rep HD} and the change of variables $w_k'=w_k-\lambda$ we have
  \begin{align}
    \label{eq:dual-cov}
    \hat{c}^{\cN,\alpha,r}_{\HR}(x,\Pemp{T}) & = \left\{
                                                       \begin{array}{rl}
                                                         \inf & \sum_{k\in[K]} w'_k \Pemp{T}(\xi_k) + \lambda r + \beta \alpha+ \eta\\[0.5em]
                                                         \st &  w\in \Re^K,\; \lambda \geq 0, \; \beta \geq 0,\; \eta \in \Re, \\[0.5em]
                                                              & w'_k \geq \lambda \left(\log \left( \frac{\lambda}{\eta - \loss^{\cN}(x,\xi_k)}\right) -1 \right) \quad \forall k \in [K],\\[0.5em]
                                                              & w'_k \geq  \lambda \left(\log \left( \frac{\lambda}{\eta - \max_{\xi \in \Sigma} \loss(x,\xi)}\right)-1\right) -\beta \quad \forall k \in [K],\\[0.5em]
                                                              & \eta \geq \max_{\xi \in \Sigma} \loss(x,\xi).
                                                       \end{array}\right.
  \end{align}
  Introduce now the associated Lagrangian function
  \begin{align*}
     L(w', \lambda, \beta, \eta; \hat q,  s)
    \defn & \textstyle\sum_{k\in[K]} w'_k \Pemp{T}(\xi_k) + \lambda r + \beta \alpha+ \eta + \sum_{k\in [K]} \hat q_k\left( \lambda \left(\log \left( \frac{\lambda}{\eta - \loss^{\cN}(x,\xi_k)}\right) -1 \right) - w'_k \right)\\
    & \textstyle\quad + \sum_{k\in [K]}  s_k\left(\lambda \left(\log \left( \frac{\lambda}{\eta - \max_{\xi \in \Sigma} \loss(x,\xi)}\right)-1\right) -\beta-w_k'\right)\\
    = & \textstyle \sum_{k\in[K]} w'_k\left( \Pemp{T}(\xi_k) - \hat q_k - s_k \right) + \beta \left(\alpha - \sum_{k\in [K]} \hat s_k \right)+\lambda r +\eta\\
    & \textstyle \quad +\sum_{k\in[K]} \hat q_k  \lambda \left(\log \left( \frac{\lambda}{\eta - \loss^{\cN}(x,\xi_k)}\right) -1 \right) +\sum_{k\in[K]} s_k \lambda \left(\log \left( \frac{\lambda}{\eta - \max_{\xi \in \Sigma} \loss(x,\xi)}\right)-1\right)
  \end{align*}
  and dual function
  \begin{align*}
    & g(\hat q,  s)\\
    \defn & \inf \set{L(w, \lambda, \beta, \eta; \hat q,  s)}{w'\in \Re^K, \lambda\in \Re_+, \beta\in \Re_+, \eta\geq \max_{\xi \in \Sigma} \loss(x,\xi)}\\
    = & \textstyle \sum_{k\in[K]}\chi_{-\infty}(\Pemp{T}(\xi_k) = \hat q_k + s_k) + \chi_{-\infty}(\sum_{k\in [K]}  s_k\leq \alpha) \\
    & \quad +\textstyle
      \left\{
      \begin{array}{rl}
        \inf & \lambda r +\eta+\sum_{k\in[K]} \hat q_k  \lambda \left(\log \left( \frac{\lambda}{\eta - \loss^{\cN}(x,\xi_k)}\right) -1 \right) +\sum_{k\in[K]} s_k \lambda \left(\log \left( \frac{\lambda}{\eta - \max_{\xi \in \Sigma} \loss(x,\xi)}\right)-1\right) \\
        \st & \lambda\in \Re_+, \eta\geq \max_{\xi \in \Sigma} \loss(x,\xi)
      \end{array}
      \right.
  \end{align*}
  where $\chi_{-\infty}(S)=0$ if $S$ is true and $-\infty$ otherwise.
  We remark that the minimization problem in Equation \eqref{eq:dual-cov} satisfies the Slater constraint qualification condition. 
  As Slater's constraint qualification condition is met we have strong duality and the dual optimal value is attained \citep[Proposition 5.3.1]{bertsekas2009convex}. Hence,
  \begingroup
  \allowdisplaybreaks
  \begin{align}
    & \hat{c}^{\cN,\alpha,r}_{\HR}(x,\Pemp{T})\\
    = & \max\set{g(\hat q, s)}{\hat q \in \Re^K_+,s \in \Re^K_+}\nonumber\\
    = & \left\{
        \begin{array}{rl}
          \max & \inf \{ \lambda r +\eta+\sum_{k\in[K]} \hat q_k  \lambda \left(\log \left( \frac{\lambda}{\eta - \loss^{\cN}(x,\xi_k)}\right) -1 \right)  \\
               & \hspace{3em} +\sum_{k\in[K]} s_k \lambda \left(\log \left( \frac{\lambda}{\eta - \max_{\xi \in \Sigma} \loss(x,\xi)}\right)-1\right) : \eta\geq \max_{\xi \in \Sigma} \loss(x,\xi), ~\lambda\in\Re_+ \}\\
          \st & \hat q \in \Re^K_+,s \in \Re^K_+,\\
               & \Pemp{T}(\xi_k) = \hat q_k + s_k \quad \forall k\in [K],\\
               &\sum_{k\in [K]}  s_k\leq \alpha
        \end{array}
        \right.\nonumber\\
    = & \left\{
        \begin{array}{rl}
          \max & \inf \{ \lambda r +\eta+\sum_{k\in[K]} \hat q_k  \sup_{p'_k\geq 0} (\loss^{\cN}(x,\xi_k) - \eta)\frac{p'_k}{\hat q_k} +\lambda\log\left(\frac{p'_k}{\hat q_k}\right)  \\
               & \hspace{3em} +\left(\sum_{k\in[K]} s_k\right) \sup_{p'_{K+1} \geq 0} (\max_{\xi \in \Sigma} \loss(x,\xi) - \eta)\frac{p'_{K+1}}{\sum_{k\in[K]} s_k} +\lambda\log\left(\frac{p_{K+1}'}{\sum_{k\in[K]} s_k}\right) \\
               & \hspace{6em} : \eta\geq \max_{\xi \in \Sigma} \loss(x,\xi), ~\lambda\in\Re_+ \}\\
          \st & \hat q \in \Re^K_+,s \in \Re^K_+,\\
               & \Pemp{T}(\xi_k) = \hat q_k + s_k \quad \forall k\in [K],\\
               &\sum_{k\in [K]}  s_k\leq \alpha
        \end{array}
        \right.\nonumber\\
    = & \left\{
        \begin{array}{rl}
          \max & \inf \{ \lambda r +\eta+\sum_{k\in[K]} \hat q_k \left[(\loss^{\cN}(x,\xi_k) - \eta)\frac{p'_k}{\hat q_k} +\lambda\log\left(\frac{p'_k}{\hat q_k}\right) \right] \\
               & \hspace{3em} +\left(\sum_{k\in[K]} s_k\right) \left[(\max_{\xi \in \Sigma} \loss(x,\xi) - \eta)\frac{p'_{K+1}}{\sum_{k\in[K]} s_k} +\lambda\log\left(\frac{p_{K+1}'}{\sum_{k\in[K]} s_k}\right) \right] \\
               & \hspace{6em}: \eta\geq \max_{\xi \in \Sigma} \loss(x,\xi), ~\lambda\in\Re_+ \}\\
          \st & \hat q \in \Re^K_+,s \in \Re^K_+, p'\in \Re_+^{K+1},\\
               & \Pemp{T}(\xi_k) = \hat q_k + s_k \quad \forall k\in [K],\\
               &\sum_{k\in [K]}  s_k\leq \alpha
        \end{array}
        \right.\nonumber\\
    = & \left\{
        \begin{array}{rl}
          \max & \sum_{k=1}^N \ell^\cN(\xi_k) p_k' +  (1-\sum_{k\in[K]}p'_{k})\max_{\xi\in\Sigma}\ell(x, \xi)\\
          \st & \hat q \in \Re^K_+,s \in \Re^K_+, p'\in \Re_+^{K+1},\\
               & \Pemp{T}(\xi_k) = \hat q_k + s_k \quad \forall k\in [K],\\
               & \sum_{k\in[K]}p'_k + p_{K+1}\leq 1,\\
               &\sum_{k\in [K]}  s_k\leq \alpha,\\          
               & \sum_{k\in [K]} \hat q_k \log(\frac{\hat q_k}{p'_k})+(\sum_{k\in[K]} s_k)\log\left(\frac{\sum_{k\in[K]} s_k}{p_{K+1}}\right)\leq r.
        \end{array}\right.\nonumber\\
        = & \left\{
        \begin{array}{rl}
          \max & \sum_{k=1}^N \ell^\cN(\xi_k) p_k' +  p_{K+1}\max_{\xi\in\Sigma}\ell(x, \xi)\\
          \st & \hat q \in \Re^K_+,s \in \Re^K_+, p'\in \Re_+^{K+1},\\
               & \Pemp{T}(\xi_k) = \hat q_k + s_k \quad \forall k\in [K],\\
               & \sum_{k\in[K+1]}p'_k= 1,\\
               &\sum_{k\in [K]}  s_k\leq \alpha,\\          
               & \sum_{k\in [K]} \hat q_k \log(\frac{\hat q_k}{p'_k})+(\sum_{k\in[K]} s_k)\log\left(\frac{\sum_{k\in[K]} s_k}{p_{K+1}}\right)\leq r.
        \end{array}\right.\nonumber\\
    = & \left\{
        \begin{array}{rl}
          \max & \sum_{k=1}^N \ell^\cN(\xi_k) p_k' +  p_{K+1}\max_{\xi\in\Sigma}\ell(x, \xi) \\
          \st & \hat q \in \Re^{K+1}_+,s \in \Re^K_+, p'\in \Re_+^{K+1},\\
               & \Pemp{T}(\xi_k) = \hat q_k + s_k \quad \forall k\in [K],\\
               & \sum_{k\in[K+1]}p'_k=1,\\
               & \sum_{k\in[K+1]}\hat q_k=1,\\
               &\sum_{k\in [K]}  s_k\leq \alpha,\\
               & \sum_{k\in[K+1]} \hat q_k \log(\frac{\hat q_k}{p'_k})\leq r.
        \end{array}\right.\nonumber
  \end{align}
  \endgroup
  Here the third equality follows from the standard convex conjugate identifies. First, observe that
  \(
    \sup_{r\geq 0} a r + \log(r) = -(1+\log(-a))
  \)
  for $a\leq 0$. Hence, for every $k\in[K]$ we have
  \[
    \textstyle-\lambda\left(1+\log\left(-\frac{\eta-\loss^{\cN}(x,\xi_k)}{\lambda} \right)\right) = \sup_{r_k\geq 0} (\eta-\loss^{\cN}(x,\xi_k))r_k+\lambda\log(r)
  \]
  and
  \[
    \textstyle-\lambda\left(1+\log\left(-\frac{\eta-\max_{\xi\in\Sigma}\ell(x, \xi)}{\lambda} \right)\right) = \sup_{r_{K+1}\geq 0} (\eta-\max_{\xi\in\Sigma}\ell(x, \xi))r_{K+1}+\lambda\log(r)
  \]
  from which the second equality follows where we choose $r_k=\tfrac{p'_k}{\hat q_k}$ and $r_{K+1}=\tfrac{p'_{K+1}}{(\sum_{k\in[K]} s_k)}$ for some vector $p\in\Re_+^{K+1}$. The fourth equality follows readily from a standard minimax theorem of \citet[Proposition 5.5.7]{bertsekas2009convex} with respect to the saddle point function
  \begin{align*}
    L(\eta, \lambda, p') \defn & \textstyle\lambda r +\eta+\sum_{k\in[K]} \hat q_k \left[(\loss^{\cN}(x,\xi_k) - \eta)\frac{p'_k}{\hat q_k} +\lambda\log\left(\frac{p'_k}{\hat q_k}\right)\right]\\
                               & \textstyle\quad +\left(\sum_{k\in[K]} s_k\right) \left[(\max_{\xi \in \Sigma} \loss(x,\xi) - \eta)\frac{p'_{K+1}}{\sum_{k\in[K]} s_k} +\lambda\log\left(\frac{p_{K+1}'}{\sum_{k\in[K]} s_k}\right) \right]\\
    = & \textstyle \lambda\left(r - \sum_{k\in[K]}\hat q_k \log\left(\frac{\hat q_k}{p'_k} \right)- (\sum_{k\in[K]} s_k)\log\left(\frac{\sum_{k\in[K]} s_k}{p'_{K+1}}\right)\right)+\eta(1-\sum_{k\in[K+1]} p'_k)\\
    & \textstyle\quad +\sum_{k\in[K]}p_k\loss^{\cN}(x,\xi_k)+p_{K+1}\max_{\xi \in \Sigma} \loss(x,\xi)
  \end{align*}
  where we may exploit $r>0$. The penultimate inequality follows from the fact that we may choose without loss of optimality
  \(
    p_{K+1} = 1-\sum_{k\in [K]}p_k
  \)
  as the logarithm is an increasing function. The final equality follows similarly from defining $q_{K+1}=1-\sum_{k\in[K]} q_k$ and observing that we have
  \(
  \textstyle \sum_{k\in K}\Pemp{T}(\xi_k) = 1 = \sum_{k\in K}\hat q_k + \sum_{k\in K}s_k.
  \)
\end{proof}

\subsection{Proof of Lemma \ref{lemma:hr:primal:reduction}}

  \begin{proof}[Proof of Lemma \ref{lemma:hr:primal:reduction}]
    First, observe that
    \begingroup
    \allowdisplaybreaks
    \begin{align*}
      & \hat{c}^{\cN,\alpha,r}_{\HR}(x,\Pemp{T})\\
      \defn & \sup \set{\Eb_{\Pb'}[\loss(x,\txi)]}{\Pb'\in \cP,\; \hat{\Qb} \in \cP, \; \LP_{\cN}(\Pemp{T}, \hat{\Qb}) \leq \alpha, \; \KL (\hat{\Qb}||\Pb') \leq r}\\
      = & \sup\set{\sup \set{\Eb_{\Pb'}[\loss(x,\txi)]}{\Pb'\in \cP, \;\KL (\hat{\Qb}||\Pb') \leq r}}{\hat{\Qb} \in \cP, ~\LP_{\cN}(\Pemp{T}, \hat{\Qb}) \leq \alpha}\\
      = & \sup\Bigg\{\inf \set{\int \lambda \log\left(\frac{\lambda}{\eta - \loss(x,\xi)}\right)\, \d \hat{\Qb}(\xi)+(r-1)\lambda+\eta}{\lambda\geq 0,~\eta \geq \max_{\xi \in \Sigma} \loss(x,\xi)}\\
      & \hspace{30em} :\; \hat{\Qb} \in \cP, ~\LP_{\cN}(\Pemp{T}, \hat{\Qb}) \leq \alpha\Bigg\}\\
      \leq & \inf \Bigg\{\sup \set{\int \lambda \log\left(\frac{\lambda}{\eta - \loss(x,\xi)}\right)\, \d \hat{\Qb}(\xi)+(r-1)\lambda+\eta}{\hat{\Qb} \in \cP, ~\LP_{\cN}(\Pemp{T}, \hat{\Qb}) \leq \alpha}\\
      & \hspace{30em} : \; \lambda\geq 0,~\eta \geq \max_{\xi \in \Sigma} \loss(x,\xi) \Bigg\}.
    \end{align*}
    \endgroup
    Here, the second equality follows from Lemma \ref{lemma:kl-solution}. Second, we remark that for any $\lambda\geq 0$ and $\eta \geq \max_{\xi \in \Sigma} \loss(x,\xi)$ we have that
    \(
    a \mapsto \lambda \log\left(\tfrac{\lambda}{(\eta - a)}\right)
    \)
    in an increasing function. In particular, denote with $\{\xi_1,\ldots,\xi_K\}$ the support of $\Pemp{T}\in \cP$ ordered to that $\loss^\cN(x, \xi_1)\leq \dots \leq \loss^\cN(x, \xi_K)$ then clearly we also have $\log\left(\tfrac{\lambda}{(\eta - \loss^\cN(x, \xi_1))}\right)\leq \dots \leq \log\left(\tfrac{\lambda}{(\eta - \loss^\cN(x, \xi_K))}\right)$. Following now Theorem \ref{thm: LP-DRO expression for empirical} we have for all $\lambda\geq 0$ and $\eta \geq \max_{\xi \in \Sigma} \loss(x,\xi)$ that
    \[
      \sup \set{\int \!\lambda \log\left(\frac{\lambda}{\eta - \loss(x,\xi)}\right) \d \hat{\Qb}(\xi)\!}{\!\hat{\Qb} \in \cP, \,\LP_{\cN}(\Pemp{T}, \hat{\Qb}) \leq \alpha} = \E{\Pemp{T}^{\cN, \alpha}\!\!}{\lambda \log\left(\frac{\lambda}{\eta - \loss(x,\txi)}\right)}.
    \]
    with worst-case distribution $\Pemp{T}^{\cN, \alpha}\in \cP$ satisfying $\LP_{\cN}(\Pemp{T}, \Pemp{T}^{\cN, \alpha})\leq \alpha$ and given explicitly in Equation \eqref{eq:wc-distribution-noise-misspecification}. We observe that the  worst-case distribution $\Pemp{T}^{\cN, \alpha}\in \cP$ is not a function of $\eta$ nor $\lambda$.
    \citet[Proposition 5]{van2021data} now implies that
    \begin{align*}
      \hat{c}^{\cN,\alpha,r}_{\HR}(x,\Pemp{T}) \leq & ~\textstyle \inf \set{\E{\Pemp{T}^{\cN, \alpha}}{\lambda \log\left(\frac{\lambda}{\eta - \loss(x,\txi)}\right)} + (r-1)\lambda+\eta} {\lambda\geq 0,~\eta \geq \max_{\xi \in \Sigma} \loss(x,\xi) }\\
      \leq &~ \hat{c}^r_{\mathrm{KL}}(x,\Pemp{T}^{\cN,\alpha}) = \E{\Pemp{T}^{\cN,\alpha, r}}{\ell(x, \txi)}.
    \end{align*}
    for some worst-case distribution $\Pemp{T}^{\cN,\alpha, r}\in \cP$ satisfying $\KL (\Pemp{T}^{\cN, \alpha}||\Pemp{T}^{\cN, \alpha, r}) \leq r$. Clearly, as our defined worst-case distributions $\Pemp{T}^{\cN, \alpha}\in \cP$ and $\Pemp{T}^{\cN,\alpha, r}\in \cP$ satisfy $\LP_{\cN}(\Pemp{T}, \Pemp{T}^{\cN, \alpha})\leq \alpha$ and $\KL (\Pemp{T}^{\cN, \alpha}||\Pemp{T}^{\cN, \alpha, r}) \leq r$, respectively. Hence, the distributions $\hat\Qb=\Pemp{T}^{\cN, \alpha}$ and $\Pb'=\Pemp{T}^{\cN, \alpha, r}$ are feasible in the maximization problem characterizing the HD predictor in Equation (\ref{eq: HD predictor}). Hence, we have
    $$\hat{c}^{\cN,\alpha,r}_{\HR}(x,\Pemp{T}) \geq \E{\Pemp{T}^{\cN,\alpha, r}}{\ell(x, \txi)}= \hat{c}^r_{\mathrm{KL}}(x,\Pemp{T}^{\cN,\alpha})$$
    from which the claim follows immediately.
  \end{proof}

\subsection{Generalization and Proof of Theorem \ref{thm: dual rep HD}}
\label{App: Dual HD}

\begin{theorem}[General Dual Formulation]
  Let $\hat{\Pb} \in \cP$. For all $x \in \cX$, the HD robust predictor \eqref{eq: HD predictor} admits for all $r>0$ the dual representation
  \begin{align*}
    \hat{c}^{\cN,\alpha,r}_{\HR}(x,\hat{\Pb}) =
    \left\{
    \begin{array}{rl}
      \inf & \int w(\xi) \,\d \hat{\Pb}(\xi) + \lambda (r-1) + \beta \alpha+ \eta\\[0.5em]
      \st &  w: \Sigma \rightarrow \Re,\; \lambda \geq 0, \; \beta \geq 0,\; \eta \in \Re, \\[0.5em]
           & w(\xi) \geq \lambda\log \left( \frac{\lambda}{\eta - \loss^{\cN}(x,\xi)}\right),~ w(\xi) \geq  \lambda \log \left( \frac{\lambda}{\eta - \max_{\xi' \in \Sigma} \loss(x,\xi')}\right) -\beta \quad \forall \xi \in \Sigma.
    \end{array}\right.
  \end{align*}
\end{theorem}

\begin{proof}
  We have that
  \begingroup
  \allowdisplaybreaks
  \begin{align*}
    & \hat{c}^{\cN,\alpha,r}_{\HR}(x,\hat{\Pb})\\
    \defn & \sup \set{\Eb_{\Pb'}[\loss(x,\txi)]}{\Pb'\in \cP,\; \hat{\Qb} \in \cP, \; \LP_{\cN}(\hat{\Pb}, \hat{\Qb}) \leq \alpha, \; \KL (\hat{\Qb}||\Pb') \leq r}\\
    = & \sup\set{\sup \set{\Eb_{\Pb'}[\loss(x,\txi)]}{\Pb'\in \cP, \;\KL (\hat{\Qb}||\Pb') \leq r}}{\hat{\Qb} \in \cP, ~\LP_{\cN}(\hat{\Pb}, \hat{\Qb}) \leq \alpha}\\
    = & \sup\Bigg\{\inf \set{\int \lambda \log\left(\frac{\lambda}{\eta - \loss(x,\xi)}\right)\, \d \hat{\Qb}(\xi)+(r-1)\lambda+\eta}{\lambda\geq 0,~\eta \geq \max_{\xi \in \Sigma} \loss(x,\xi)}\\
    & \hspace{30em} :\; \hat{\Qb} \in \cP, ~\LP_{\cN}(\hat{\Pb}, \hat{\Qb}) \leq \alpha\Bigg\}\\
    = & \lim_{\epsilon \downarrow 0}\sup\Bigg\{ \inf \set{\int \lambda \log\left(\frac{\lambda}{\eta - \loss(x,\xi)}\right)\, \d \hat{\Qb}(\xi)+(r-1)\lambda+\eta}{\lambda\geq 0,~\eta \geq \max_{\xi \in \Sigma} \loss(x,\xi)+\epsilon} \\
    & \hspace{30em} : \; \hat{\Qb} \in \cP, ~\LP_{\cN}(\hat{\Pb}, \hat{\Qb}) \leq \alpha \Bigg\}\\
    = & \lim_{\epsilon \downarrow 0} \inf \Bigg\{\sup \set{\int \lambda \log\left(\frac{\lambda}{\eta - \loss(x,\xi)}\right)\, \d \hat{\Qb}(\xi)+(r-1)\lambda+\eta}{\hat{\Qb} \in \cP, ~\LP_{\cN}(\hat{\Pb}, \hat{\Qb}) \leq \alpha}\\
    & \hspace{27.8em} : \; \lambda\geq 0,~\eta \geq \max_{\xi \in \Sigma} \loss(x,\xi) +\epsilon \Bigg\}
  \end{align*}
  Here, the second equality follows from Lemma \ref{lemma:kl-solution}.
  The third inequality follows from the fact that
  \begin{align*}
    & \inf \set{\int \lambda \log\left(\frac{\lambda}{\eta - \loss(x,\xi)}\right)\, \d \hat{\Qb}(\xi)+(r-1)\lambda+\eta}{\lambda\geq 0,~\eta \geq \max_{\xi \in \Sigma} \loss(x,\xi)}\\
    \leq & \inf \set{\int \lambda \log\left(\frac{\lambda}{\eta - \loss(x,\xi)}\right)\, \d \hat{\Qb}(\xi)+(r-1)\lambda+\eta}{\lambda\geq 0,~\eta \geq \max_{\xi \in \Sigma} \loss(x,\xi)+\epsilon}\\
    \leq & \inf \set{\int \lambda \log\left(\frac{\lambda}{\eta - \loss(x,\xi)}\right)\, \d \hat{\Qb}(\xi)+(r-1)\lambda+\eta}{\lambda\geq 0,~\eta \geq \max_{\xi \in \Sigma} \loss(x,\xi)}+\epsilon.
  \end{align*}
  for any $\epsilon>0$.
  We will now show that the fourth equality follows from the minimax theorem of \citet{sion1958general}. Let
  \[
    L(\lambda,\eta, \hat \Qb) = \int \lambda \log\left(\frac{\lambda}{\eta - \loss(x,\xi)}\right)\, \d \hat{\Qb}(\xi)+(r-1)\lambda+\eta
  \]
  be our saddle point function and remark that this function is convex in $(\lambda, \eta)$ for any fixed $\hat \Qb$ and concave (linear) in $\hat \Qb$ for any fixed $(\lambda, \eta)$. As shown in the proof of \citet[Proposition 5]{van2021data} the variables $(\lambda, \eta)$ can be restricted to the compact set
  \[
    \eta \leq M_3\defn\frac{\max_{\xi\in\Sigma}\ell(x, \xi)+\epsilon-\exp(-r)\min_{\xi\in\Sigma}\ell(x, \xi)}{1-\exp(-r)} \quad {\rm{and}} \quad \lambda \leq M_r\defn\exp(\log(M_3 -\textstyle\min_{\xi\in \Sigma}\ell(x, \xi))-r)
  \]
  without loss of optimality. Furthermore, we have the dual representation $L(\lambda,\eta, \hat \Qb)= \max_{\Pb'\in \cP_+} \int \ell(x, \xi) \, \d\Pb'+\eta(1-\int \, \d\Pb')+\lambda(r-\KL (\hat{\Qb}||\Pb'))$ \citep{van2021data} where $\cP_+$ denotes here all nonnegative Borel measures on $\Sigma$ establishing the lower semicontinuity of the function $L(\lambda,\eta, \hat \Qb)$ in $(\lambda, \beta)$ for any $\hat \Qb$.
  The minimax theorem of \citet{sion1958general} applies as $L(\lambda,\eta, \hat \Qb)$ is continuous in $\hat \Qb$ as the integrant
  \[
    \xi\mapsto \lambda \log\left(\frac{\lambda}{\eta - \loss(x,\xi)}\right)
  \]
  is continuous and bounded for any $\lambda$ and $\eta$ such that $\eta - \loss(x,\xi)\geq \epsilon>0$ for all $\xi\in \Sigma$.
  
  Finally,
  \begin{align*}
    & \hat{c}^{\cN,\alpha,r}_{\HR}(x,\hat{\Pb})\\
    = & \lim_{\epsilon \downarrow 0}  \left\{
        \begin{array}{rl}
          \inf & (1-\alpha)\beta+ \int \max\left(\lambda \log\left(\frac{\lambda}{\eta - \loss^\cN(x,\xi)}\right)-\beta, 0\right) \, \d \Pb'(\xi) +\alpha \lambda\log\left(\tfrac{\lambda}{(\eta - \max_{\xi\in\Sigma} \loss(x,\xi))}\right) \\
               & \qquad +(r-1)\lambda+\eta\\
          \st & \beta\leq \lambda\log\left(\tfrac{\lambda}{(\eta - \max_{\xi\in\Sigma} \loss(x,\xi))}\right),~\lambda\geq 0,~\eta \geq \max_{\xi \in \Sigma} \loss(x,\xi)+\epsilon.
        \end{array}
        \right.\\
    = & \lim_{\epsilon \downarrow 0}  \left\{
        \begin{array}{rl}
          \inf & \int \max\left(\lambda \log\left(\frac{\lambda}{\eta - \loss^\cN(x,\xi)}\right), \beta\right) \, \d \Pb'(\xi)+(r-1)\lambda-\alpha\beta+\eta+\alpha \lambda\log\left(\tfrac{\lambda}{(\eta - \max_{\xi\in\Sigma} \loss(x,\xi))}\right) \\
          \st & \beta\leq \lambda\log\left(\tfrac{\lambda}{(\eta - \max_{\xi\in\Sigma} \loss(x,\xi))}\right),~\lambda\geq 0,~\eta \geq \max_{\xi \in \Sigma} \loss(x,\xi)+\epsilon.
        \end{array}
        \right.\\
    = & \lim_{\epsilon \downarrow 0} \left\{
        \begin{array}{rl}
          \inf & \int \max\left(\lambda \log\left(\frac{\lambda}{\eta - \loss^\cN(x,\xi)}\right), \lambda\log\left(\frac{\lambda}{(\eta - \max_{\xi\in\Sigma} \loss(x,\xi))}\right)-\beta'\right) \, \d \Pb'(\xi)+(r-1)\lambda+\alpha\beta'+\eta \\
          \st & \beta'\geq 0,~\lambda\geq 0,~\eta \geq \max_{\xi \in \Sigma} \loss(x,\xi)+\epsilon.
        \end{array}
        \right.\\
    = & \left\{
        \begin{array}{rl}
          \inf & \int \max\left(\lambda \log\left(\frac{\lambda}{\eta - \loss^\cN(x,\xi)}\right), \lambda\log\left(\frac{\lambda}{(\eta - \max_{\xi\in\Sigma} \loss(x,\xi))}\right)-\beta'\right) \, \d \Pb'(\xi)+(r-1)\lambda+\alpha\beta'+\eta \\
          \st & \beta'\geq 0,~\lambda\geq 0,~\eta \geq \max_{\xi \in \Sigma} \loss(x,\xi).
        \end{array}
        \right.
  \end{align*}
  \endgroup
  The first equality follows from Theorem \ref{thm: LP-DRO expression} where we remark that
  \[
  \max_{n\in\cN, \xi-n\in \Sigma} \lambda\log\left(\tfrac{\lambda}{(\eta - \loss(x,\xi-n))}\right) = \lambda\log\left(\tfrac{\lambda}{(\eta - \loss^\cN(x,\xi))}\right)
  \]
  and that we can restrict
  $\beta\leq \max_{\xi\in\Sigma}\lambda\log\left(\tfrac{\lambda}{(\eta - \loss^\cN(x,\xi))}\right) = \lambda\log\left(\tfrac{\lambda}{(\eta - \max_{\xi\in\Sigma} \loss(x,\xi))}\right)$.
  The second equality follows from $\beta+\max\left(\lambda \log\left(\tfrac{\lambda}{(\eta - \loss^\cN(x,\xi))}\right)-\beta, 0\right) = \max\left(\lambda \log\left(\tfrac{\lambda}{(\eta - \loss^\cN(x,\xi))}\right), \beta\right)$. The third equality follows from the change of variable $\beta'=\lambda\log\left(\tfrac{\lambda}{(\eta - \max_{\xi\in\Sigma} \loss(x,\xi))}\right)-\beta$. The final inequality follows from
  the lower semicontinuity of the objective function due to the representation
  \begin{align*}
    & \int \max\left(\lambda \log\left(\frac{\lambda}{\eta - \loss^\cN(x,\xi)}\right), \lambda\log\left(\frac{\lambda}{(\eta - \max_{\xi\in\Sigma} \loss(x,\xi))}\right)-\beta'\right) \, \d \Pb'(\xi)+(r-1)\lambda+\alpha\beta'+\eta\\
   = & \sup \set{L(\lambda, \eta, \hat \Qb)}{\hat{\Qb} \in \cP, ~\LP_{\cN}(\hat{\Pb}, \hat{\Qb}) \leq \alpha}
  \end{align*}
  where the lower semicontinuity of $L(\lambda, \eta, \hat \Qb)$ in $(\lambda,\eta)$ was established before.
\end{proof}

\section{Omitted proofs of Section \ref{sec: random corr}}\label{App: proofs of random corr}

\subsection{Proof of Theorem \ref{thm: HR robustness}}\label{App: proof of robustness HR}

\begin{proof}[Proof of Theorem \ref{thm: HR robustness}]
We distinguish two cases.
  
\textbf{Case I: $\Eb_{\Pb}[\loss(x,\txi)] = \max_{\xi \in \Sigma} \loss(x,\xi)$.} In this case, the distribution $\Pb$ is supported on the set $\Sigma_\infty=\argmax_{\xi \in \Sigma} \loss(x,\xi)$.
Let $\Sigma_\infty^\cN=((\Sigma_\infty+\cN) \cap \Sigma)\setminus \Sigma_\infty$.
As before, we denote with $\hat{\Pb}_T$ the empirical distribution of the data and we define $\hat \alpha^\infty_T\defn \hat{\Pb}_T(\Sigma_\infty)$, $\hat \alpha^\cN_T\defn \hat{\Pb}_T(\Sigma_\infty^\cN)$ and $\hat \alpha^\Sigma_T\defn \hat{\Pb}_T(\Sigma\setminus (\Sigma_\infty\cup \Sigma_\infty^\cN))$.
Intuitively, $\hat \alpha^\infty_T$ denotes the fraction of all observed data points in the set $\Sigma_\infty$, $\hat \alpha^\cN_T$ the fraction of all observed data points close but not in the set $\Sigma_\infty$ and $\hat \alpha^\Sigma_T$ the remaining fraction of samples far from $\Sigma_\infty$.
Similarly, we let $\Qb$ denote the distribution of the noisy data $\txi^c$ and define $\alpha^\infty=\Qb(\Sigma_\infty)$,  $\alpha^\cN\defn \Qb(\Sigma^\cN_\infty)$ and $\alpha^\Xi\defn \Qb(\Sigma\setminus (\Sigma_\infty\cup\Sigma_\infty^\cN))$. We remark that the random variable $T (\hat \alpha^\infty_T, \hat \alpha^\cN_T, \hat \alpha^\Sigma_T)$ is distributed as a multinomial distribution with parameter $(\alpha^\infty, \alpha^\cN, \alpha^\Sigma)$. From the fact that $\LP_{\cN}(\Qb,\Pb)\leq \alpha$ it follows that we must have that $\alpha^\Sigma\leq \alpha$.

  Define first a distribution $\bar{\Qb}_T \in \cP$ through the characterization
  \begin{align*}
    & \bar{\Qb}_T(B)\\
    \defn & \int_B \frac{\one{\xi\in \Sigma_\infty}}{\hat{\Pb}_T(\Sigma_\infty)} \alpha^\infty \, \d \hat{\Pb}_T(\xi) + \int_B \frac{\one{\xi\in \Sigma_\infty^\cN}}{\hat{\Pb}_T(\Sigma_\infty^\cN)} \alpha^\cN \, \d \hat{\Pb}_T(\xi)+\int_B \frac{\one{\xi\in \Sigma\setminus (\Sigma_\infty \cup \Sigma^\cN_\infty)}}{\hat{\Pb}_T(\Sigma\setminus (\Sigma_\infty \cup \Sigma^\cN_\infty))} \alpha^\Sigma \, \d \hat{\Pb}_T(\xi)
  \end{align*}
  for any measurable set $B$ in $\Xi$. Notice that by construction we have that $\bar{\Qb}_T(\Sigma^\infty)=\alpha^\infty$, $\bar{\Qb}_T(\Sigma^\infty_\cN)=\alpha^\cN$  and $\bar{\Qb}_T(\Sigma\setminus(\Sigma_\infty \cup \Sigma^\cN_\infty))=\alpha^\Sigma$ with Radon-Nikodym derivative
  \[
    \frac{\d \hat{\Pb}_T}{\d  \bar{\Qb}_T }(\xi) = \one{\xi\in \Sigma_\infty} \frac{\hat \alpha^\infty_T}{\alpha^\infty}+\one{\xi\in \Sigma_\infty^\cN} \frac{\hat \alpha^\cN_T}{\alpha^\cN}+\one{\xi\in \Sigma\setminus(\Sigma_\infty \cup \Sigma^\cN_\infty)} \frac{\hat \alpha_T^\Sigma}{\alpha^\Sigma}
  \]
  for all $\xi\in \Sigma$.
  Intuitively, the distribution $\bar{\Qb}_T$ is a simple rescaling of the distribution $\hat{\Pb}_T$ in each of the regions of interest $\Sigma_\infty$, $\Sigma^\cN_\infty$ and $\Sigma\setminus(\Sigma_\infty \cup \Sigma^\cN_\infty)$ so that the total mass in those regions coincides with the total probability mass assigned to those regions by the distribution $\Qb$.
  Hence, we have
  \begingroup
  \allowdisplaybreaks
  \begin{align*}
      & \KL(\hat{\Pb}_T || \bar{\Qb}_T) \\
      &= \int \log\left( \frac{\d \hat{\Pb}_T}{\d  \bar{\Qb}_T }(\xi) \right) \,\d \hat{\Pb}_T(\xi) \\
      &= \int_{\Sigma_\infty} \log\left( \frac{\d \hat{\Pb}_T}{\d  \bar{\Qb}_T }(\xi) \right) \,\d \hat{\Pb}_T(\xi) +\int_{\Sigma^\cN_\infty} \log\left( \frac{\d \hat{\Pb}_T}{\d  \bar{\Qb}_T }(\xi) \right) \,\d \hat{\Pb}_T(\xi) + \int_{\Sigma\setminus (\Sigma_\infty \cup \Sigma^\cN_\infty)} \log\left( \frac{\d \hat{\Pb}_T}{\d  \bar{\Qb}_T }(\xi) \right) \,\d \hat{\Pb}_T(\xi) \\
      &=
        \hat \alpha_T^\infty
        \log\left( \frac{\hat \alpha_T^\infty}{\alpha^\infty}\right)
        +
        \hat \alpha_T^\cN
        \log\left( \frac{\hat \alpha_T^\cN}{\alpha^\cN}\right)
        +
         \hat \alpha_T^\Sigma
        \log\left( \frac{\hat \alpha_T^\Sigma}{\alpha^\Sigma}\right)=: D((\hat \alpha^\infty_T, \hat \alpha^\cN_T, \hat \alpha^\Sigma_T), (\alpha^\infty, \alpha^\cN, \alpha^\Sigma))
  \end{align*}  
  \endgroup
  It can be shown that asymptotically in $T$, $\Qb^{\infty}(D((\hat \alpha^\infty_T, \hat \alpha^\cN_T, \hat \alpha^\Sigma_T), (\alpha^\infty, \alpha^\cN, \alpha^\Sigma)) > r) \leq \exp(-rT + o(T))$ for all $r>0$
  \citep{csiszar1998method, agrawal2020finite} exploiting that $T (\hat \alpha^\infty_T, \hat \alpha^\cN_T, \hat \alpha^\Sigma_T)$ is distributed as a multinomial distribution with parameter $(\alpha^\infty, \alpha^\cN, \alpha^\Sigma)$. Hence, $\Qb^{\infty}(\KL(\hat{\Pb}_T || \bar{\Qb}_T) \leq r) = 1-\Prob(D((\hat \alpha^\infty_T, \hat \alpha^\cN_T, \hat \alpha^\Sigma_T), (\alpha^\infty, \alpha^\cN, \alpha^\Xi))  > r) \geq 1- \exp(-rT + o(T))$.

  Remark that we can associate with any point $\xi\in \Sigma_\infty\cup\Sigma^\cN_\infty$ a point $\Pi_{\Sigma_\infty}(\xi)\in \Sigma_\infty$ so that $\xi-\Pi_{\Sigma_\infty}\in \cN$. Furthermore, denote with $\xi_\infty$ an arbitrary point in $\Sigma_\infty$.
  Define the random variable $\bar \Pb_T$ through
  \[
    \bar \Pb_T(B) = \bar \Qb_T(\set{\xi\in \Sigma_\infty\cup\Sigma^\cN_\infty}{\Pi_{\Sigma_\infty}(\xi)\in B})+\bar \Qb_T(\Sigma\setminus(\Sigma_\infty\cup\Sigma^\cN_\infty))\cdot \one{\xi_\infty\in B}
  \]
  for all measurable sets $B$ in $\Sigma$. It can be remarked that $\bar \Pb_T \in \cP(\Sigma^\infty)$ as indeed we have $\bar \Pb_T(\Sigma^\infty)=\bar \Qb_T(\set{\xi\in \Sigma_\infty\cup\Sigma^\cN_\infty}{\Pi_{\Sigma_\infty}(\xi)\in\Sigma_\infty })+\bar \Qb_T(\Sigma\setminus(\Sigma_\infty\cup\Sigma^\cN_\infty))=\Qb_T(\Sigma)=1$. We can interpret $\bar \Pb_T$ as the resulting measure after transporting the probability measure $\bar \Qb_T$ by moving all mass at any location $\xi\in \Sigma_\infty\cup\Sigma^\cN_\infty$ to their associated element in $\Sigma_\infty$ over a distance bounded in $\cN$ and the remaining mass at location $\xi\not\in \Sigma_\infty\cup\Sigma^\cN_\infty$ is moved to the point $\xi_\infty\in \Sigma_\infty$. That is, according to a coupling $\gamma_T(\xi, \Pi_{\Sigma_\infty}(\xi))=\bar \Qb_T(\xi)$ for all $\xi\in \Sigma_\infty\cup\Sigma^\cN_\infty$ and $\gamma_T(\xi, \xi_\infty)=\bar \Qb_T(\xi)$ for all $\xi\in \Sigma\setminus(\Sigma_\infty\cup\Sigma^\cN_\infty)$.
  Note that indeed $\gamma_T\in \Gamma(\bar \Qb_T, \bar \Pb_T)$.
  We have thus
\begin{align*}  
  \LP_{\cN}(\bar{\Qb}_T, \bar \Pb_T) & = \inf_{\gamma\in \Gamma(\bar Q_T, \bar \Pb_T) }\int \one{\xi-\xi'\not\in \cN} \,\d \gamma(\xi, \xi')\\
                                         & \leq \int \one{\xi-\xi'\not\in \cN} \,\d \gamma_T(\xi, \xi')\\
                                         & = \int_{\Sigma_\infty\cup\Sigma^\cN_\infty}\one{\xi-\Pi_{\Sigma_\infty}(\xi) \not\in \cN } \,\d \bar {\Qb}_T(\xi)+\int_{\Sigma\setminus(\Sigma_\infty\cup\Sigma^\cN_\infty)} \one{\xi-\xi_\infty \not\in \cN } \,\d \bar {\Qb}_T(\xi)\\
                                         & = \int_{\Sigma\setminus(\Sigma_\infty\cup\Sigma^\cN_\infty)} \,\d \bar {\Qb}_T(\xi)\\
                                         & =  \alpha^\Sigma \leq \alpha.
\end{align*}
Here the second equality follows from the fact that by construction $\xi-\Pi_{\Xi_\infty}(\xi)\in \cN$ when $\xi\in \Sigma_\infty\cup \Sigma_\infty^\cN$.

Hence,
   \begingroup
  \allowdisplaybreaks
  \begin{align*}
    \Qb^\infty\left(
    \Eb_{\Pb}[\loss(x,\txi)]) \leq \hat{c}^{\cN,\alpha,r}_{\HRo}(x,\Pemp{T}) 
    \right) 
    &= 
      \Qb^\infty\left(
    \max_{\xi\in \Sigma}\ell(x, \xi) = \hat{c}^{\cN,\alpha,r}_{\HRo}(x,\Pemp{T})\right) \\
    & = 
      \Qb^\infty\left(
       \exists \Pb' \in \cP(\Sigma_\infty), \; \exists \Qb' \in \cP, \;
      \KL (\Pemp{T}||\Qb') \leq r, \;  \LP_\cN(\Qb',\Pb') \leq \alpha
      \right) \\
    &\geq 
      \Qb^\infty\left(
      \KL (\Pemp{T}||\bar{\Qb}_T) \leq r, \;  \LP_\cN(\bar{\Qb}_T,\bar \Pb_T) \leq \alpha 
      \right) \\
    &= 
      \Qb^\infty\left(
      \KL (\Pemp{T}||\bar{\Qb}_T) \leq r
      \right) \\
    &\geq 
       1- \exp(-rT + o(T)).
  \end{align*}
  \endgroup
  
  \textbf{Case II: $\Eb_{\Pb}[\loss(x,\txi)] < \max_{\xi \in \Sigma} \loss(x,\xi)$.}
  Let $\txi \sim \Pb$ be the random variable of the true uncertainty before noise and misspecification. Let $\txi^c$ be the random noisy and corrupted observation obtained from $\txi$ and denote with $\Qb \in \cP$ its distribution. Note that $\txi$ and $\txi^c$ may indeed be correlated and let $\hat{\gamma}\in \Gamma(\Qb,\Pb)$ be their joint distribution. We have
  \begin{align*}
    \LP_\cN(\Qb,\Pb) 
    &=
      \inf_{\gamma \in \Gamma(\Qb,\Pb)} \int \indic(\xi-\xi'  \not \in \cN) \,\d\hat{\gamma}(\xi,\xi') \\
    &\leq
      \int\indic(\xi-\xi'  \not \in \cN) \,\d\hat{\gamma}(\xi,\xi')\\
    &= \Prob\left(\txi^c - \txi \not \in \cN \right).
  \end{align*}
  As the noise $\tn$ realizes in $\cN$, the event $\txi^c - \txi \not \in \cN$ occurs only in the case of misspecification of the sample. Hence, $\LP_\cN(\Qb,\Pb) \leq \Prob\left(\tc = 1\right) \leq \alpha$.
  

  Let $\mathcal{D}: = \{\hat{\Pb} \in \cP \; : \; \hat{c}^{\cN,\alpha,r}_{\HRo}(x,\hat{\Pb}) < \Eb_{\Pb}[\loss(x,\txi)]) \}$. We have
  $$
  \Qb^\infty\left(
    \hat{c}^{\cN,\alpha,r}_{\HRo}(x,\Pemp{T}) \geq \Eb_{\Pb}[\loss(x,\txi)])
    \right)
    =
    1-
    \Qb^\infty\left(
    \Pemp{T} \in \mathcal{D}
    \right).
  $$
  If $\mathcal{D}$ is empty, the result is trivial. We suppose now it is not.
  As $\Pemp{T}$ is an empirical distribution sampled from the corrupted distribution $\Qb$, Sanov's theorem ensures that
  \begin{align*}
  \Qb^\infty\left(
    \hat{c}^{\cN,\alpha,r}_{\HRo}(x,\Pemp{T}) \geq \Eb_{\Pb}[\loss(x,\txi)]
    \right)
    \geq
      1 - \exp\left(-T\inf_{\hat{\Pb} \in \bar{\mathcal{D}}} \KL(\hat{\Pb} || \Qb) + o(T)\right).
  \end{align*}
  It suffices to show that $\inf_{\hat{\Pb} \in \bar{\mathcal{D}}} \KL(\hat{\Pb} || \Qb) \geq r$. For this purpose, we show that for all $\hat{\Pb} \in \bar{\mathcal{D}}$, $\KL(\hat{\Pb} || \Qb) \geq r$. By the lower semicontinuity of $\hat{c}^{\cN,\alpha,r}_{\HRo}(x,\cdot)$ (see Lemma \ref{lemma: HR continuous}), we have $\bar{\mathcal{D}} \subseteq \{\hat{\Pb} \in \cP \; : \; \hat{c}^{\cN,\alpha,r}_{\HRo}(x,\hat{\Pb}) \leq \Eb_{\Pb}[\loss(x,\txi)] \}$.
  Let $\hat{\Pb} \in \bar{\mathcal{D}}$ and suppose for the sake of contradiction that $\KL(\hat{\Pb} || \Qb) < r$. We will show that $\hat{c}^{\cN,\alpha,r}_{\HRo}(x,\hat{\Pb}) > \Eb_{\Pb}[\loss(x,\txi)]$ contradicting therefore that $\hat{\Pb} \in \bar{\mathcal{D}}$.
  
  Recall that $\Qb$ is the distribution of the corrupted observations and $\Pb$ is the distribution before corruption.
  We have $\LP_\cN(\Qb,\Pb)  \leq \alpha$ and $\KL(\hat{\Pb} || \Qb) < r$. Therefore,
  $\Pb \in 
      \{ \Pb' \in \cP \; : \; \exists \Qb' \in \cP,  \;
      \KL (\hat{\Pb}||\Qb') \leq r, \;  \LP_\cN(\Qb',\Pb') \leq \alpha \}$ which implies that $\hat{c}^{\cN,\alpha,r}_{\HRo}(x,\hat{\Pb}) \geq \Eb_{\Pb}[\loss(x,\txi)]$ by definition of $\hat{c}^{\cN,\alpha,r}_{\HRo}$ stated in Equation \eqref{eq: HRo predictor}. 
      
Let $\xi_{\infty} \in \argmax_{\xi\in \Sigma} \loss(x,\xi)$ and let denote $\delta_{\xi_{\infty}}$ the a Dirac distribution on $\xi_{\infty}$. The function $\lambda \in [0,1] \rightarrow \KL(\hat{\Pb} || (1-\lambda)\Qb + \lambda \delta_{\xi_{\infty}})$ is continuous in $\lambda$ as the KL divergence is lower semicontinuous and any convex function is upper semicontinuous on a locally simplicial set \citep[Theorem 10.2]{rockafellar2015convex}. Therefore, as $\KL(\hat{\Pb} || \Qb) <r$, we can chose $\lambda>0$ small enough such that $\KL(\hat{\Pb} || (1-\lambda)\Qb + \lambda \delta_{\xi_{\infty}}) \leq r$. We show next that this implies that 
$(1-\lambda)\Pb + \lambda \delta_{\xi_{\infty}} \in 
      \{ \Pb' \in \cP \; : \; \exists \Qb' \in \cP,  \;
      \KL (\hat{\Pb}||\Qb') \leq r, \;  \LP_\cN(\Qb',\Pb') \leq \alpha \}$. Denote $\Qb(\lambda) = (1-\lambda)\Qb + \lambda \delta_{\xi_{\infty}}$ and $\Pb(\lambda) = (1-\lambda)\Pb + \lambda \delta_{\xi_{\infty}}$.
      It suffices to show that $\LP_\cN(\Qb(\lambda),\Pb(\lambda)) \leq \alpha$, as then, we have the inclusion in the set by choosing $\Qb' = \Qb(\lambda)$. We have
      \begin{align*}
          \LP_\cN(\Qb(\lambda),\Pb(\lambda)) 
          &=
          \inf \left\{\int \indic(\xi-\xi' \not\in\cN) \,\d\gamma(\xi,\xi') \; : \; \gamma \in \Gamma(\Qb(\lambda),\Pb(\lambda)) \right\} \\
          &\leq
          \inf \left\{\int \indic(\xi-\xi' \not\in\cN) \,\d\gamma(\xi,\xi') 
          \; : \;
          \gamma = (1-\lambda) \gamma' + \lambda \delta_{\xi_{\infty},\xi_{\infty}}, \;  \gamma' \in \Gamma(\Qb,\Pb) \right\} \\
          &= 
          \inf \left\{(1-\lambda)\int \indic(\xi-\xi' \not\in\cN) \,\d\gamma'(\xi,\xi') 
          \; : \;
          \gamma' \in \Gamma(\Qb,\Pb) \right\} + \lambda\one{0\not\in\cN}\\
          &= (1-\lambda) \LP_\cN(\Qb,\Pb)  = (1-\lambda)\alpha \leq \alpha.
      \end{align*}      
Hence, we have $\Pb(\lambda) \in 
      \{ \Pb' \in \cP \; : \; \exists \Qb' \in \cP,  \;
      \KL (\hat{\Pb}||\Qb') \leq r, \;  \LP_\cN(\Qb',\Pb') \leq \alpha \}$
which implies by definition of $\hat{c}^{\cN,\alpha,r}_{\HRo}$ that $\hat{c}^{\cN,\alpha,r}_{\HRo}(x,\hat{\Pb}) \geq \Eb_{\Pb(\lambda)}[\loss(x,\txi)] = (1-\lambda) \Eb_{\Pb}[\loss(x,\txi)] + \lambda \max_{\xi \in \Sigma}\loss(x,\xi) > \Eb_{\Pb}[\loss(x,\txi)]$.
\end{proof}

\subsection{Proof of Theorem \ref{thm: HR finite formulation}} \label{App: proof of finite HR}
\begin{proof}[Proof of Theorem \ref{thm: HR finite formulation}]
  We can write the supremum in Equation \eqref{eq: HRo predictor} as
  \begin{align*}
    \sup \left\{ 
    \Eb_{\Pb'}[\loss(x,\txi)]
    :  
    \Pb'\in \cP,\,\Qb' \in \cP, \gamma \in \Gamma(\Qb',\Pb'),
    \int \log  \left(\frac{\d\Pemp{T}}{\d\Qb'}(\xi)\right)\, \d\Pemp{T}(\xi) \leq r,
    \int \indic(\xi-\xi' \not \in \cN) \,\d\gamma(\xi,\xi') \leq \alpha
    \right\}.
  \end{align*}
  For all $k\in [K]$, let $\xi'_k \in \xi_k-\argmax_{\noise \in \cN} \loss(x,\xi_k - n)$ and let $\xi_\infty \in \argmax_{\xi \in \Sigma} \loss(x,\xi)$.
  We will show that there exists an optimal solution $\bar{\Pb}' \in \cP, \bar{\Qb}' \in \cP$ and $\bar{\gamma} \in \Gamma(\bar{\Qb}',\bar{\Pb}')$ verifying: (i) $\supp\bar{\Pb}' \subseteq \{\xi'_{k} \; : \; k \in [K]\} \cup \{\xi_{\infty}\}$, 
  (ii) $\supp\bar{\Qb}' \subseteq \{\xi_{k} \; : \; k \in [K]\} \cup \{\xi_{\infty}\}$,
  (iii) $\supp \bar{\gamma} \subseteq \{ (\xi_k,\xi'_{k}) \; : \; k \in [K]\} \cup \{ (\xi_k,\xi_{\infty}) \; : \; k \in [K]\}\cup \{\xi_\infty, \xi_\infty\}$.
  This result implies the stated finite formulation, by choosing $p'_k = \bar{\Pb}'(\xi'_k)$, $q'_k = \bar{\Qb}'(\xi_k)$ and $s_k = \gamma(\xi_k,\xi_{\infty})$ for all $k \in [K]$ as well as $q_{K+1} = \bar{\Qb}'(\xi_\infty)$ and $p'_{K+1} = \bar{\Pb}'(\xi_\infty)$. 

  Let $\Pb'^\star\in \cP,\Qb'^\star \in \cP$ be an optimal solution in Equation \eqref{eq: HRo predictor}. Denote $\Xi \defn \supp (\Qb'^\star) \setminus \left(\{\xi_{k} \; : \; k \in [K]\} \cup \{\xi_{\infty}\} \right)$.
  We start by constructing $\bar{\Qb}'$. Consider $\bar{\Qb}'\in \cP$ defined as $\bar{\Qb}'(A) = \Qb'^\star(A \setminus \Xi) + \Qb'^\star(\Xi)\indic(\xi_{\infty} \in A)$ for all events $A \in \cB(\Sigma)$. 
  Notice that indeed we have $\supp(\bar \Qb') \subseteq \{\xi_{k} \; : \; k \in [K]\} \cup \{\xi_{\infty}\}$.
  We first verify that the new solution still verifies $\KL (\Pemp{T}||\bar{\Qb}') \leq r$.
  We have
  \begin{align*}
    \int \log  \left(\frac{\d\Pemp{T}}{\d\bar{\Qb}'}(\xi)\right)\, \d\Pemp{T}(\xi)
    &=
      \sum_{k\in[K]} \log \left( \frac{\Pemp{T}(\xi_k)}{\bar{\Qb}'(\xi_k)} \right)\Pemp{T}(\xi_k)\\
    &= 
      \sum_{k\in[K]} \log \left( \frac{\Pemp{T}(\xi_k)}{\Qb'^\star(\xi_k) + \Qb'^\star(\Xi)\indic(\xi_k = \xi_{\infty})} \right)\Pemp{T}(\xi_k)\\
    &\leq 
      \sum_{k\in[K]} \log \left( \frac{\Pemp{T}(\xi_k)}{\Qb'^\star(\xi_k)} \right)\Pemp{T}(\xi_k) \leq r
  \end{align*}
  where the first inequality uses the fact that the logarithm is increasing and the second inequality is by feasibility of $\Qb'^\star$.

  Theorem \ref{thm: LP-DRO expression} ensures that there exists $\bar{\Pb}' \in \argmax \{\Eb_{\Pb'}[\loss(x,\txi)] \; : \; \Pb' \in \cP, \; \LP_\cN(\bar{\Qb}',\Pb') \leq \alpha\}$ such that $\supp \bar{\Pb}' \subseteq \{ \xi'_k \; : \; k \in [K] \} \cup \{ \xi_{\infty}\}$. Consider the solution $\bar{\Pb}'\in \cP,\bar{\Qb}' \in \cP$. We have
  \begin{align*}
    \max
    \{
    \Eb_{\Pb'}[\loss(x,\txi)]
    & \; : \; 
      \Pb'\in \cP,\;\Qb' \in \cP, \;
      \KL (\Pemp{T}||\Qb') \leq r, \;
      \LP_\cN(\Qb',\Pb') \leq \alpha 
      \}\\
    &=
      \max 
      \{
      \max \{
      \Eb_{\Pb'}[\loss(x,\txi)]
      \; : \; 
      \Pb' \in \cP, \;
      \LP_\cN(\Qb',\Pb') \leq \alpha 
      \}
      \; : \;
      \Qb' \in \cP,\;
      \KL (\Pemp{T}||\Qb') \leq r
      \}
  \end{align*}
  Hence, the solution $\bar{\Pb}',\bar{\Qb}'$ is feasible, as $\bar{\Qb}'$ is feasible in the outer maximum and $\bar{\Pb}'$ is feasible in the inner by construction. Finally, Theorem \ref{thm: LP-DRO expression} ensures that there exists a coupling $\bar{\gamma} \in \Gamma(\bar{\Qb}',\bar{\Pb}')$ of support $\supp \bar{\gamma} \subseteq \{ (\xi_k,\xi'_{k}) \; : \; k \in [K]\} \cup \{ (\xi_k,\xi_{\infty}) \; : \; k \in [K]\}\cup \{\xi_\infty, \xi_\infty\}$ corresponding to this solution.

  Let us show that $\bar{\Qb}'$ also attains the optimal cost.
  Using Theorem \ref{thm: LP-DRO expression} and the $\CVaR$ formula \eqref{eq:cvar-max-char}, we have

  \begin{align*}
    \Eb_{\Pb'^\star}(\loss(x,\Tilde{\xi}))
    &=
      \max \{
      \Eb_{\Pb'}(\loss(x,\Tilde{\xi}))
      \; : \; 
      \Pb' \in \cP, \;
      \LP_\cN(\Qb'^\star,\Pb') \leq \alpha 
      \} \\
    & =
      \sup\{\Eb_{\Qb'^\star}(\loss^\cN(x,\Tilde{\xi}) R(\txi)) \; : \; \Eb_{\Qb'^\star}(R(\txi)) = 1-\alpha,~R:\Sigma\to[0,1] \} +
      \alpha \max_{\xi \in \Sigma} \loss(x,\xi) \numberthis \label{proof eq: cost of new sol}
  \end{align*}  
  Consider the mapping $S:\Sigma\to\Sigma$ defined as
  \begin{align*}
    S(\xi)=
    \begin{cases}
      \xi & \text{if} \; \xi \in \{\xi_{k} \; : \; k \in [K]\}, \\
      \xi_{\infty} & \text{otherwise.}
    \end{cases}
  \end{align*}
  Informally, this mapping moves the mass distribution $\Qb'^\star$ into the mass distribution $\bar{\Qb}'$. 
  Recall that $\loss(x,\xi_{\infty}) = \max_{\xi \in \Sigma} \loss(x,\xi)$. Hence, for all $\xi \in \Sigma$, we have $\loss^\cN(x,\xi) \leq \loss^\cN(x,S(\xi))$. Hence,
  \begin{align*}
    &\sup\{\Eb_{\Qb'^\star}(\loss^\cN(x,\Tilde{\xi}) R(\txi)) \; :  \; \Eb_{\Qb'^\star}(R(\txi)) = 1-\alpha,~R:\Sigma\to[0,1] \}
                           \\
    \leq &\sup\{\Eb_{\Qb'^\star}(\loss^\cN(x,S(\Tilde{\xi})) R(\txi)) \; : \; \Eb_{\Qb'^\star}(R(\txi)) = 1-\alpha,~R:\Sigma\to[0,1] \}\\
    = & \textstyle \sup\{\sum_{k\in[K]} \int_{S(\xi)=\xi_k} \loss^\cN(x,\xi_k) R(\xi) \, \d\Qb'^\star(\xi) + \int_{S(\xi)=\xi_\infty} \loss^\cN(x,\xi_\infty) R(\xi) \, \d\Qb^\star(\xi) \; : \\
    & \qquad\qquad  \; \Eb_{\Qb'^\star}(R(\txi)) = 1-\alpha,~R:\Sigma\to[0,1] \}\\
    = & \textstyle \sup\{\sum_{k\in[K]} \loss^\cN(x,\xi_k) \tilde R(\txi_k) \Qb'^\star(S(\xi)=\txi_k) + \loss^\cN(x,\xi_\infty) \tilde R(\txi_\infty) \Qb'^\star(S(\xi)=\txi_\infty)  \; : \\
    & \qquad \qquad\; \textstyle\sum_{k\in [K]}\tilde R(\txi_k) \Qb'^\star(S(\xi)=\txi_k)+ \tilde R(\txi_\infty) \Qb'^\star(S(\xi)=\txi_\infty)= 1-\alpha,~\tilde R:\Sigma\to[0,1] \}\\
        =& \sup\{\Eb_{\bar \Qb'}(\loss^\cN(x,\Tilde{\xi}) \tilde R(\txi)) \; :  \; \Eb_{\bar \Qb'}(\tilde R(\txi)) = 1-\alpha,~\tilde R:\Sigma\to[0,1] \}.
  \end{align*}
  Plugging this inequality in \eqref{proof eq: cost of new sol}, and using Theorem \ref{thm: LP-DRO expression}, we get 
  \begin{align*}
    \Eb_{\Pb'^\star}(\loss(x,\Tilde{\xi}))
    &\leq \sup\{\Eb_{\bar \Qb'}(\loss^\cN(x,\Tilde{\xi}) \tilde R(\txi)) \; :  \; \Eb_{\bar \Qb'}(\tilde R(\txi)) = 1-\alpha,~\tilde R:\Sigma\to[0,1] \} + \alpha \max_{\xi \in \Sigma} \loss(x,\xi) \\
    &= 
      \sup \{
      \Eb_{\Pb'}(\loss(x,\Tilde{\xi}))
      \; : \; 
      \Pb' \in \cP, \;
      \LP_\cN(\bar{\Qb}',\Pb') \leq \alpha 
      \}\\
    &= 
      \Eb_{\bar{\Pb}'}(\loss(x,\Tilde{\xi})).
  \end{align*}
  Hence, the distributions $\bar{\Pb}',\bar{\Qb}'$ are also optimal solutions.
\end{proof}

\begin{lemma}[{\citet{love2015phi}, \citet{van2021data}}]
  \label{lemma:kl-solution}
  We have
  \begin{align*}
    \hat{c}^r_{\mathrm{KL}}(x,\Pb')
    & =
      \min \left\{ 
      \int \lambda \log\left(\frac{\lambda}{\eta - \loss(x,\xi)}\right)\, \d \Pb'(\xi)+(r-1)\lambda+\eta 
      \; : \; \lambda\geq 0,~\eta \geq \max_{\xi \in \Sigma} \loss(x,\xi)
    \right\}
  \end{align*}
  for all $\Pb'\in\cP$ and $r> 0$.
\end{lemma}

\subsection{Generalization and Proof of Theorem \ref{thm: dual rep HR}}
\label{App: Dual HR}

\begin{theorem}[General Dual Formulation]\label{thm: general dual rep}
Let $\hat{\Pb} \in \cP$. For all $x\in \cX$, the $\HRo$ predictor \eqref{eq: HRo predictor} admits for all $r>0$ the dual representation
  \begin{equation*}
    \hat{c}^{\cN,\alpha,r}_{\HRo}(x,\hat{\Pb})=
    \left\{
    \begin{array}{rl}
      \inf & \int w(\xi) \, \d\hat{\Pb}(\xi) + \lambda (r-1) + \beta \alpha + \eta\\[0.5em]
      \st & w: \Sigma \rightarrow \Re, \; \lambda \geq 0, \; \beta \in \Re, \; \eta \in \Re,\\[0.5em]
           & w(\xi) \geq \lambda \log \left( \frac{\lambda}{\eta - \loss^{\cN}(x,\xi)}\right), ~ w(\xi) \geq \lambda \log \left( \frac{\lambda}{\eta - \max_{\xi' \in \Sigma} \loss(x,\xi') +\beta}\right) \quad \forall \xi \in \Sigma, \\[0.5em]
           & \eta \geq \max_{\xi \in \Sigma} \loss(x,\xi).
    \end{array}
    \right.
  \end{equation*}
\end{theorem}
\begin{proof}
  We have that
  \begingroup
  \allowdisplaybreaks
  \begin{align*}
    & \hat{c}^{\cN,\alpha,r}_{\HRo}(x,\hat{\Pb})\\
     \defn & \sup \set{\Eb_{\Pb'}[\loss(x,\txi)]}{\Pb'\in \cP,\, \Qb' \in \cP, \; \KL (\hat{\Pb}||\Qb') \leq r, \; \LP_{\cN}(\Qb',\Pb') \leq \alpha}\\
    = & \sup\set{\sup \set{\Eb_{\Pb'}[\loss(x,\txi)]}{\Pb'\in \cP, \;\LP_{\cN}(\Qb',\Pb') \leq \alpha}}{\Qb' \in \cP, ~ \KL (\hat{\Pb}||\Qb') \leq r}\\
  = & \sup\set{(1-\alpha)\CVaR_{\Qb'}^\alpha(\loss^\cN(x,\txi))
      +
      \alpha\max_{\xi\in \Sigma}\loss(x,\xi)}{\Qb' \in \cP, ~ \KL (\hat{\Pb}||\Qb') \leq r}\\
    = & \sup\set{\inf \set{\beta(1-\alpha)+\Eb_{\Qb'}[ \max(\loss^\cN(x, \txi)-\beta,0)]}{\beta\leq \max_{\xi\in \Sigma} \ell(x, \xi)}
        }{\Qb' \in \cP, ~ \KL (\hat{\Pb}||\Qb') \leq r} \\
    & \qquad +
        \alpha\max_{\xi\in \Sigma}\loss(x,\xi)\\
    = & \inf\set{\beta(1-\alpha)+\sup \set{\Eb_{\Qb'}[ \max(\loss^\cN(x, \txi)-\beta,0)]}{\Qb' \in \cP, ~ \KL (\hat{\Pb}||\Qb') \leq r}
        }{\beta\leq \max_{\xi\in \Sigma} \loss(x, \xi)}\\
      & \qquad +
        \alpha\max_{\xi\in \Sigma}\loss(x,\xi)
  \end{align*}
  Here, the second equality follows from Theorem \ref{thm: LP-DRO expression}. The third equality follows from the dual characterization of the conditional value-at-risk stated in Equation \eqref{eq:cvar-min-char} where we remark that we can restrict $\beta$ to be satisfy $\beta\leq \max_{\xi\in \Sigma}\loss^\cN(x, \xi)=\max_{\xi\in \Sigma}\loss(x, \xi)$ as $0\in \cN$. We will now show that the fourth equality follows from the minimax theorem of \citet{sion1958general}. Let
  \[
    L(\beta, \Qb') = \beta(1-\alpha)+\Eb_{\Qb'}[ \max(\loss^\cN(x, \txi)-\beta,0)]
  \]
  be our saddle point function and remark that this function is convex in $\beta$ for any fixed $\Qb'$ and concave (linear) in $\Qb'$ for any fixed $\beta$. As for any $\Qb'$ the function $L(\beta, \Qb')$ is convex in $\beta\in\Re$ it is also continuous. We also need that $L(\beta, \Qb')$ is upper semicontinuous in $\Qb'$ for any fixed $\beta$. From the maximum theorem of \cite[p.\ 116]{berge1997topological} and the compactness of the  noise set $\cN$ it follows that the inflated loss function is continuous. Clearly, the function $\max(\loss^\cN(x, \txi)-\beta,0)$ is continuous as well and is bounded by $\max_{\xi\in\Sigma}\ell^\cN(x, \xi) = \max_{\xi\in\Sigma}\ell(x, \xi)<\infty$.
  Hence, the function $L(\beta, \Qb')$ is continuous in $\Qb'$ for any $\beta$ as we have indeed for any sequence $\Qb'_i\in\cP$ converging to $\Qb'$ that
  \[
    \lim_{i\to\infty} \int \max(\loss^\cN(x, \txi)-\beta,0) \, \d \Qb'_i  = \int \max(\loss^\cN(x, \txi)-\beta,0) \, \d \Qb'.
  \]
  The minimax theorem of \citet{sion1958general} hence holds.

  From Lemma \ref{lemma:kl-solution}, we have therefore
  \begin{align*}
    & \hat{c}^{\cN,\alpha,r}_{\HRo}(x,\hat{\Pb})\\
    = & \left\{
        \begin{array}{r@{~}l}
          \inf & \displaystyle\int \lambda \log\left(\frac{\lambda}{\eta - \max(\loss^\cN(x, \xi)-\beta,0)}\right)\, \d \Pb'(\xi)+(r-1)\lambda+\eta+\beta(1-\alpha)+\alpha\max_{\xi\in \Sigma}\loss(x,\xi)\\
          \st & \beta\leq \max_{\xi\in \Sigma} \ell(x, \xi), ~\lambda\geq 0,~\eta \geq \max_{\xi \in \Sigma} \max(\loss^\cN(x,\xi)-\beta, 0)
        \end{array}
        \right.\\
    = & \left\{
        \begin{array}{r@{~}l}
          \inf & \displaystyle\int \max\left(\lambda \log\left(\frac{\lambda}{\eta - \loss^\cN(x, \xi)+\beta}\right), \lambda \log\left(\frac{\lambda}{\eta}\right)\right)\, \d \Pb'(\xi)+(r-1)\lambda+\eta+\beta(1-\alpha)+\alpha\max_{\xi\in \Sigma}\loss(x,\xi)\\
          \st & \beta\leq \max_{\xi\in \Sigma} \ell(x, \xi), ~\lambda\geq 0,~\eta \geq \max(\max_{\xi \in \Sigma}\loss^\cN(x,\xi)-\beta, 0)
        \end{array}
        \right.\\
    = & \left\{
        \begin{array}{r@{~}l}
          \inf & \displaystyle\int \max\left(\lambda \log\left(\frac{\lambda}{\eta' - \loss^\cN(x, \xi)}\right), \lambda \log\left(\frac{\lambda}{\eta'-\beta}\right)\right)\, \d \Pb'(\xi)+(r-1)\lambda+\eta'-\beta\alpha+\alpha\max_{\xi\in \Sigma}\loss(x,\xi)\\
          \st & \beta\leq \max_{\xi\in \Sigma} \ell(x, \xi), ~\lambda\geq 0,~\eta' \geq \max(\max_{\xi\in \Sigma}\loss(x,\xi), \beta)
        \end{array}
        \right.\\
    = & \left\{
        \begin{array}{r@{~}l}
          \inf & \displaystyle\int \max\left(\lambda \log\left(\frac{\lambda}{\eta' - \loss^\cN(x, \xi)}\right), \lambda \log\left(\frac{\lambda}{\eta'-\max_{\xi\in \Sigma}\loss(x,\xi)+\beta'}\right)\right)\, \d \Pb'(\xi)+(r-1)\lambda+\eta'+\beta'\alpha\\
          \st & \beta'\geq 0, ~\lambda\geq 0,~\eta' \geq \max(\max_{\xi\in \Sigma}\loss(x,\xi), \max_{\xi\in \Sigma}\loss(x,\xi)-\beta')
        \end{array}
        \right.\\
        = & \left\{
        \begin{array}{r@{~}l}
          \inf & \displaystyle\int \max\left(\lambda \log\left(\frac{\lambda}{\eta' - \loss^\cN(x, \xi)}\right), \lambda \log\left(\frac{\lambda}{\eta'-\max_{\xi\in \Sigma}\loss(x,\xi)+\beta'}\right)\right)\, \d \Pb'(\xi)+(r-1)\lambda+\eta'+\beta'\alpha\\
          \st & \beta'\geq 0, ~\lambda\geq 0,~\eta' \geq \max_{\xi\in \Sigma}\loss(x,\xi)
        \end{array}
        \right.    
  \end{align*}
  \endgroup
  The second equality follows from the fact that the logarithm function is an increasing function. The third and fourth inequalities follow from the change of variables $\eta'=\eta+\beta$ and $\beta'=\max_{\xi\in\Sigma}\ell(x, \xi)-\beta$. Finally, remark that $\max_{\xi\in\Sigma}\ell^\cN(x, \xi) = \max_{\xi\in\Sigma}\ell(x, \xi)$ as $0\in \cN$.
\end{proof}

We can derive an equivalent dual problem based on \citet[Proposition 5]{van2021data} which we will use to prove continuity of the of the HR predictor $\hat{c}^{\cN,\alpha,r}_{\HRo}$ in Lemma \ref{lemma: HR continuous}.

\begin{lemma}[General Dual Formulation II]\label{thm: general dual rep bis}
Let $\hat{\Pb} \in \cP$. For all $x\in \cX$, the $\HRo$ predictor \eqref{eq: HRo predictor} admits for all $r>0$ the dual representation
\begin{equation}
  \label{eq: general dual rep bis}
    \hat{c}^{\cN,\alpha,r}_{\HRo}(x,\hat{\Pb})=
    \left\{
    \begin{array}{rl}
      \inf & \displaystyle \eta - \exp(-r) \exp\left( \int \log\left(\eta - w(\xi)\right) \d\hat{\Pb}(\xi) \right) + \beta \alpha \\[0.5em]
      \st & w: \Sigma \rightarrow \Re, \; \beta \in \Re_+, \; \eta \in \Re,\\[0.5em]
           & w(\xi) \geq \loss^\cN(x, \xi), ~ w(\xi) \geq \max_{\xi\in \Sigma}\loss(x,\xi)-\beta \quad \forall \xi \in \Sigma, \\[0.5em]
           & \eta \geq \max_{\xi \in \Sigma} \loss(x,\xi).
    \end{array}
    \right.
  \end{equation}
\end{lemma}

\begin{proof}
  We shown in the Proof of Theorem \ref{thm: general dual rep}  that
  \begingroup
  \allowdisplaybreaks
  \begin{align*}
    & \hat{c}^{\cN,\alpha,r}_{\HRo}(x,\hat{\Pb})\\
    = & \inf\set{\beta(1-\alpha)+\sup \set{\Eb_{\Qb'}[ \max(\loss^\cN(x, \txi)-\beta,0)]}{\Qb' \in \cP, ~ \KL (\hat{\Pb}||\Qb') \leq r}
        }{\beta\leq \max_{\xi\in \Sigma} \loss(x, \xi)}\\
      & \qquad +
        \alpha\max_{\xi\in \Sigma}\loss(x,\xi).
  \end{align*}

  Using \citet[Proposition 5]{van2021data}, we have
  \begin{align*}
    & \hat{c}^{\cN,\alpha,r}_{\HRo}(x,\hat{\Pb})\\
    = & \left\{
        \begin{array}{r@{~}l}
          \inf & \displaystyle \eta -\exp(-r) \exp\left(\int \log\left(\eta - \max(\loss^\cN(x, \xi)-\beta,0)\right) \d \Pb'(\xi)\right)+\beta(1-\alpha)+\alpha\textstyle\max_{\xi\in \Sigma}\loss(x,\xi)\\
          \st & \beta\leq \max_{\xi\in \Sigma} \ell(x, \xi), \eta \geq \max_{\xi \in \Sigma} \max(\loss^\cN(x,\xi)-\beta, 0)
        \end{array}
        \right.\\
    = & \left\{
        \begin{array}{r@{~}l}
          \inf & \displaystyle \eta' -\exp(-r) \exp\left(\int \log\left(\eta'- \max(\loss^\cN(x, \xi),\beta)\right) \d \Pb'(\xi)\right)-\beta\alpha+\alpha\textstyle\max_{\xi\in \Sigma}\loss(x,\xi)\\
          \st & \beta\leq \max_{\xi\in \Sigma} \ell(x, \xi), \eta' \geq \max_{\xi \in \Sigma} \max(\loss^\cN(x,\xi), \beta)
        \end{array}
            \right.\\
    = & \left\{
        \begin{array}{r@{~}l}
          \inf & \displaystyle \eta' -\exp(-r) \exp\left(\int \log\left(\eta'-\max\left(\loss^\cN(x, \xi),{\textstyle\max_{\xi\in \Sigma}\loss(x,\xi)}-\beta'\right)\right) \d \Pb'(\xi)\right)+\beta'\alpha\\
          \st & \beta'\geq 0, \eta' \geq \max_{\xi \in \Sigma} \loss(x,\xi)
        \end{array}
        \right.
  \end{align*}
  \endgroup
  The second equality follows from the fact that the logarithm function is an increasing function. The third and fourth inequalities follow from the change of variables $\eta'=\eta+\beta$ and $\beta'=\max_{\xi\in\Sigma}\ell(x, \xi)-\beta$. Finally, remark that $\max_{\xi\in\Sigma}\ell^\cN(x, \xi) = \max_{\xi\in\Sigma}\ell(x, \xi)$ as $0\in \cN$.
\end{proof}

\subsection{Continuity of the HR predictor}
\label{sec:cont-hr-pred}

\begin{lemma}\label{lemma: HR continuous}
For $\alpha>0$, $r>0$ and $x\in \cX$, the cost predictor
$\hat{\Pb} \mapsto \hat{c}^{\cN,\alpha,r}_{\HRo}(x,\hat{\Pb})$ is weakly continuous.
\end{lemma}
\begin{proof}
  Let again $\loss^{\infty}(x) = \max_{\xi\in \Sigma} \loss(x,\xi)$ and define
\begin{equation}\label{proof eq: c}
    \begin{array}{r@{~}l}
      \hat{c}_{\epsilon}(\hat{\Pb})\defn \inf & f(\beta, \eta)\defn
     \eta - \exp(-r) \exp\left( \int \log\left(\eta - \max(\loss^\cN(x, \xi), \loss^{\infty}(x)-\beta)\right) \d\hat{\Pb}(\xi) \right) + \beta \alpha\\
      \st & \beta \geq 0, \; \eta \geq \loss^{\infty}(x)+\epsilon
    \end{array}
  \end{equation}
  for all $\hat{\Pb}\in \cP$ and $\epsilon>0$. 
  Let us first show that the feasible set of the optimization problem in Equation \eqref{proof eq: c} can be restricted to a compact set independent of $\hat \Pb$ without loss of optimality.
  For any $\beta> M_1\defn \loss^{\infty}(x) - \min_{\xi\in\Sigma} \loss^\cN(x,\xi)$ we have that the objective function simplifies to
  \[
    f(\beta, \eta)=\eta - \exp(-r) \exp\left( \int \log\left(\eta - \loss^\cN(x, \xi)\right) \d\hat{\Pb}(\xi) \right) + \beta \alpha
  \]
  which is strictly increasing in $\beta$ as $\alpha>0$. Therefore, we may impose $\beta \leq M_1$ without loss of optimality.
  \citet[Proposition 5]{van2021data} ensures that for any fixed $\beta$ the minimum in $\eta$ is attained at 
  \[
    \ell^\infty(x)\leq \eta \leq M_2=\frac{\ell^\infty(x)+\epsilon-\exp(-r)\min_{\xi\in \Sigma}\ell^{\cN}(x, \xi)}{1-\exp(-r)}.
  \]
    
  Let us show the continuity of the function $c_{\epsilon}: \cP \rightarrow \Re$, defined as
  \begin{equation}\label{proof eq: cepsilon}
    c_{\epsilon} (\hat{\Pb}) =
    \left\{
      \begin{array}{rl}
        \inf & \eta - \exp(-r) \exp\left( \int \log\left(\eta - \max(\loss^\cN(x, \xi), \loss^{\infty}(x)-\beta)\right) \d\hat{\Pb}(\xi) \right) + \beta \alpha\\[0.5em]
        \st & M_1 \geq \lambda \geq 0, \; M_1 \geq \beta \geq 0, \; M_2\geq \eta \geq \loss^{\infty}(x) + \epsilon.
      \end{array}
    \right.
  \end{equation}
  for all $\hat{\Pb}\in \cP$. We will use the maximum theorem of \cite[p.\ 116]{berge1997topological} to show that the function $c_\epsilon:\cP \to \Re$ is continuous. It suffices to show that the set defined by the constraints is compact and that the objective function is continuous.
  First, the feasible region is clearly compact and nonempty.
  We need to show that the objective function is continuous in $(\beta, \eta, \Pb)$.
  Consider an arbitrary sequence $(\beta_i, \eta_i, \Pb_i)$ for $i\geq 1$ with limit $(\beta, \eta, \Pb)$.
  Remark that we have
  \begin{align*}
    & \lim_{i\to\infty} \int \log\left(\eta - \max(\loss^\cN(x, \xi), \loss^{\infty}(x)-\beta)\right) \,\d\hat{\Pb}_i(\xi) 
   =  \int \log\left(\eta - \max(\loss^\cN(x, \xi), \loss^{\infty}(x)-\beta)\right) \,\d\hat{\Pb}(\xi)
  \end{align*}
  as the integrant is bounded in $[\log(\epsilon), \log(M_2-\min_{\xi\in\Sigma}\loss^{\cN}(x, \xi))]$ and continuous. It is immediate that we have also
   \begin{align*}
     & \abs{ \log\left(\eta - \max(\loss^\cN(x, \xi), \loss^{\infty}(x)-\beta)\right) - \log\left(\eta' - \max(\loss^\cN(x, \xi), \loss^{\infty}(x)-\beta')\right) } 
   \leq  \tfrac{\left(\abs{\eta-\eta'}+\abs{\beta-\beta'}\right)}{\epsilon}.
   \end{align*}
   for all $(\beta, \eta)$ and $(\beta', \eta')$ feasible in \eqref{proof eq: c}. Hence,
   \begin{align*}
      & \lim_{i\to\infty} \Bigg| \int \log\left(\eta_i - \max(\loss^\cN(x, \xi), \loss^{\infty}(x)-\beta_i)\right) \,\d\hat{\Pb}_i(\xi)  - \int \log\left(\eta - \max(\loss^\cN(x, \xi), \loss^{\infty}(x)-\beta)\right) \,\d\hat{\Pb}(\xi)\Bigg|\\
     \leq & \lim_{i\to\infty} \Bigg| \int \log\left(\eta_i - \max(\loss^\cN(x, \xi), \loss^{\infty}(x)-\beta_i)\right) \,\d\hat{\Pb}_i(\xi) - \int \log\left(\eta - \max(\loss^\cN(x, \xi), \loss^{\infty}(x)-\beta)\right) \,\d\hat{\Pb}_i(\xi)\Bigg|+\\
      & \lim_{i\to\infty} \Bigg|\int \log\left(\eta - \max(\loss^\cN(x, \xi), \loss^{\infty}(x)-\beta)\right) \,\d\hat{\Pb}_i(\xi) - 
      \int \log\left(\eta - \max(\loss^\cN(x, \xi), \loss^{\infty}(x)-\beta)\right) \,\d\hat{\Pb}(\xi)\Bigg|\\
     \leq & \lim_{i\to\infty} \int\tfrac{\left(\abs{\eta_i-\eta}+\abs{\beta_i-\beta}\right)}{\epsilon}  \,\d\hat{\Pb}_i(\xi)=0.
   \end{align*}
   As the exponential function is continuous, we can apply Berge's maximum theorem and conclude that $c_{\epsilon}$ is continuous.
   
  Following Lemma \ref{thm: general dual rep bis} it follows that $c_{\epsilon} (\hat{\Pb}) \geq \hat{c}^{\cN,\alpha,r}_{\HRo}(x,\hat{\Pb})$ for all $\hat{\Pb}$, as any feasible dual solution $(\beta ,\eta)$ in the problem \eqref{proof eq: cepsilon} is also feasible in problem \eqref{eq: general dual rep bis}. Furthermore, for a given $\hat{\Pb} \in \cP$, if  $(\beta ,\eta)$ is a feasible solution of problem \eqref{eq: general dual rep bis} with cost $c$, then the dual point $(\lambda,\beta ,\eta + \epsilon)$ is feasible in problem \eqref{proof eq: c} and its cost is less than $c + \epsilon$. Hence, $c_{\epsilon} (\hat{\Pb}) \leq \hat{c}^{\cN,\alpha,r}_{\HRo}(x,\hat{\Pb}) + \epsilon$. We conclude that for all $\hat{\Pb} \in \cP$ and $\epsilon>0$, $c_{\epsilon} (\hat{\Pb}) \geq \hat{c}^{\cN,\alpha,r}_{\HRo}(x,\hat{\Pb}) \geq c_{\epsilon} (\hat{\Pb}) - \epsilon$. Furthermore, as $c_{\epsilon}$ is continuous for all $\epsilon>0$, we have that $\hat{c}^{\cN,\alpha,r}_{\HRo}(x,\cdot)$ is continuous as uniform limits of continuous functions are continuous.
\end{proof}

\section{Application to Linear Regression and Classification}
\label{App: Linear-application}
\subsection{Linear Regression Example}
We consider here a linear regression problem based on historically observed data  $(X_t,Y_t)_{t\in[T]}$ and $L_1$ loss function $\loss(\theta, (X,Y)) = |\theta^\top X - Y|$.
Given a particular covariate and response variable $(X,Y) \in \Re^d \times \Re$, $d>1$, this loss function is convex in the regression vector $\theta$ which we assume is constrained in the convex parameter set $\Theta\subseteq \Re^d$. Ordinary $L_1$ regression considers minimizing the empirical risk 
\[
  \min_{\theta \in \Theta} ~\textstyle \frac{1}{T} \sum_{t\in [T]} |\theta^\top X_t - Y_t|.
\]
The resulting $L_1$ regressor is however well documented to be susceptible to overfitting effects. 

In what follows we first point out that several classical regression regularization schemes studied in the literature reduce to a particular case of the HR linear regression formulation introduced here. We do remark though that none of these regularization schemes protect against statistical error, noise and misspecification simultaneously as does HR linear regression. Typically, these regularization schemes are designed to protect against only one particular source of overfitting instead.

\paragraph{Robustness Against Statistical Error  $(\cN = \{0\}, r>0, \alpha=0)$}
Suppose our objective is robustness against statistical error only.
Assume here that a worst case has been observed and $\Sigma = \{(X_t, Y_t)\}_{t\in[T]}$. Hence, we have the identities $\max_{k\in [K]}\loss^{\cN}(x,\xi_k)= \max_{\xi \in \Sigma} \loss(x,\xi)=\max_{k\in [K]}\loss(x,\xi_k)$.
HR linear regression reduces following Equation \eqref{eq:hr alpha=0} to
\begin{equation*}
  \begin{array}{r@{~}l}
    \min_{\theta\in\Theta} \sup & \frac 1T \sum_{t\in[T]} p_t |\theta^\top X_t - Y_t|\\[0.5em]
    \st & p \in \Re_+^T , ~\sum_{t\in[T]} \frac{1}{T} \log(\tfrac{1}{(Tp_t)}) \leq r
  \end{array}
\end{equation*}
which is precisely the likelihood robust approach proposed in \citet{wang2016likelihood}. We remark again that this approach protects against statistical error but offers no safeguards against either noise or misspecification.

Let us investigate the regime in which the robustness radius decreases as $r(T)=r'/T$ with $T\to\infty$.
From \citet[Theorem 2]{namkoong2017variance} for $\Theta$ compact it follows that
$$
\left\{
      \begin{array}{rl}
        \sup & \frac 1T \sum_{t\in[T]} p_t |\theta^\top X_t - Y_t|\\[0.5em]
        \st & p \in \Re^T , ~\sum_{t\in[T]}\log(\tfrac{1}{(Tp_t)}) \leq r'
      \end{array}
\right.
=
\frac{1}{T}
  \sum_{t\in[T]}
  |\theta^\top X_t - Y_t|+
  \sqrt{\frac{2r'}{T} \Var_{\hat{\Pb}_T}(|\theta^\top \tilde{X} - \tilde{Y}|)
  }
  +\Delta_T(\theta) 
$$
with $\lim_{T\to\infty}\sup_{\theta\in\Theta} \Delta_T(\theta) \sqrt{T} = 0$ in probability which we recognize as the sample variance penalization approach of \citet{maurer2009empirical}. Hence, in the large data regime in the absence of noise and misspecification the HR regression formulations  offer a variance regularization interpretation. As observed though by \citet{namkoong2017variance} it should be remarked that the variance regularized formulation is not convex and hence it is often preferable to use convex likelihood robust formulation instead.

\paragraph{Robustness Against Noise $(\cN = \cB_2(0,\epsilon)\times \{0\}, r=0, \alpha =0)$}
Suppose now our objective to protect again noise in the data input.
As in \citet{xu2009robustness} we will assume that the covariate data may be noisy and subject to misspecification while the response is not. We assume more precisely that the covariate data is subjected to noise bounded in Euclidean norm while the response variable is noise free and hence take $\cN = \cB_2(0,\epsilon)\times \{0\}$ for $\epsilon\geq 0$. In this regression context we have that the inflated loss function can be evaluated efficiently and in fact admits the analytic expression \(\ell^{\cN}(\theta, (X, Y)) = |\theta^\top X - Y| + \epsilon \|\theta\|_2. \)

As here only robustness against norm bounded noise is desired, HR linear regression reduces following Equation (\ref{eq:hr r=0 alpha=0}) to
\begin{align*}
    \textstyle \min_{\theta\in \Theta} \frac{1}{T} \sum_{t\in[T]}\sup_{\|n_t\|_2 \leq \epsilon} |\theta^\top (X_t-n_t) - Y_t| 
    =
    \min_{\theta\in \Theta} \frac{1}{T} \sum_{t\in[T]} |\theta^\top X_t - Y_t| + \epsilon \|\theta\|_2
\end{align*}
which is readily identified as classical Tikhonov or Ridge regularization. This robust interpretation of classical Ridge regularization was pointed out by \citet{xu2009robustness} already. We do indicate though that classical Ridge regularization does not guard against either statistical noise or misspecification. In fact, as we assumed here that the noise set is of the form $\cN = \cB_2(0,\epsilon)\times \{0\}$ it can be remarked that Ridge regularization only protects against noisy covariate data.

\paragraph{Holistic Robustness $(\cN = \cB_2(0,\epsilon)\times \{0\}, r>0, \alpha >0)$}
While several classical ``robust'' formulations exist for linear regression, a natural question is how to combine the benefits of each formulation and provide efficient robustness. The previous sections in this paper suggest that Ridge regression protects against noise while SVP and the likelihood approach protect again statistical error. A natrual combination of both, with additional protection against misspecification, naturally spurs from the HR predictor.

Holistic robust linear regression can be defined following the discussion in Section \ref{sec:HR-robustness} for the particular loss function and noise sets considered here.
We assume here the existence of a compact set $\Sigma$ so that $(X, Y) \in \Sigma$ and take here for simplicity $\Sigma=\set{(X_t,Y_t)}{t\in[T]}+\cB_2(0,\epsilon')\times \{0\}$ for some sufficiently large $\epsilon'\geq \epsilon$.
Theorem \ref{thm: dual rep HD} suggests the following HR linear regression formulation
\begin{equation*}
  \left\{
    \begin{array}{rl}
      \inf &  \frac 1T \sum_{t \in [T]} w_t + \lambda (r-1) + \beta \alpha + \eta\\[0.5em]
      \st & \theta\in \Theta, \; w\in \Re^T, \; \lambda \geq 0, \; \beta \geq 0, \; \eta \in \Re,\\[0.5em]
           & w_t \geq \lambda \log \left( \frac{\lambda}{\eta - |\theta^\top X_t - Y_t| - \epsilon \|\theta\|_2}\right), ~ w_t \geq \lambda \log \left( \frac{\lambda}{\eta - \max_{t'\in[T]} \, |\theta^\top X_{t'} - Y_{t'}| - \epsilon' \|\theta\|_2 }\right)-\beta \quad \forall t \in [T], \\[0.5em]
           & \eta \geq \max_{t\in[T]} \, |\theta^\top X_t - Y_t| + \epsilon' \|\theta\|_2.
    \end{array}\right.
\end{equation*}
As we have pointed out in Section \ref{sec:HR-robustness}, this HR linear regression formulation is tractable convex minimization problems which can be solved using off-the-shelf optimization routines. Interestingly, this formulation does not appear to be a straightforward combination of the previous robust formulations at a first glance. However, it appears naturally from our HR-DRO results of Section \ref{sec:HR-robustness} and indeed offers efficient simultaneous robustness to statistical error, noise and misspecification.

\subsection{Linear Classification Example}

A linear classifier characterized by a coefficient vector $\theta\in \Theta\subseteq\Re^d$ and bias term $b \in \Re$ predicts class labels $Y\in \{-1,1\}$ as $\theta^\top X - b$ when given access to covariate data $X \in \Re^d$. 
Given historical data $(X_t,Y_t)_{t\in[T]}$, empirical risk minimization following \citet{bertsimas2019robust} would suggest to use a classifier associated with a minimizer in
\begin{align}\label{eq: ERM SVM}
   \min_{\theta\in \Theta,\,b\in \Re} \textstyle \frac{1}{T} \sum_{t\in [T]} \max\{ 1- Y_t(\theta^\top X_t -b), 0\}=\left\{
      \begin{array}{rl}
        \min & \frac{1}{T} \sum_{t\in [T]} \nu_t\\
        \st & \theta\in \Theta,~b\in \Re,~\nu\in\Re^T_+,\\
            & Y_t(\theta^\top X_t -b) \geq 1 - \nu_t \quad \forall t\in [T].
      \end{array}
      \right.
\end{align}
A data set is denoted as separable if there exists a separating classifier with coefficient vector $\bar \theta\in \Theta$ and bias term $\bar b\in \Re$ so that $Y_t(\bar \theta^\top X_t -\bar b)\geq 1$ for all $t\in [T]$.
Clearly, when the data points are separable the minimum in Equation \eqref{eq: ERM SVM} will be zero.
The hinge loss function $\ell((\theta, b), (X, Y))=\max\{ 1- Y(\theta^\top X -b), 0\}$ considered here however does not require the data to be separable.
As in the previous section we will assume that the covariate data may be noisy or may be misspecified while the labels data are correct.
We assume more precisely that the historical covariate data is subjected to noise bounded in Euclidean norm while the dependent data is noise free and hence take $\cN = \cB_2(0,\epsilon)\times \{0\}$ for $\epsilon\geq 0$. 

HR linear classification can now be defined following the discussion in Section \ref{sec:HR-pred} for the hinge loss function and noise set considered here.
Assume here that $(X, Y)\in \Sigma=\set{(X_t,Y_t)}{t\in[T]}+\cB_2(0,\epsilon')\times \{0\}$ for some sufficiently large $\epsilon'\geq \epsilon$.
In this context the inflated loss function can be characterized as \(\ell^{\cN}((\theta,b), (X, Y)) = \max (1- Y(\theta^\top X -b) + \epsilon\|\theta\|_2, 0).
\)
Theorem \ref{thm: dual rep HD} suggests the following holistic robust classifier
\begin{equation*}
  \left\{
    \begin{array}{rl}
      \inf & \frac 1T  \sum_{t \in [T]} w_t + \lambda (r-1) + \beta \alpha + \eta\\[0.5em]
      \st & \theta\in \Theta,  \; b\in \Re, \; w\in \Re^T, \; \lambda \geq 0, \; \beta \geq 0, \; \eta \in \Re_+,\\[0.5em]
           & w_t \geq \lambda \log \left( \frac{\lambda}{\eta} \right)\quad \forall t \in [T],\\[0.5em]
           & w_t \geq \lambda \log \left( \frac{\lambda}{\eta - 1+ Y_t(\theta^\top X_t -b) - \epsilon\|\theta\|_2} \right) \quad \forall t \in [T],\\[0.5em]
           & w_t \geq \lambda \log \left( \frac{\lambda}{\eta - 1+ \min_{t'\in[T]}Y_{t'}(\theta^\top X_{t'} -b) - \epsilon' \|\theta\|_2 }\right)-\beta \quad \forall t \in [T], \\[0.5em]
           & \eta \geq \max_{t\in[T]} \, 1- Y_t(\theta^\top X_t -b) + \epsilon\|\theta\|_2
    \end{array}\right.
\end{equation*}
to safeguard against statistical error, noise and misspecification.
As we have pointed out in Section \ref{sec:HR-robustness}, the HR linear classification formulations is a tractable convex minimization problems which can be solved using off-the-shelf optimization routines.

In the remainder of this section we shall visualize the effect each of the robustness parameters $(\cN,r,\alpha)$ has on the HR linear classifiers by way of a small example.
We consider data points $(X_t,Y_t)_{t\in[T]}$ distributed in the two dimensional square $[0,1]^2$.
The data points of the two classes are distributed following uniform distributions on rectangles on the two sides of the hyperplane with coefficient vector $\theta=(-1,1)$ and bias term $b=0$.
Figure \ref{fig: SVM Uniform} illustrates the data distribution along with the true classifier, and an ERM classifier run on smaller data sample.
In the following experiments, we do not observe data points from this distribution directly but rather from a corrupted distributions and study the effect of robustness provided by our HR approach.

\begin{figure}[t]
  \centering
  \begin{subfigure}{0.45\textwidth}
    \includegraphics[width=\textwidth]{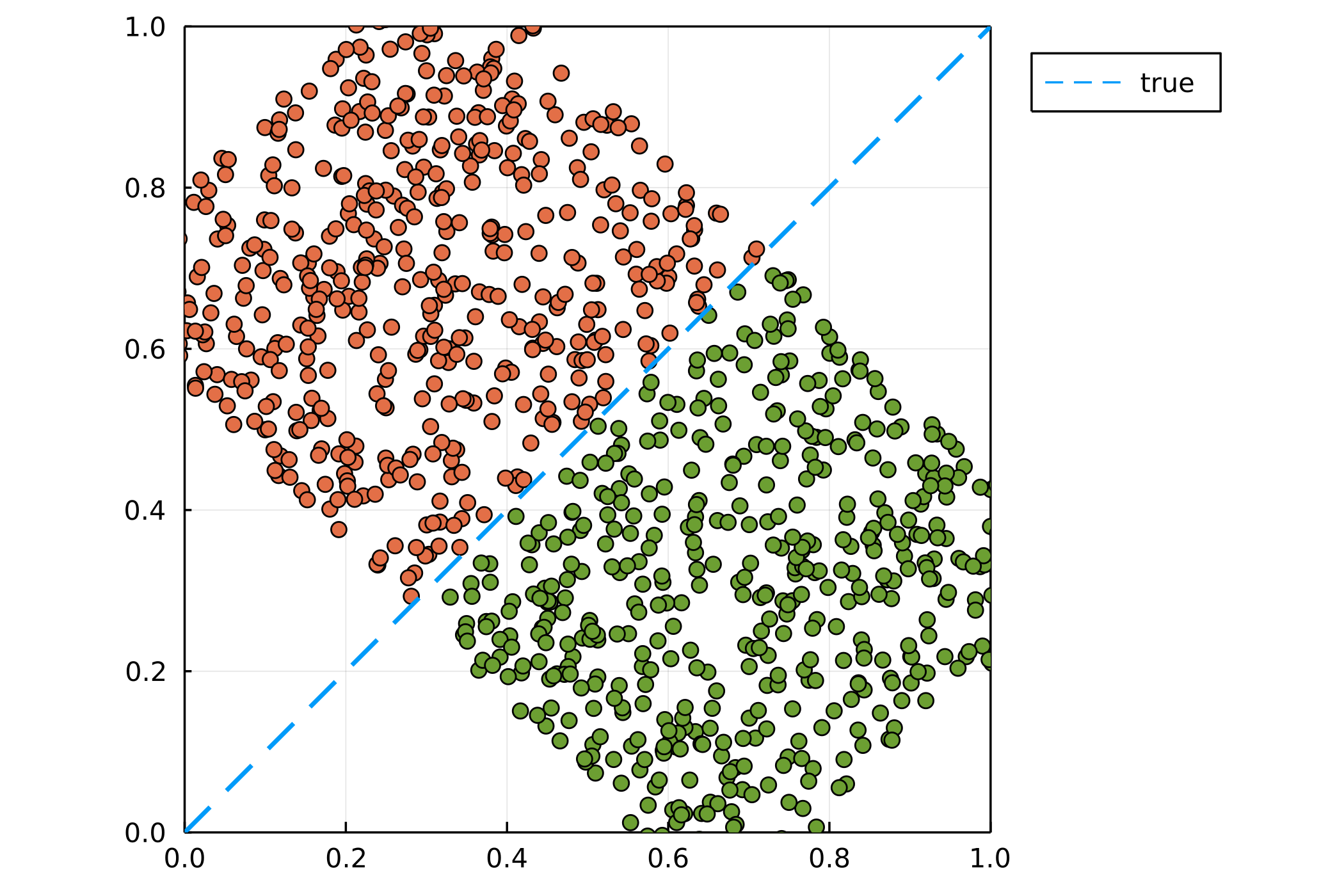}
    \caption{}
  \end{subfigure}
  \hfill
  \begin{subfigure}{0.45\textwidth}
    \includegraphics[width=\textwidth]{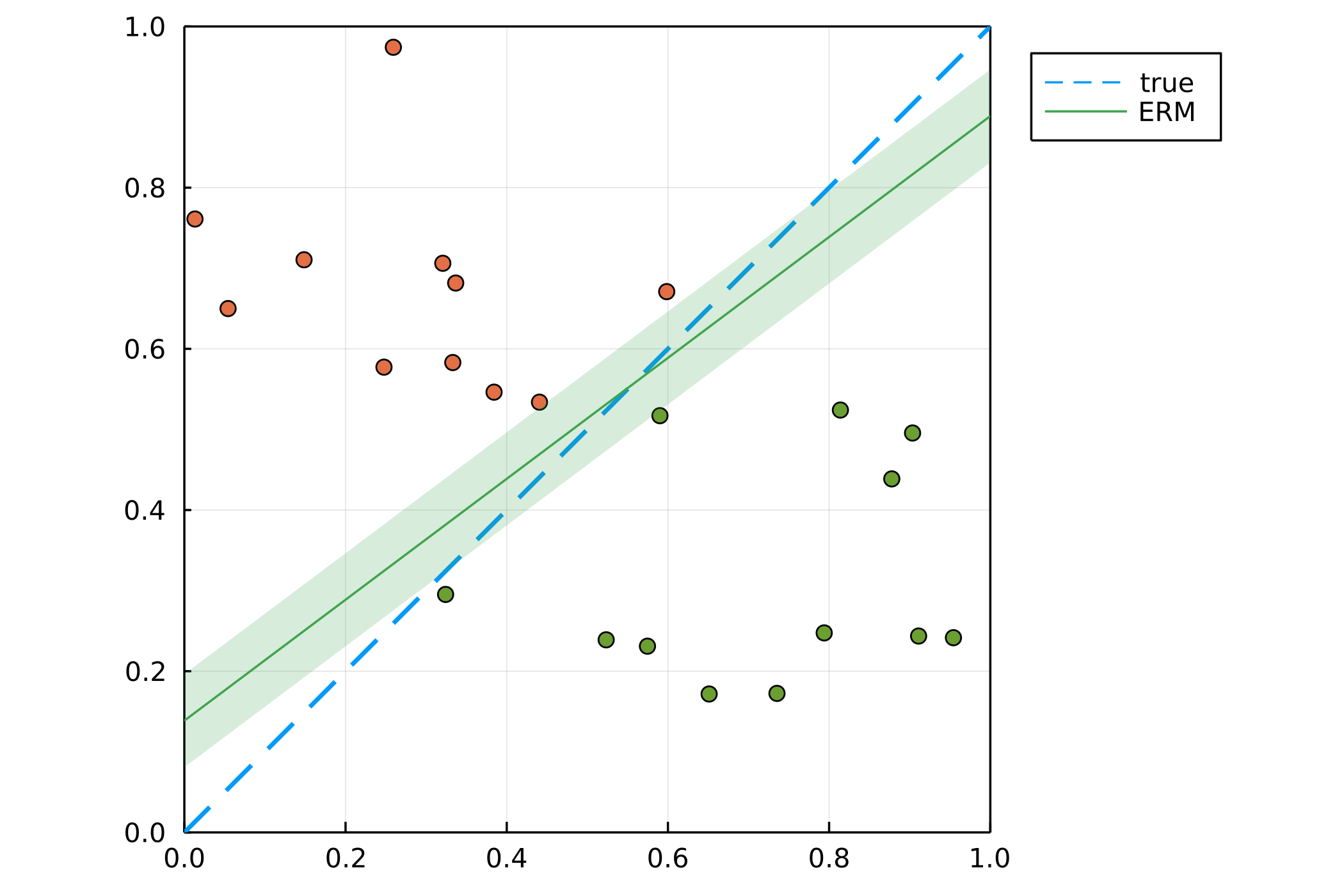}
    \caption{}
  \end{subfigure}
  \caption{(a) Data points of class $-1$ (in red) are distributed following a uniform distribution on a rectangle of center $(0.3,0.7)$, length 0.52 and width 0.46. Data points of class $1$ (in blue) are distributed following the symmetric of class $-1$ with respect to the true classifier $\theta=(-1,1),b=0$. (b) The ribbon around an ERM classifier represents its margin $\set{X\in\Re^d}{|\theta^\top X -b| \leq 1}$ with width $\tfrac{1}{\norm{\theta}_2}$. Here, the ERM loss is $0$ as the data is separable.}
  \label{fig: SVM Uniform}
\end{figure}

\paragraph{Robustness Against Noise $(\epsilon>0, r=0, \alpha =0)$}
Consider the HR robust classifier safeguarded against noise but not against either statistical error nor misspecification $(\epsilon>0, r=0, \alpha =0)$. The HR robust classifier can here be reduced to
\begin{align*}
  & \textstyle \min_{\theta\in\Theta,\,b\in \Re} \frac{1}{T} \sum_{t\in[T]} \max \left\{ 1- Y_t(\theta^\top X_t -b) + \epsilon \|\theta\|_2, 0\right\}\\
  = & \left\{
      \begin{array}{rl}
        \min &   \epsilon \|\theta\|_2+\frac{1}{T} \sum_{t\in[T]} \nu'_t\\
        \st  & \theta\in\Theta,~b\in \Re,~\nu\in\Re^T,\\
             & Y_t(\theta^\top X_t -b)\geq 1-\nu_t\quad \forall t\in[T],\\
             & \nu_t\geq -\epsilon \|\theta\|_2 \quad \forall t\in[T]
      \end{array}
      \right.
\end{align*}
which is nearly identical to the popular soft-margin SVM of \citet{cortes1995support} with the exception of the ultimate constraint; see also \citet[Appendix A]{bertsimas2019robust}.
The soft-margin SVM classifier balances between minimizing the total hinge loss and the maximization of the classifier margin $1/\norm{\theta}_2$. The HR support vector machine formulation enjoys the alternative interpretation that soft-margin SVM minimizes the maximum hinge loss with covariate data subjected to norm bounded noise; see also Figure \ref{fig: SVM HR noise}(a).

The equivalence between regularization and noise robustness for classification was noted already by \citet{xu2009robustness, bertsimas2019robust}.
Data distributions where soft-margin SVM is known to outperform ERM are normal distributions for each class.
We indicate that we can plausibly attribute this superior performance to the robustness against noise interpretation of soft-margin SVM.
Observe that normally distributed points can be seen as uniformally distributed points perturbed by (mostly) small noise; see Figure \ref{fig: noise generation}.

\begin{figure}[!htb]
  \centering
  \begin{subfigure}{0.45\textwidth}
    \includegraphics[width=\textwidth]{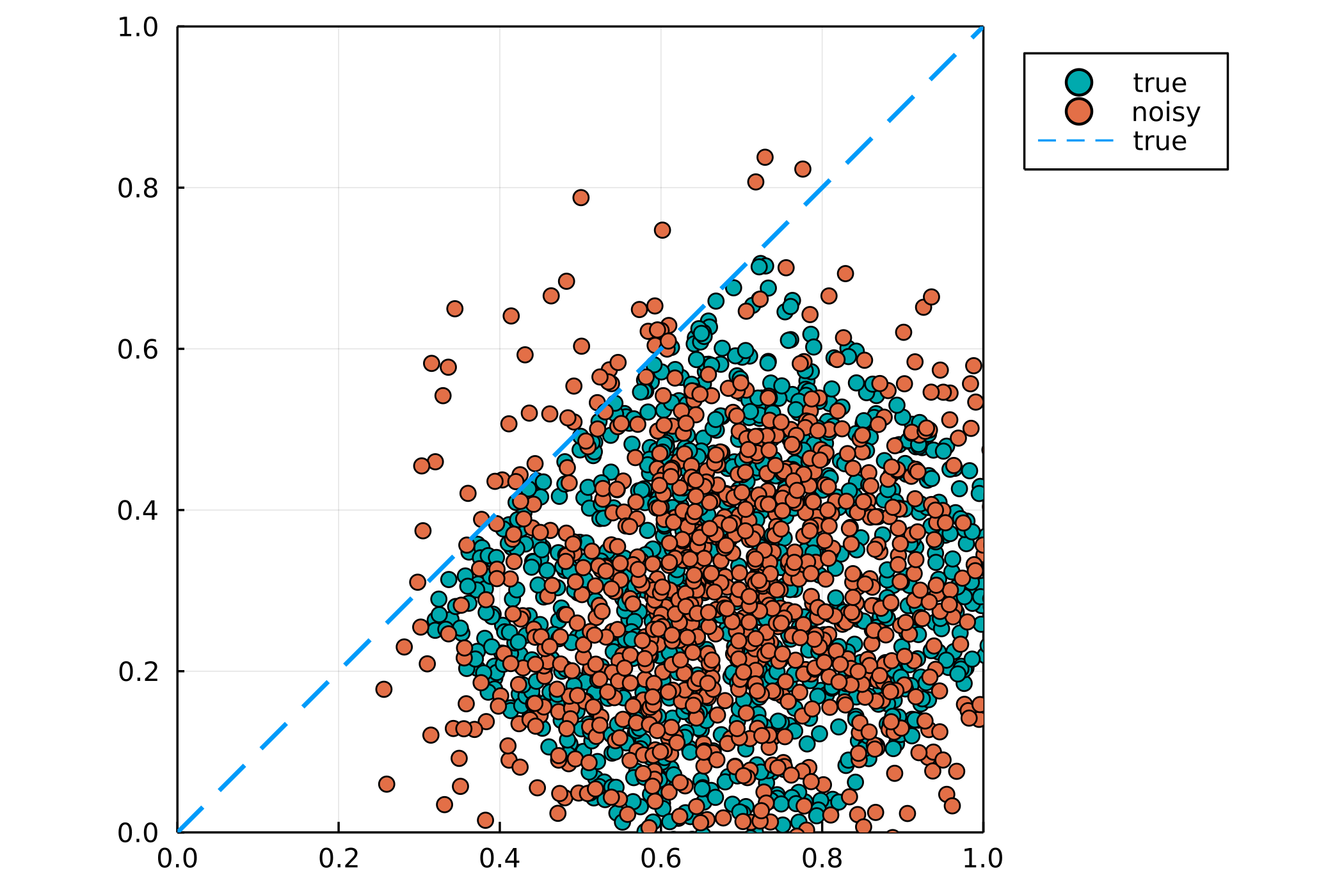}
    \caption{}
  \end{subfigure}
  \hfill
  \begin{subfigure}{0.45\textwidth}
    \includegraphics[width=\textwidth]{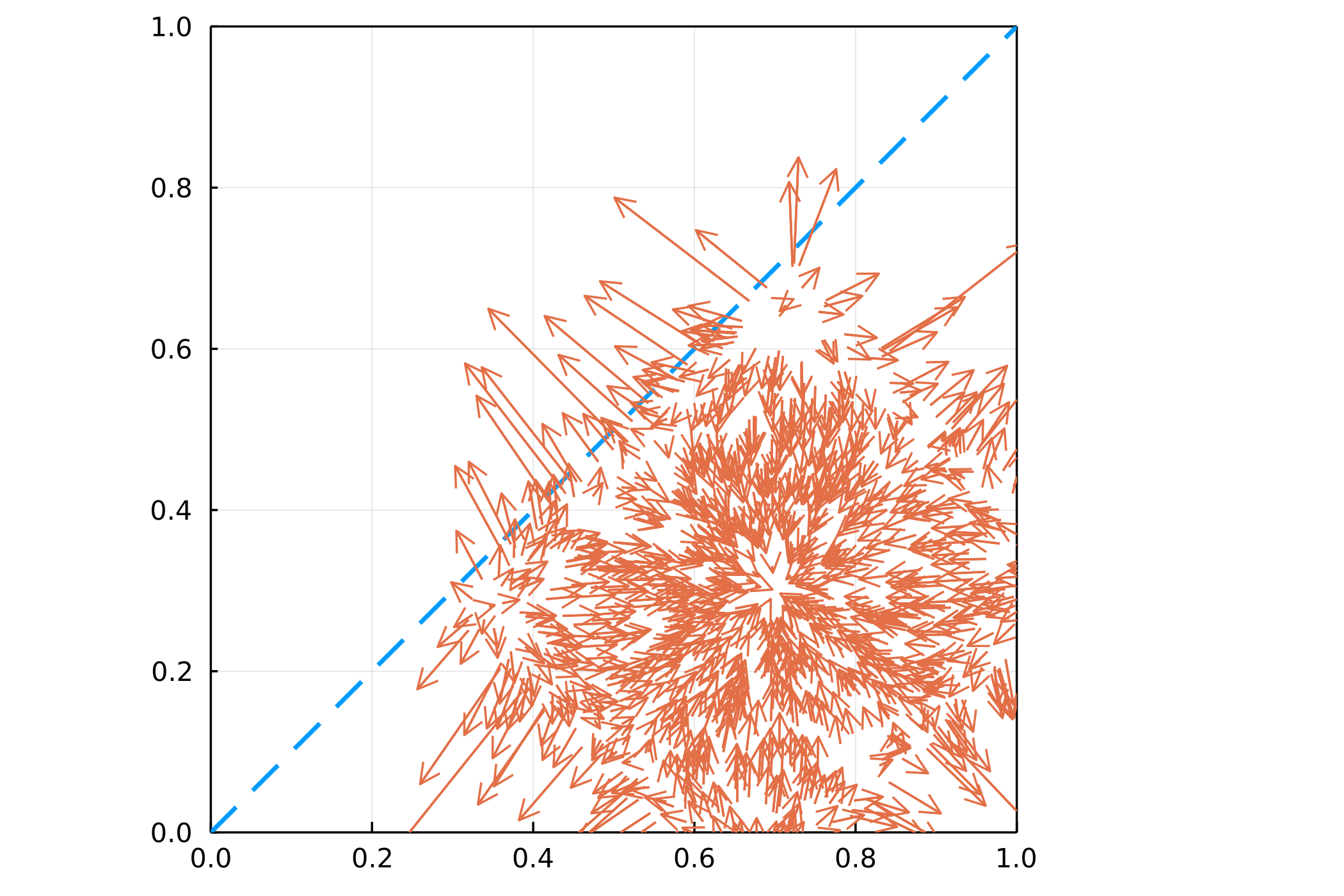}
    \caption{}
  \end{subfigure}
\caption{Let $\Pb_n^c$ denote a corrupted data distribution in which the uniform distributions characterizing $\Pb$ for both label classes have been replaced by an equivalent normal distribution sharing the same mean and variance.
 (a) True distribution $\Pb$ versus corrupted distribution $\Pb^c$ for a the points associated with the label 1. (b) Equivalent noise perturbation corrupting each of the uniform samples.}
\label{fig: noise generation}
\end{figure}

When run on the uncorrupted uniformally distributed data points, ERM and HR robust classification perform similarly as shown in Figure \ref{fig: SVM HR noise}(a). However, when run on the normally distributed noisy data points, HR robust classification out-performs ERM (see Figure \ref{fig: SVM HR noise}(b)).
One intuitive explanations is that the normally distributed cloud of points have a circular shape, hence, very few points actually are near the separating hyperplane $\tset{X\in\Re^d}{\theta\tpose X=b}$ and shape the classifier making it sensitive to noise in a few individual data points.
The soft-margin classifier, however, by encouraging a larger margin $1/\|\theta\|_2$ is associated with a separating hyperplane shaped by many more points decreasing sensitivity to the noise corrupting an individual data point.
The soft-margin classifier viewed as an HR robust classifier offers an alternative perspective on its superior performance.
The HR classifier considers the worst-case scenario of each data point perturbed by Euclidean bounded noise and hence it perturbs all data points in the direction of the hyperplane. Naturally, this leads to many more points in the vicinity of the hyperplane all of which shape the classifier thereby decreasing its noise sensitivity.

\begin{figure}[!htb]
  \centering
  \begin{subfigure}{0.45\textwidth}
    \includegraphics[width=\textwidth]{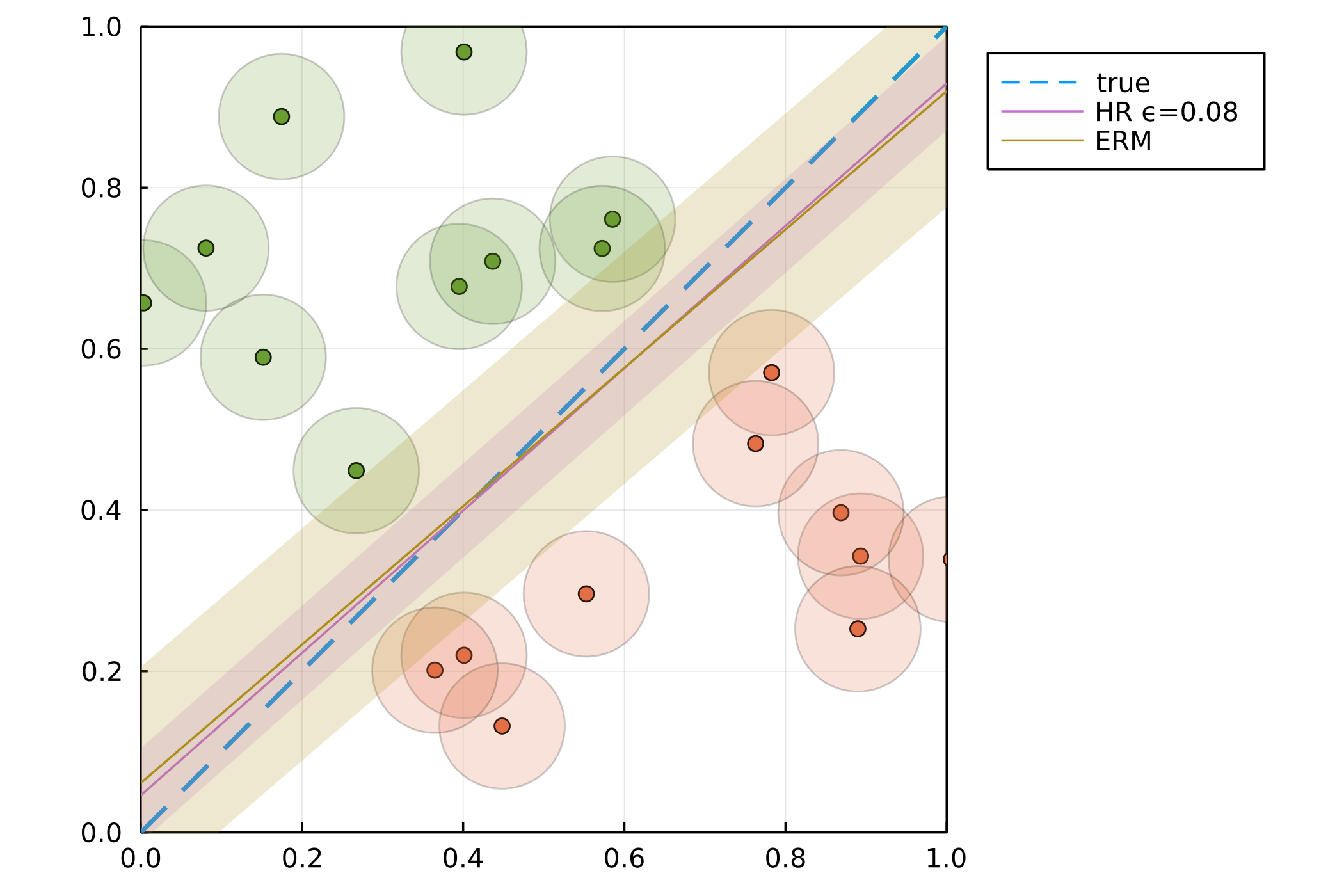}
    \caption{}
  \end{subfigure}
  \hfill
  \begin{subfigure}{0.45\textwidth}
    \includegraphics[width=\textwidth]{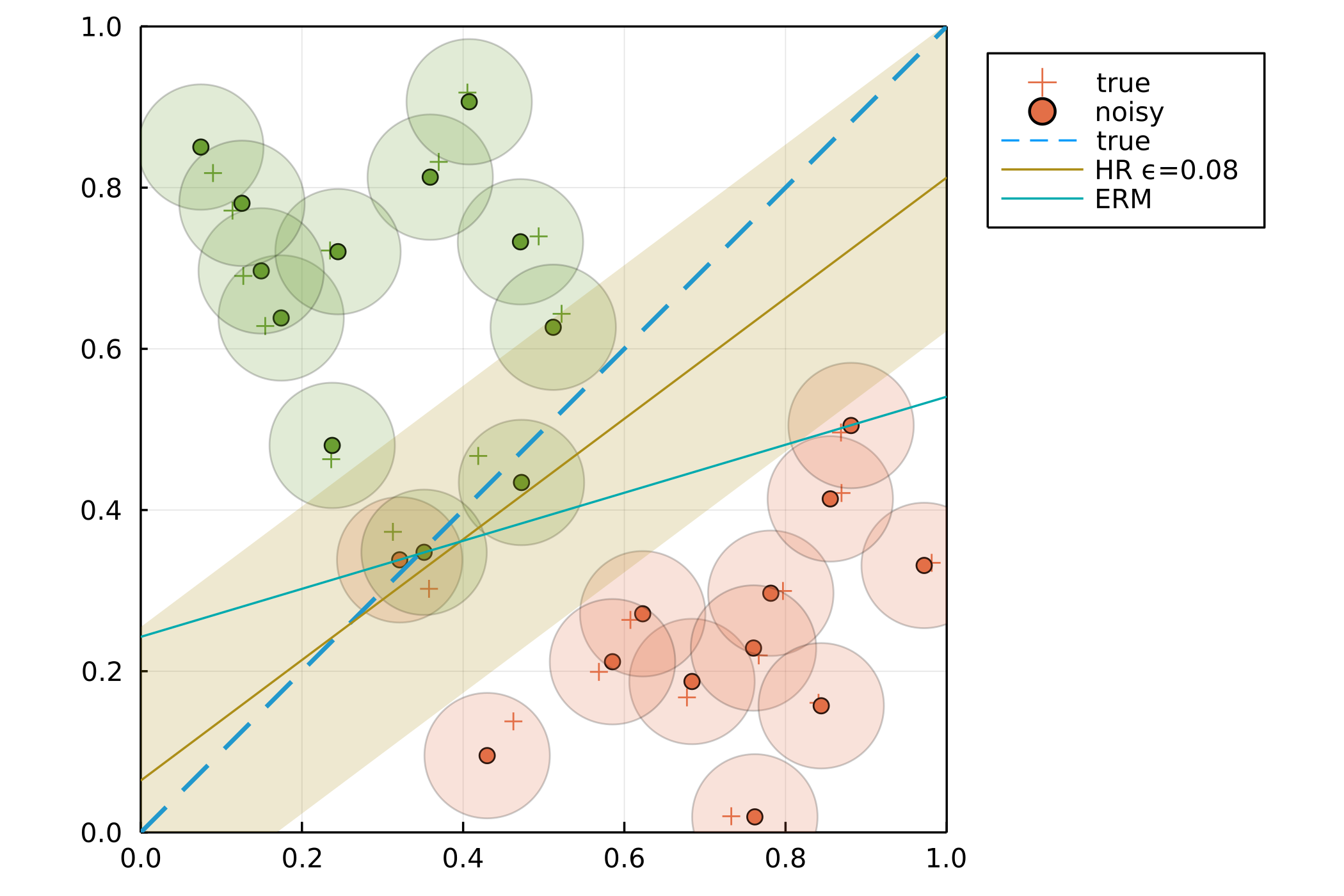}
    \caption{}
  \end{subfigure}
\caption{(a) ERM and HR robust classification  run on uncorrupted data points. The circles around each point represent the Euclidean noise ball of radius $\epsilon = 0.08$. (b) ERM and HR robust classification with $\epsilon = 0.08$ run on the normally distributed noisy data points.}
\label{fig: SVM HR noise}
\end{figure}

\paragraph{Robustness Against Noise and Statistical Error $(\epsilon=\epsilon'>0, r > 0, \alpha =0$)}
We consider now HR robust classification with $\epsilon>0$ and $r > 0$. That is, we desire robustness against noise and statistical error simultaneously but not to misspecification as here $\alpha =0$. The classifier following  Equation \eqref{eq:hr alpha=0} can be characterized here as the minimizer to
\[
  \begin{array}{r@{~}l}
    \min_{\theta\in\Theta,\,b\in \Re} \sup & \textstyle \sum_{t\in [T]} p_t \max\left\{ 1- Y_t(\theta^\top X_t -b) + \epsilon\|\theta\|_2, 0 \right\}\\
    \st & p \in \Re_+^T, \; \sum_{t\in [T]}p_t = 1, \; \sum_{t\in[T]} \frac{1}{T} \log(\tfrac{1}{(Tp_t)}) \leq r.
  \end{array}
\]
The HR robust classifier combines soft-margin classification which provides robustness against noise with weighting all data points according to their likelihood as advocated by \citet{wang2016likelihood}. Instead of taking the empirical average of the loss of the data points, HR-SVM indeed redistributes the weights of the data points to provide robustness against statistical error. 

Statistical error can play a prominent role in classification when the data is imbalanced.
Even when the out-of-sample distribution gives an equal chance of observing both label classes, by the luck of the draw there may be an imbalance between the number of observed data points in each label class.
Data imbalance becomes even more problematic if the out-of-sample distribution only gives a small chance of observing data associated with the minority label.
Both ERM or soft-margin SVM will on imbalanced data sets favor accuracy in labeling the majority class at the expense of the minority class.
Two common strategies to combat such data imbalance are artificially weighing the samples or by directly removing or adding samples \citep{chawla2002smote}.
Our HR robust classifier can be interpreted to reweigh the observed samples to avoid overfitting to data imbalance as well as safeguard against statistical error in general. The weighing here is based on the statistical likelihood of observing the sample and is closely related to empirical likelihood method by \citet{owen2001empirical}. It is important to note that this weighing is a function of the particular considered loss function and even adapts to the particular classifier considered. This in contrast to popular weighing methods \citep{chawla2002smote} which are generic and independent of the loss function.

Figure \ref{fig: SVM HR noise stat error} illustrate the HR robust classifier on an imbalanced data set. Figure \ref{fig: SVM HR noise stat error}(a) shows in particular that increasing the parameter $r$ weighs the data points in a way as to favor the minority class. Figure \ref{fig: SVM HR noise stat error}(b) exhibits these weights visually for $r=0.5$.

\begin{figure}[h]
  \centering
  \begin{subfigure}{0.45\textwidth}
    \includegraphics[width=\textwidth]{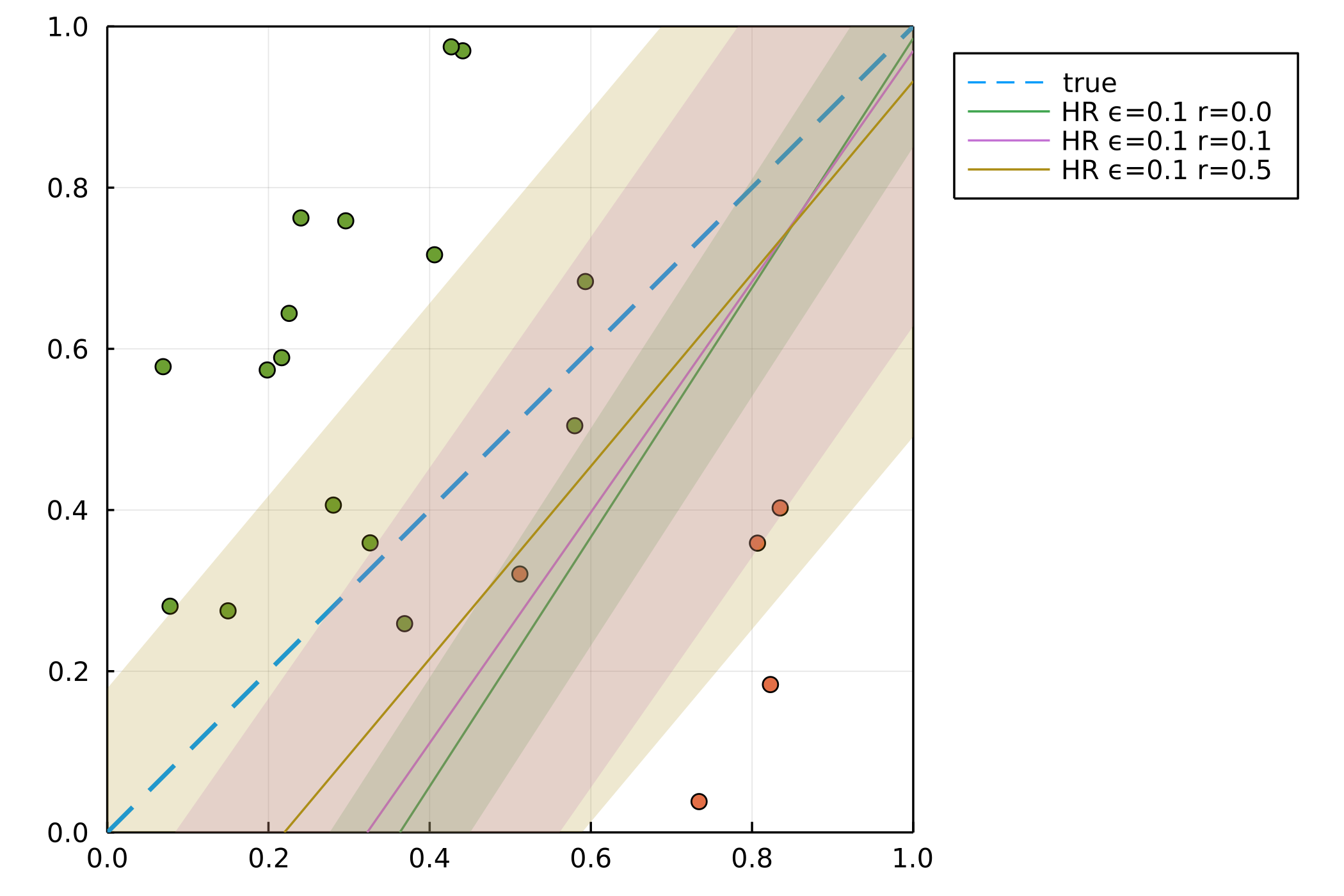}
    \caption{}
  \end{subfigure}
  \hfill
  \begin{subfigure}{0.45\textwidth}
    \includegraphics[width=\textwidth]{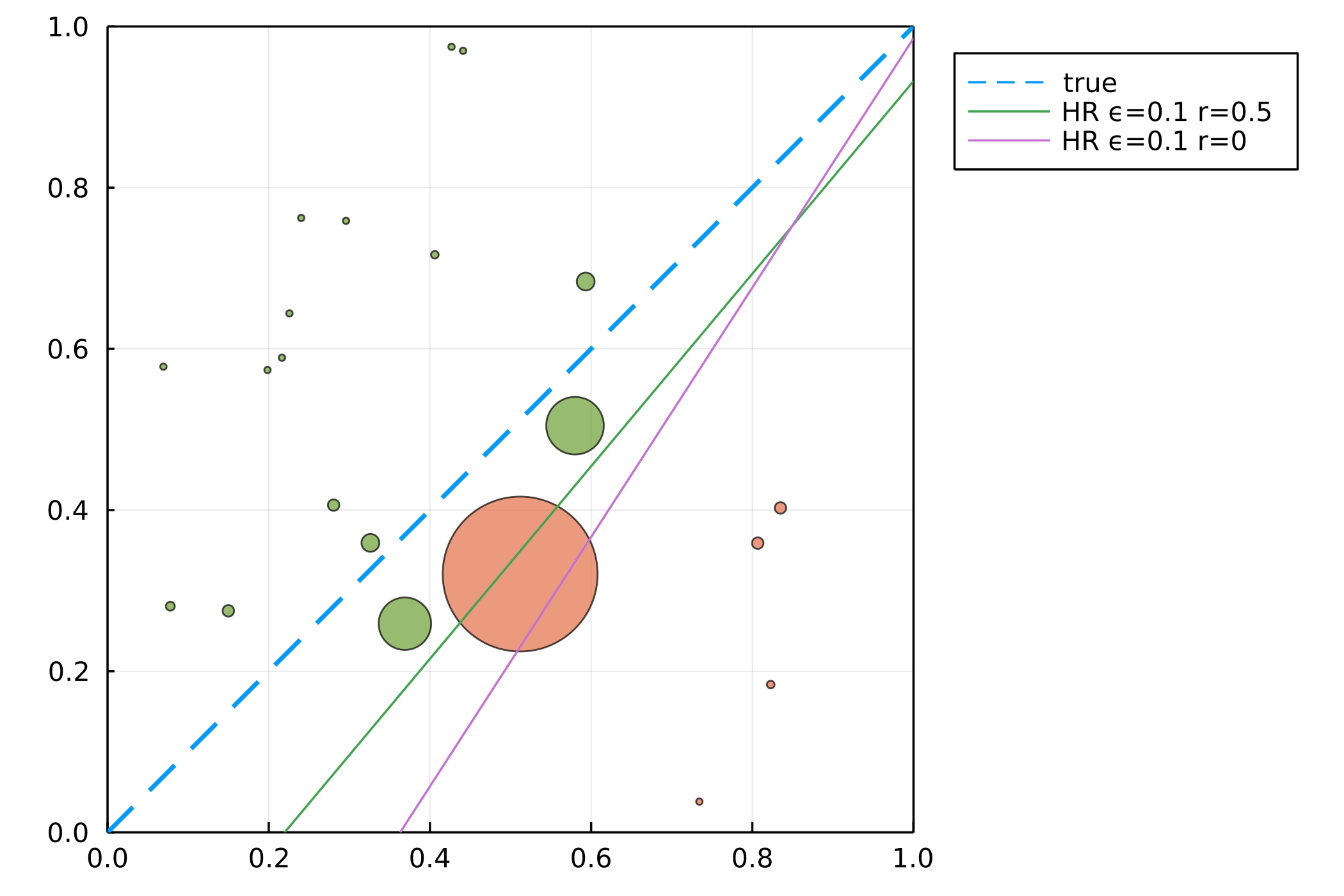}
    \caption{}
  \end{subfigure}
\caption{(a) The HR robust classifier with $r\in \{0,0.1,0.5\}$ on 25 data points with ratio $1/4$ imbalance between classes. (b) Weights characterizing the HR robust classifier with $r=0.5$ for each data point.}
\label{fig: SVM HR noise stat error}
\end{figure}

\subsubsection{Robustness Against Noise and Misspecification $(\epsilon=\epsilon'>0,  r=0, \alpha >0$)}

We remark that for the special case of $r=0$ HR SVMs reduce to another well known robust SVM approach, namely the $\nu$-SVMs of \citet{scholkopf2000new}. We consider here that the worst-case cost is independent of the classifier, and hence does not affect the optimization problem.
We can therefore write the HR-SVM following formulation Equation \eqref{eq:hr r=0}  as
\begin{align*}
    \min_{\theta\in\Theta,b\in B}\, (1-\alpha)\CVaR^{\alpha}_{\hat{\Pb}_T}\left( \max \left\{ 1- \tilde{Y}(\theta^\top \tilde{X} -b) + \epsilon\|\theta\|_2, 0 \right\} \right)
\end{align*}
This formulation results in the well known $\nu$-SVMs classifiers by \citet[Section 10]{akansu2016financial}.

\section{Application to Portfolio Selection}
\subsection{Hyperparameters}\label{sec: hyperparams}
We detail our choices of hyperparameters below. Denote by $\mathcal{U}_n([a,b])$ a uniform discretization of $[a,b]$ with $n$ points, starting at $a$ included.
The parameters of each model are chosen as:
\begin{itemize}
    \item \textbf{HR-DRO:} $\epsilon \in \{0\} \cup \{ 10^{-k} \; : \; k \in \mathcal{U}_{44}([-1,1])\}$, 
    $\alpha \in \{0\}$,
    $r \in \{0\} \cup\{ 10^{-k} \; : \; k \in \mathcal{U}_{44}([-1,3])\}$.
    \item \textbf{W-DRO:} $\epsilon \in  \{ 10^{-k} \; : \; k \in \mathcal{U}_{2000}([-1,3])\}$.
    \item \textbf{KL-DRO:} $r \in \{ 10^{-k} \; : \; k \in \mathcal{U}_{2000}([-1,3])\}$.
    \item \textbf{Mean-CVaR:} $\rho \in \mathcal{U}_{2000}([0,100])$.
    \item \textbf{Markowitz:} $\rho \in \mathcal{U}_{2000}([0,100])$.
\end{itemize}
\subsection{Selected Stock Tickers}
\label{stock-tickers}
\noindent MSFT, AAPL, LLY, JPM, XOM, WMT, UNH, PG, JNJ, COST, ORCL, HD, MRK, BAC, CVX, KO, AMD, PEP, TMO, ADBE, WFC, DHR, MCD, DIS, AMAT, TXN, ABT, CAT, GE, AXP, AMGN, VZ, PFE, IBM, NEE, CMCSA, UNP, SCHW, MU, COP, RTX, NKE, SPGI, INTC, ETN, HON, LOW, LRCX, SYK, T, C, PGR, BA, MDT, LMT, TJX, DE, ADI, KLAC, MMC, ADP, CI, FI, BMY, SO, WM, GD, DUK, CDNS, MO, SHW, CL, TT, ITW, EOG, TGT, CVS, CTAS, PH, NOC, SLB, BDX, ECL, CSX, EMR, USB, AON, PNC, FDX, MSI, WELL, APD, MMM, OXY, AJG, MNST, PCAR, VLO, TFC, AIG.

\subsection{Further experiments}
\label{app:other-performance}
In this section, we present further experiments exploring more in details insights of the portfolio selection experiment.
\begin{itemize}
    \item Figures \ref{fig:20231231-3}, \ref{fig:20231231-1}, \ref{fig:20230930-3}, \ref{fig:20230930-1} show the performance of the return risk, with and without validation of the benchmark methods with different data sets.
    \item Table \ref{tab:test-violation-formula-2}, \ref{tab:validation-violation-formula-1} and \ref{tab:validation-violation-formula-2} present testing and validation average risk tolerance threshold (positive part of standard deviation minus threshold), and testing and validation average risk tolerance deviate (standard deviation minus threshold).
    \item Figure \ref{fig:contour-plot-concatenated} shows a more detailed visualization of the influence of the three HR hyper-parameters on risk and return.
\end{itemize}

\begin{figure}[h]
    \begin{minipage}[b]{0.45\textwidth}
        \includegraphics[width=\textwidth]{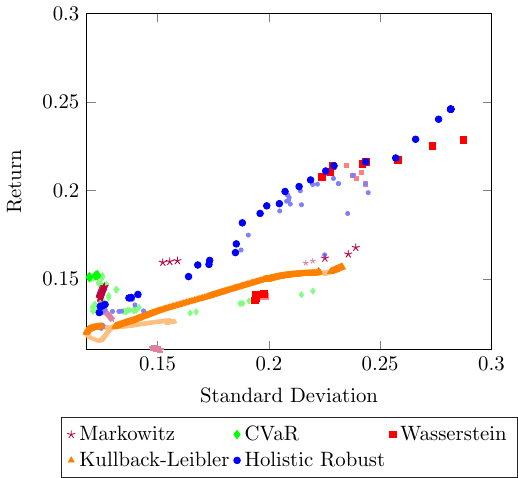}
        \caption{Testing return-risk with hyper-parameters chosen using validation. Dataset shifted by 1 quarter.}
        \label{fig:20231231-3}
    \end{minipage}
    \hfill
    \begin{minipage}[b]{0.45\textwidth}
        \includegraphics[width=\textwidth]{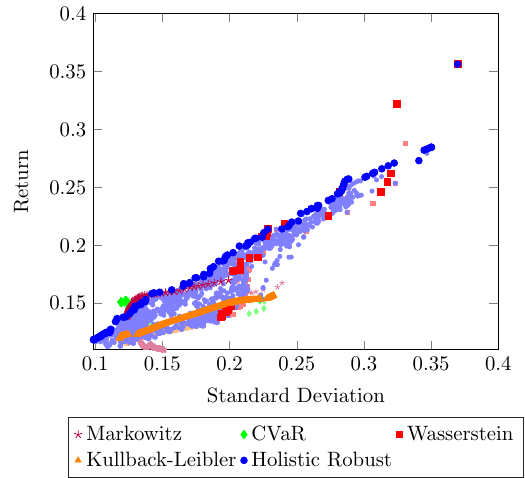}
        \caption{Full achievable testing return-risk with all hyper-parameters. Dataset shifted back by 1 quarter.}
        \label{fig:20231231-1}
    \end{minipage}
\end{figure}

\begin{figure}[h]
    \begin{minipage}[b]{0.45\textwidth}
        \includegraphics[width=\textwidth]{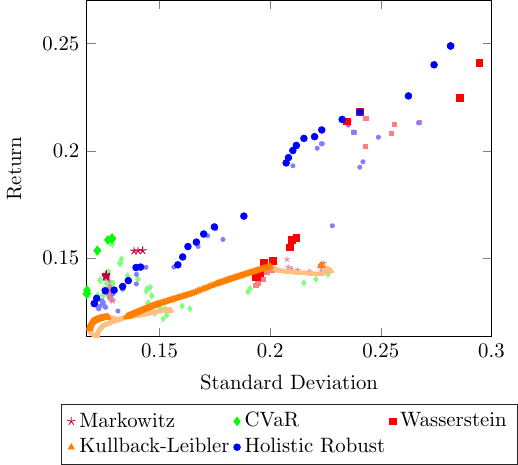}
        \caption{Testing return-risk with hyper-parameters chosen using validation. Dataset shifted back by 2 quarters.}
        \label{fig:20230930-3}
    \end{minipage}
    \hfill
    \begin{minipage}[b]{0.45\textwidth}
        \includegraphics[width=\textwidth]{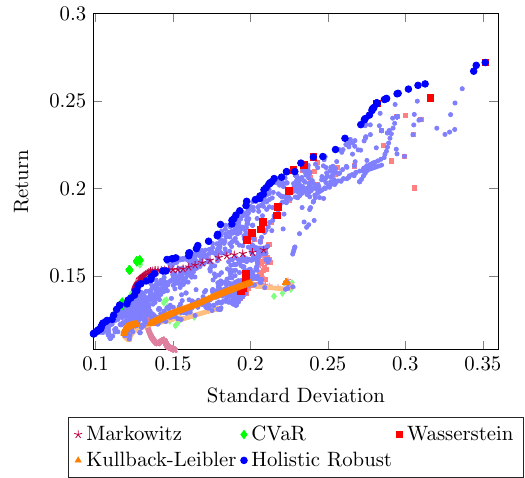}
        \caption{Full achievable testing return-risk with all hyper-parameters. Dataset shifted back by 2 quarters.}
        \label{fig:20230930-1}
    \end{minipage}
\end{figure}

\begin{figure}[h]
    \begin{minipage}[b]{0.45\textwidth}
        \includegraphics[width=\textwidth]{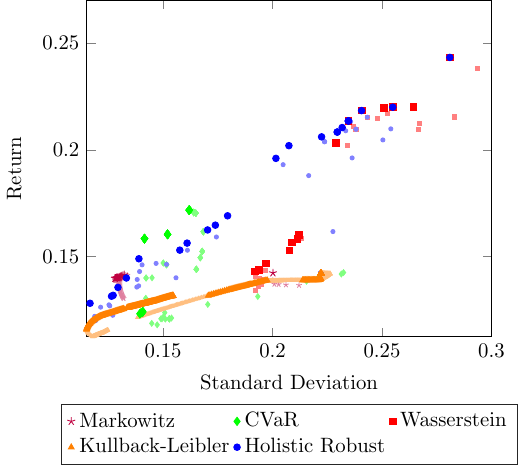}
        \caption{Testing return-risk with hyper-parameters chosen using validation. Dataset shifted by 3 quarters.}
        \label{20230630-3}
    \end{minipage}
    \hfill
    \begin{minipage}[b]{0.45\textwidth}
        \includegraphics[width=\textwidth]{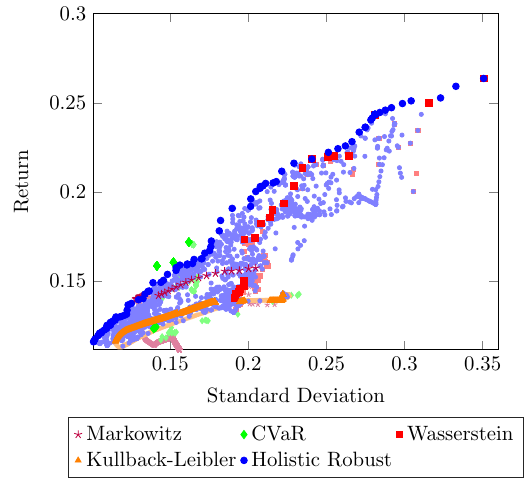}
        \caption{Full achievable testing return-risk with all hyper-parameters. Dataset shifted by 3 quarters.}
        \label{20230630-1}
    \end{minipage}
\end{figure}

\begin{table}[htbp]
\centering
\begin{tabular}{lccccc}
\toprule
Dataset & CVaR & W & KL & Markowitz & HR \\
\midrule
Original & $\boldsymbol{-5.20 \times 10^{-2}}$ & $1.70 \times 10^{-2}$ & $-4.32 \times 10^{-2}$ & $-3.77 \times 10^{-2}$ & \textbf{$-4.87 \times 10^{-3}$} \\
Shifted 1Q & $\boldsymbol{-5.11 \times 10^{-2}}$ & $1.85 \times 10^{-2}$ & $-3.02 \times 10^{-2}$ & $-2.15 \times 10^{-2}$ & \textbf{$2.88 \times 10^{-3}$} \\
Shifted 2Q & $\boldsymbol{-4.50 \times 10^{-2}}$ & $8.98 \times 10^{-3}$ & $-3.28 \times 10^{-2}$ & $-3.08 \times 10^{-2}$ & \textbf{$8.81 \times 10^{-4}$} \\
Shifted 3Q & $\boldsymbol{-4.17 \times 10^{-2}}$ & \textbf{$1.91 \times 10^{-2}$} & $-3.32 \times 10^{-2}$ & $-3.48 \times 10^{-2}$ & $2.02 \times 10^{-3}$ \\
\bottomrule
\end{tabular}
\caption{Average risk tolerance \textit{deviation} in the \textit{testing} set. Tolerance deviation is defined as \textit{testing} standard deviation minus the tolerance. Each dataset is shifted back by 0, 1, 2, or 3 quarters from the original dataset.}
\label{tab:test-violation-formula-2}
\end{table}

\begin{table}[htbp]
\centering
\begin{tabular}{lccccc}
\toprule
Dataset & CVaR & W & KL & Markowitz & HR \\
\midrule
Original & $3.45 \times 10^{-3}$ & $4.90 \times 10^{-3}$ & $5.67 \times 10^{-3}$ & $2.70 \times 10^{-3}$ & $\boldsymbol{6.01 \times 10^{-4}}$ \\
Shifted 1Q & $3.51 \times 10^{-3}$ & $5.47 \times 10^{-3}$ & $5.56 \times 10^{-3}$ & $2.57 \times 10^{-3}$ & $\boldsymbol{2.94 \times 10^{-5}}$ \\
Shifted 2Q & $3.39 \times 10^{-3}$ & $5.27 \times 10^{-3}$ & $5.94 \times 10^{-3}$ & $2.55 \times 10^{-3}$ & $\boldsymbol{0.0}$ \\
Shifted 3Q & $1.54 \times 10^{-3}$ & $5.11 \times 10^{-3}$ & $5.46 \times 10^{-3}$ & $1.83 \times 10^{-3}$ & $\boldsymbol{0.0}$ \\
\bottomrule
\end{tabular}
\caption{Average risk tolerance violation in the \textit{validation} set. Tolerance violation is defined as the positive part of the \textit{validation} standard deviation minus the tolerance. Each dataset is shifted back by 0, 1, 2, or 3 quarters from the original dataset.}
\label{tab:validation-violation-formula-1}
\end{table}

\begin{table}[htbp]
\centering
\begin{tabular}{lccccc}
\toprule
Dataset & CVaR & W & KL & Markowitz & HR \\
\midrule
Original & $-1.59 \times 10^{-2}$ & $\boldsymbol{-3.55 \times 10^{-2}}$ & $-7.27 \times 10^{-3}$ & $-1.49 \times 10^{-2}$ & $-2.75 \times 10^{-2}$ \\
Shifted 1Q & $-3.85 \times 10^{-2}$ & $\boldsymbol{-7.00 \times 10^{-2}}$ & $-4.08 \times 10^{-3}$ & $-1.48 \times 10^{-2}$ & $-3.75 \times 10^{-2}$ \\
Shifted 2Q & $-2.91 \times 10^{-2}$ & $\boldsymbol{-6.23 \times 10^{-2}}$ & $6.29 \times 10^{-5}$ & $-1.13 \times 10^{-2}$ & $-4.56 \times 10^{-2}$ \\
Shifted 3Q & $-1.96 \times 10^{-2}$ & $-2.09 \times 10^{-2}$ & $-1.82 \times 10^{-3}$ & $-2.34 \times 10^{-2}$ & $\boldsymbol{-2.64 \times 10^{-2}}$ \\
\bottomrule
\end{tabular}
\caption{Average risk tolerance \textit{deviation} in the \textit{validation} set. Tolerance deviation is defined as \textit{validation} standard deviation minus the tolerance. Each dataset is shifted back by 0, 1, 2, or 3 quarters from the original dataset.}
\label{tab:validation-violation-formula-2}
\end{table}

\begin{figure}[h]
  \centering
  \includegraphics[width=1\linewidth]{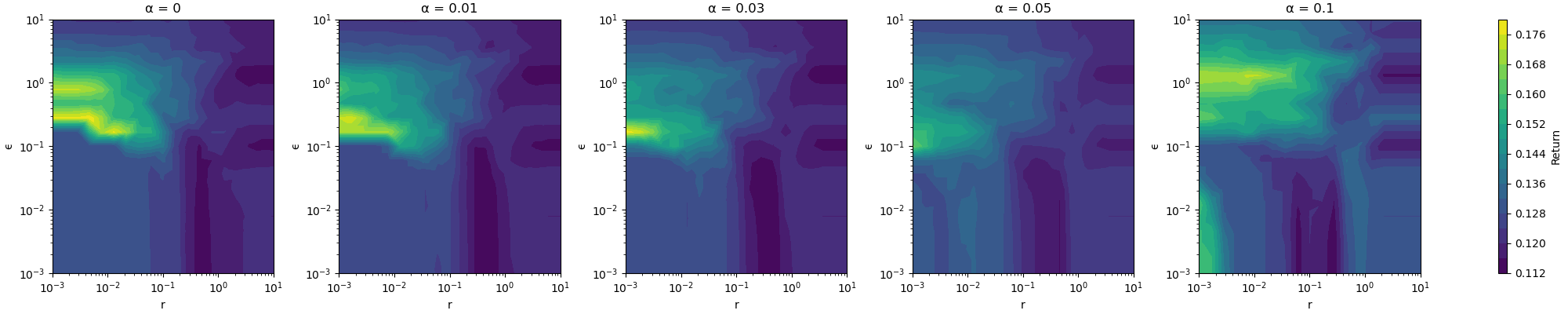}
  \includegraphics[width=1\linewidth]{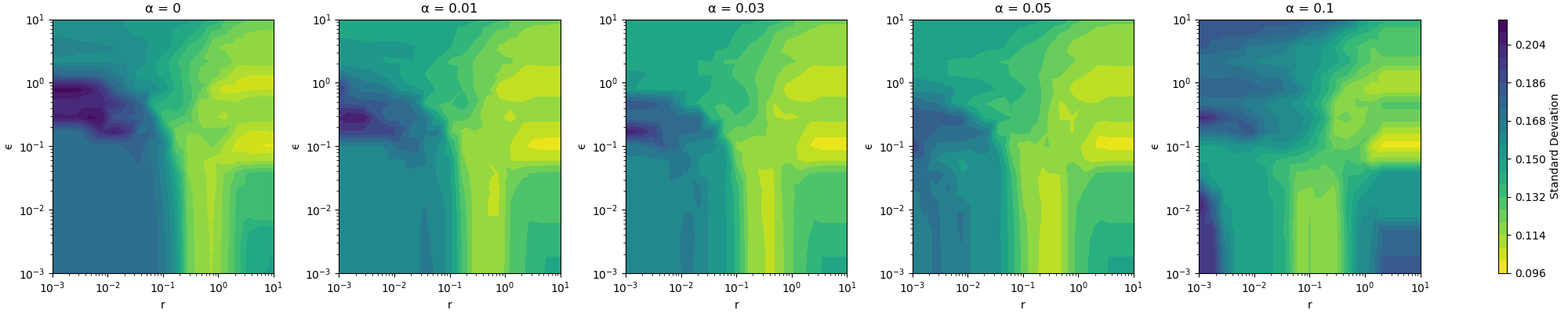}
  \includegraphics[width=1\linewidth]{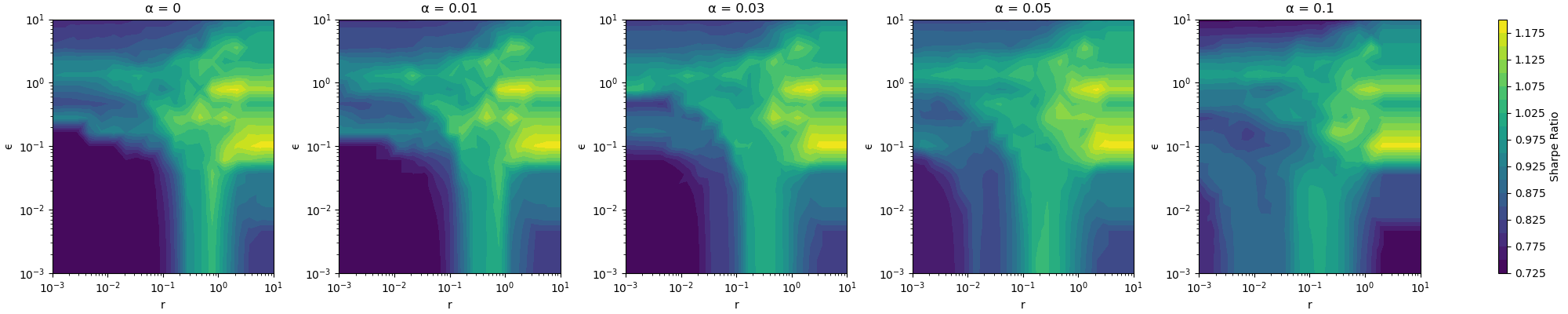}
  \caption{From top to bottom: Testing return, standard deviation, and Sharpe Ratio as a function of HR parameters. Brighter is better. Trained on data from \DTMdate{1995-03-31} to \DTMdate{2017-09-30} and tested on data from \DTMdate{2018-09-30} to \DTMdate{2023-09-30}. }
  \label{fig:contour-plot-concatenated}
\end{figure}

\end{document}

%% file: preamble.tex
\usepackage[margin=1in]{geometry}
\setlength{\parskip}{5pt}
\makeatletter
\def\thm@space@setup{%
\thm@preskip=1em \thm@postskip=0pt
}
\makeatother

\usepackage{natbib}
 \bibpunct[, ]{(}{)}{,}{a}{}{,}

\bibliographystyle{unsrtnat}

\usepackage[utf8]{inputenc}

\usepackage{amsmath, amssymb, bbold, mathtools}
\usepackage{footnote}
\usepackage{acronym}
\usepackage{enumerate}
\usepackage{hyperref}
\usepackage{multirow}
\usepackage{booktabs}
\usepackage{todonotes}
\usepackage[UKenglish]{datetime2}

\acrodef{dro}[DRO]{distributionally robust optimization}
\acrodef{ldt}[LDT]{large deviation theory}
\acrodef{ldp}[LDP]{large deviation principle}
\acrodef{lln}[LLN]{law of large numbers}
\acrodef{kl}[KL]{Kullback-Leibler}
\acrodef{iid}[{i.i.d.\ \!\!}]{independent identically distributed}
\acrodef{qp}[QP]{quadratic program}
\acrodef{qcqp}[QCQP]{quadratically constrained quadratic program}
\acrodef{vod}[VoD]{value of data}
\acrodef{saa}[SAA]{stochastic average approximation}
\acrodef{fl}[FL]{Fenchel-Legendre}

\usepackage{tikz}
\usetikzlibrary{shapes.geometric}
\usetikzlibrary{matrix,positioning}
\usepackage{pgfplots}
\pgfplotsset{compat=1.15}
\usepackage{mathrsfs}
\usetikzlibrary{arrows}
\usepackage{array}
\usepackage{float}

\definecolor{red}{RGB}{163, 31, 52}
\definecolor{gray}{RGB}{194, 192, 191}
\definecolor{blue}{RGB}{59, 89, 152}
\definecolor{green}{RGB}{2, 138, 67}

\newcommand{\keywords}[1]{\textbf{Keywords:} #1}


\newcommand{\norm}[1]{\left\|#1\right\|}

\newcommand{\abs}[1]{\left|#1\right|}
\newcommand{\E}[2]{\mathbf{E}_{#1} \left[ #2 \right]}

\newcommand{\one}[1]{\mathbb{1}\left\{#1\right\}}

\newcommand{\set}[2]{\left\{ #1\ : \ #2 \right\}}
\newcommand{\tset}[2]{\{ #1\ : \ #2 \}}
\newcommand{\tpose}{^\top}

\newcommand{\defn}[0]{:=}

\renewcommand{\tfrac}[2]{#1/#2}
\newcommand{\Prob}{{\rm{Prob}}}

\DeclareMathOperator*{\argmin}{arg\,min}
\DeclareMathOperator*{\argmax}{arg\,max}

\def\d{\mathrm{d}}
\def\st{\mathrm{s.t.}}

\DeclareMathOperator{\supp}{supp}

\DeclareMathOperator{\cl}{cl}
\DeclareMathOperator{\interior}{int}

\usepackage{amsthm}
\theoremstyle{plain}
\newtheorem{theorem}{Theorem}[section]
\newtheorem{lemma}[theorem]{Lemma}

\newtheorem{corollary}[theorem]{Corollary}
\theoremstyle{definition}

\newtheorem{remark}[theorem]{Remark}

\newtheorem{example}[theorem]{Example}

\newcommand{\LPdist}{\pi_{\mathrm{LP}}}
\newcommand{\KL}{{\mathrm{KL}}}
\newcommand{\LP}{\mathrm{LP}}

\newcommand{\indic}{\mathbb{1}}
\newcommand{\CVaR}{\mathrm{CVaR}}
\newcommand{\HR}{\mathrm{HR}}
\newcommand{\HRo}{\mathrm{HRo}}

\newcommand{\Eb}{\mathbb{E}}
\newcommand{\Pb}{\mathbb{P}}
\newcommand{\Qb}{\mathbb{Q}}
\newcommand{\cX}{\mathcal{X}}

\newcommand{\cU}{\mathcal{U}}
\newcommand{\cB}{\mathcal{B}}
\newcommand{\cP}{\mathcal{P}}

\newcommand{\cN}{\mathcal{N}}
\newcommand{\cA}{\mathcal{A}}
\newcommand{\cD}{\mathcal{D}}

\newcommand{\into}{\mathrm{o}}

\newcommand{\LPtxt}{\mathrm{LP}}

\newcommand{\txi}{\Tilde{\xi}}
\newcommand{\tn}{\Tilde{n}}
\newcommand{\tc}{\Tilde{c}}
\newcommand{\tzeta}{\Tilde{\zeta}}

\usepackage{subcaption}

\newcommand{\noise}{n}

\newcommand{\Pemp}[1]{\hat{\Pb}_{#1}}

\newcommand{\Var}{\mathrm{Var}}

\newcommand\numberthis{\addtocounter{equation}{1}\tag{\theequation}}

\newcommand{\loss}{\ell}
\renewcommand{\Re}{\mathbf{R}}
\newcommand{\integ}{\mathbf{N}}


%% file: figures/cvar.tex
\begin{tikzpicture}

\def\radius{0.43cm}
   \tikzstyle{sample} = [circle, minimum width = 2*\radius, very thick,
   draw=black!50!blue!60]
  \tikzstyle{arrow} = [thick,->,>=stealth]
   
  
   \node[sample,fill = green!50] (xi 9) at (9,5) {$\xi_{[T]}$};
    \foreach \x in {8,...,4}
    {
    \node[sample,fill = green!50] (xi \x) at (\x,5) {};
    }
   \begin{scope}
    \clip (3,5) circle (\radius);
        \fill[green!50] (3.1,6) rectangle (4,4);
    \end{scope}
    \node[sample] (xi 3) at (3,5) {$\xi_{[p]}$};
    \foreach \x in {2}
    {
    \node[sample] (xi \x) at (\x,5) {};
    }
    \node[sample] (xi 1) at (1,5) {$\xi_{[1]}$};
  
  \draw [orange, thick,
    decorate, 
    decoration = {brace,
        raise=2pt,
        amplitude=10pt}] (9.45,4.5) -- (3.1,4.5)
        node[pos=0.5,below=14pt,black]{Total loss = $(1-\alpha)\CVaR^{\alpha}$};

\draw[thick,dashed, black!50!blue!60] (3.1,4.3) -- (3.1,6) node[above] {$\alpha T$};

\draw[thick, black!50!blue!60] (0.5,4.4) -- (3.05,4.4);
\draw[thick, color = black!50!blue!60, text width=3cm] (0.5,4.3) -- (3.05,4.3) node[pos=0.7, below=10pt] {{\color{black}Replace with $\alpha \max_{\xi\in \Sigma} \loss(x,\xi)$}};

\draw[arrow] (6.5,6.3) -- (9,6.3) node[above, pos=0.5] {higher loss $\loss^{\mathcal N}$};
\end{tikzpicture} 

%% file: figures/post-sample.tex
\begin{tikzpicture}


   \tikzstyle{dist} = [rectangle, minimum width=1.6cm, minimum height=1.7cm, text centered, text width=3cm, draw=black]
   \tikzstyle{arrow} = [thick,->,>=stealth]
   
   
   \node[dist] (P') at (-1,5) {True Distribution \\ $\Pb$};
   \node[dist] (Q) at (5,5) {Empirical distribution $\hat{\Qb}$};
   \node[dist] (PT) at (11,5) {Corrupted empirical distribution $\hat{\Pb}_T$};

   \draw[->] (P'.north) to[bend left] node[above,midway,align = center] {sample}  (Q) ;
  \draw[<->] (P'.east) -- (Q.west) node[below,midway,align = center] {KL divergence \\ $\leq r$};
  \draw[<->] (Q.east) -- (PT.west) node[below,midway,align = center] {LP distance \\ $\leq \alpha$};
  \draw[->] (Q.north) to[bend left] node[above,midway,align = center] {noise and misspecification} (PT) ;
  
\end{tikzpicture} 

%% file: figures/on-sample.tex
\begin{tikzpicture}


   \tikzstyle{dist} = [rectangle, minimum width=1.4cm, minimum height=1.7cm, text centered, text width=3cm, draw=black]
   \tikzstyle{arrow} = [thick,->,>=stealth]
   
   
   \node[dist] (P') at (-1,5) {True Distribution \\ $\Pb$};
   \node[dist] (Q) at (5,5) {Corrupted Distribution \\ $\Qb$};
   \node[dist] (PT) at (11,5) {Observed Samples \\ $\hat{\Pb}_T$};

   \draw[->] (P'.north) to[bend left] node[above,midway,align = center] {noise and misspecification}  (Q) ;
  \draw[<->] (P'.east) -- (Q.west) node[below,midway,align = center] {LP distance \\ $\leq \alpha$};
  \draw[<->] (Q.east) -- (PT.west) node[below,midway,align = center] {KL divergence \\ $\leq r$};
  \draw[->] (Q.north) to[bend left] node[above,midway,align = center] {sample} (PT) ;
  
\end{tikzpicture} 


%% file: HR_v3.bbl
\begin{thebibliography}{68}
\providecommand{\natexlab}[1]{#1}
\providecommand{\url}[1]{\texttt{#1}}
\expandafter\ifx\csname urlstyle\endcsname\relax
  \providecommand{\doi}[1]{doi: #1}\else
  \providecommand{\doi}{doi: \begingroup \urlstyle{rm}\Url}\fi

\bibitem[Smith and Winkler(2006)]{smith2006optimizer}
James~E. Smith and Robert~L. Winkler.
\newblock The optimizer’s curse: Skepticism and postdecision surprise in
  decision analysis.
\newblock \emph{Management Science}, 52\penalty0 (3):\penalty0 311--322, 2006.

\bibitem[Bertsimas et~al.(2018{\natexlab{a}})Bertsimas, Gupta, and
  Kallus]{bertsimas2018data}
Dimitris Bertsimas, Vishal Gupta, and Nathan Kallus.
\newblock Data-driven robust optimization.
\newblock \emph{Mathematical Programming}, 167:\penalty0 235--292,
  2018{\natexlab{a}}.

\bibitem[Bertsimas et~al.(2018{\natexlab{b}})Bertsimas, Gupta, and
  Kallus]{bertsimas2018robust}
Dimitris Bertsimas, Vishal Gupta, and Nathan Kallus.
\newblock Robust sample average approximation.
\newblock \emph{Mathematical Programming}, 171:\penalty0 217--282,
  2018{\natexlab{b}}.

\bibitem[Lam(2019)]{lam2019recovering}
Henry Lam.
\newblock Recovering best statistical guarantees via the empirical
  divergence-based distributionally robust optimization.
\newblock \emph{Operations Research}, 67\penalty0 (4):\penalty0 1090--1105,
  2019.

\bibitem[Van~Parys et~al.(2021)Van~Parys, Esfahani, and Kuhn]{van2021data}
Bart~P.G. Van~Parys, Peyman~Mohajerin Esfahani, and Daniel Kuhn.
\newblock From data to decisions: Distributionally robust optimization is
  optimal.
\newblock \emph{Management Science}, 67\penalty0 (6):\penalty0 3387--3402,
  2021.

\bibitem[Love and Bayraksan(2015)]{love2015phi}
David Love and G{\"u}zin Bayraksan.
\newblock Phi-divergence constrained ambiguous stochastic programs for
  data-driven optimization.
\newblock \emph{Technical report, Department of Integrated Systems Engineering,
  The Ohio State University, Columbus, Ohio}, 2015.

\bibitem[Duchi et~al.(2021)Duchi, Glynn, and Namkoong]{duchi2021statistics}
John~C. Duchi, Peter~W. Glynn, and Hongseok Namkoong.
\newblock Statistics of robust optimization: A generalized empirical likelihood
  approach.
\newblock \emph{Mathematics of Operations Research}, 46\penalty0 (3):\penalty0
  946--969, 2021.

\bibitem[Bennouna and Van~Parys(2021)]{bennouna2021learning}
M.~Bennouna and Bart~P.G. Van~Parys.
\newblock Learning and decision-making with data: Optimal formulations and
  phase transitions.
\newblock \emph{arXiv preprint arXiv:2109.06911}, 2021.

\bibitem[Vapnik(1999)]{vapnik1999nature}
Vladimir Vapnik.
\newblock \emph{The nature of statistical learning theory}.
\newblock Springer science \& business media, 1999.

\bibitem[Maurer and Pontil(2009)]{maurer2009empirical}
Andreas Maurer and Massimiliano Pontil.
\newblock Empirical bernstein bounds and sample variance penalization.
\newblock In \emph{Conference on Learning Theory}, 2009.

\bibitem[Madry et~al.(2018)Madry, Makelov, Schmidt, Tsipras, and
  Vladu]{madry2017towards}
Aleksander Madry, Aleksandar Makelov, Ludwig Schmidt, Dimitris Tsipras, and
  Adrian Vladu.
\newblock Towards deep learning models resistant to adversarial attacks.
\newblock In \emph{International Conference on Learning Representations}, 2018.

\bibitem[Sinha et~al.(2017)Sinha, Namkoong, Volpi, and
  Duchi]{sinha2017certifying}
Aman Sinha, Hongseok Namkoong, Riccardo Volpi, and John Duchi.
\newblock Certifying some distributional robustness with principled adversarial
  training.
\newblock \emph{arXiv preprint arXiv:1710.10571}, 2017.

\bibitem[Bertsimas et~al.(2021)Bertsimas, Boix, Carballo, and
  Hertog]{bertsimas2021robust}
Dimitris Bertsimas, Xavier Boix, Kimberly~Villalobos Carballo, and Dick~den
  Hertog.
\newblock A robust optimization approach to deep learning.
\newblock \emph{arXiv preprint arXiv:2112.09279}, 2021.

\bibitem[Xu et~al.(2009)Xu, Caramanis, and Mannor]{xu2009robustness}
Huan Xu, Constantine Caramanis, and Shie Mannor.
\newblock Robustness and regularization of support vector machines.
\newblock \emph{Journal of machine learning research}, 10\penalty0 (7), 2009.

\bibitem[Ambrosio(2003)]{ambrosio2003lecture}
Luigi Ambrosio.
\newblock Lecture notes on optimal transport problems.
\newblock In \emph{Mathematical aspects of evolving interfaces}, pages 1--52.
  Springer, 2003.

\bibitem[Gao and Kleywegt(2023)]{gao2016distributionally}
Rui Gao and Anton Kleywegt.
\newblock Distributionally robust stochastic optimization with {Wasserstein}
  distance.
\newblock \emph{Mathematics of Operations Research}, 48\penalty0 (2):\penalty0
  603--655, 2023.

\bibitem[Wang et~al.(2022)Wang, Becker, Van~Parys, and Stellato]{wang2022mean}
Irina Wang, Cole Becker, Bart Van~Parys, and Bartolomeo Stellato.
\newblock Mean robust optimization.
\newblock \emph{arXiv preprint arXiv:2207.10820}, 2022.

\bibitem[Mohajerin~Esfahani and Kuhn(2018)]{esfahani2015data}
Peyman Mohajerin~Esfahani and Daniel Kuhn.
\newblock Data-driven distributionally robust optimization using the
  wasserstein metric: performance guarantees and tractable reformulations.
\newblock \emph{Mathematical Programming}, 171\penalty0 (1-2):\penalty0
  115--166, 2018.

\bibitem[Rice et~al.(2020)Rice, Wong, and Kolter]{rice2020overfitting}
Leslie Rice, Eric Wong, and Zico Kolter.
\newblock Overfitting in adversarially robust deep learning.
\newblock In \emph{International Conference on Machine Learning}, pages
  8093--8104. PMLR, 2020.

\bibitem[Tukey(1958)]{tukey1958bias}
John Tukey.
\newblock Bias and confidence in not quite large samples.
\newblock \emph{Ann. Math. Statist.}, 29:\penalty0 614, 1958.

\bibitem[Huber(1981)]{huber1981robust}
Peter~J. Huber.
\newblock \emph{Robust Statistics}.
\newblock Wiley Series in Probability and Mathematical Sciences. Wiley, 1981.

\bibitem[Diakonikolas et~al.(2019)Diakonikolas, Kamath, Kane, Li, Moitra, and
  Stewart]{diakonikolas2019robust}
Ilias Diakonikolas, Gautam Kamath, Daniel Kane, Jerry Li, Ankur Moitra, and
  Alistair Stewart.
\newblock Robust estimators in high-dimensions without the computational
  intractability.
\newblock \emph{SIAM Journal on Computing}, 48\penalty0 (2):\penalty0 742--864,
  2019.

\bibitem[Diakonikolas and Kane(2019)]{diakonikolas2019recent}
Ilias Diakonikolas and Daniel~M Kane.
\newblock Recent advances in algorithmic high-dimensional robust statistics.
\newblock \emph{arXiv preprint arXiv:1911.05911}, 2019.

\bibitem[Goldblum et~al.(2022)Goldblum, Tsipras, Xie, Chen, Schwarzschild,
  Song, Madry, Li, and Goldstein]{goldblum2022dataset}
Micah Goldblum, Dimitris Tsipras, Chulin Xie, Xinyun Chen, Avi Schwarzschild,
  Dawn Song, Aleksander Madry, Bo~Li, and Tom Goldstein.
\newblock Dataset security for machine learning: Data poisoning, backdoor
  attacks, and defenses.
\newblock \emph{IEEE Transactions on Pattern Analysis and Machine
  Intelligence}, 45\penalty0 (2):\penalty0 1563--1580, 2022.

\bibitem[Zhu et~al.(2019)Zhu, Huang, Shafahi, Li, Taylor, Studer, and
  Goldstein]{zhu_transferable_2019}
Chen Zhu, W.~Ronny Huang, Ali Shafahi, Hengduo Li, Gavin Taylor, Christoph
  Studer, and Tom Goldstein.
\newblock Transferable {Clean}-{Label} {Poisoning} {Attacks} on {Deep} {Neural}
  {Nets}, May 2019.
\newblock URL \url{http://arxiv.org/abs/1905.05897}.
\newblock arXiv:1905.05897 [stat].

\bibitem[Prokhorov(1956)]{prokhorov1956convergence}
Yu~V. Prokhorov.
\newblock Convergence of random processes and limit theorems in probability
  theory.
\newblock \emph{Theory of Probability \& Its Applications}, 1\penalty0
  (2):\penalty0 157--214, 1956.

\bibitem[Erdo{\u{g}}an and Iyengar(2006)]{erdougan2006ambiguous}
Emre Erdo{\u{g}}an and Garud Iyengar.
\newblock Ambiguous chance constrained problems and robust optimization.
\newblock \emph{Mathematical Programming}, 107\penalty0 (1):\penalty0 37--61,
  2006.

\bibitem[Dembo and Zeitouni(2009)]{dembod1996large}
Amir Dembo and Ofer Zeitouni.
\newblock \emph{Large deviations techniques and applications}, volume~38.
\newblock Springer Science \& Business Media, 2009.

\bibitem[Wang et~al.(2021)Wang, Gao, and Xie]{wang2021sinkhorn}
Jie Wang, Rui Gao, and Yao Xie.
\newblock Sinkhorn distributionally robust optimization.
\newblock \emph{arXiv preprint arXiv:2109.11926}, 2021.

\bibitem[Nietert et~al.(2023)Nietert, Goldfeld, and
  Shafiee]{nietert2023OutlierRobustWassersteinDRO}
Sloan Nietert, Ziv Goldfeld, and Soroosh Shafiee.
\newblock Outlier-{{Robust Wasserstein DRO}}, November 2023.

\bibitem[Chan et~al.(2024{\natexlab{a}})Chan, Van~Parys, and
  Bennouna]{chan2024DistributionalRobustnessRobust}
Gabriel Chan, Bart Van~Parys, and Amine Bennouna.
\newblock From {{Distributional Robustness}} to {{Robust Statistics}}: {{A
  Confidence Sets Perspective}}, October 2024{\natexlab{a}}.

\bibitem[Strassen(1965)]{strassen1965existence}
Volker Strassen.
\newblock The existence of probability measures with given marginals.
\newblock \emph{The Annals of Mathematical Statistics}, 36\penalty0
  (2):\penalty0 423--439, 1965.

\bibitem[Givens and Shortt(1984)]{givens1984class}
Clark~R. Givens and Rae~M. Shortt.
\newblock A class of wasserstein metrics for probability distributions.
\newblock \emph{Michigan Mathematical Journal}, 31\penalty0 (2):\penalty0
  231--240, 1984.

\bibitem[Bertsimas et~al.(2022)Bertsimas, Shtern, and Sturt]{bertsimas2022two}
Dimitris Bertsimas, Shimrit Shtern, and Bradley Sturt.
\newblock Two-stage sample robust optimization.
\newblock \emph{Operations Research}, 70\penalty0 (1):\penalty0 624--640, 2022.

\bibitem[Xie(2019)]{xie2019tractable}
Weijun Xie.
\newblock Tractable reformulations of distributionally robust two-stage
  stochastic programs with $\infty-$ wasserstein distance.
\newblock \emph{arXiv preprint arXiv:1908.08454}, 2019.

\bibitem[Nguyen et~al.(2020)Nguyen, Zhang, Blanchet, Delage, and
  Ye]{nguyen2020distributionally}
Viet~Anh Nguyen, Fan Zhang, Jose Blanchet, Erick Delage, and Yinyu Ye.
\newblock Distributionally robust local non-parametric conditional estimation.
\newblock \emph{Advances in Neural Information Processing Systems},
  33:\penalty0 15232--15242, 2020.

\bibitem[Rahimian et~al.(2019)Rahimian, Bayraksan, and Homem-de
  Mello]{rahimian2019identifying}
Hamed Rahimian, G{\"u}zin Bayraksan, and Tito Homem-de Mello.
\newblock Identifying effective scenarios in distributionally robust stochastic
  programs with total variation distance.
\newblock \emph{Mathematical Programming}, 173\penalty0 (1-2):\penalty0
  393--430, 2019.

\bibitem[Shapiro(2017)]{shapiro2017distributionally}
Alexander Shapiro.
\newblock Distributionally robust stochastic programming.
\newblock \emph{SIAM Journal on Optimization}, 27\penalty0 (4):\penalty0
  2258--2275, 2017.

\bibitem[Jiang and Guan(2018)]{jiang2018risk}
Ruiwei Jiang and Yongpei Guan.
\newblock Risk-averse two-stage stochastic program with distributional
  ambiguity.
\newblock \emph{Operations Research}, 66\penalty0 (5):\penalty0 1390--1405,
  2018.

\bibitem[Liu et~al.(2023)Liu, Van~Parys, and Lam]{liu2023smoothed}
Zhenyuan Liu, Bart~PG Van~Parys, and Henry Lam.
\newblock Smoothed $ f $-divergence distributionally robust optimization:
  Exponential rate efficiency and complexity-free calibration.
\newblock \emph{arXiv preprint arXiv:2306.14041}, 2023.

\bibitem[Dahl and Andersen(2021)]{dahl2021primal}
Joachim Dahl and Erling~D. Andersen.
\newblock A primal-dual interior-point algorithm for nonsymmetric
  exponential-cone optimization.
\newblock \emph{Mathematical Programming}, pages 1--30, 2021.

\bibitem[Ben-David et~al.(1994)Ben-David, Borodin, Karp, Tardos, and
  Wigderson]{ben1994power}
Shai Ben-David, Allan Borodin, Richard Karp, Gabor Tardos, and Avi Wigderson.
\newblock On the power of randomization in on-line algorithms.
\newblock \emph{Algorithmica}, 11:\penalty0 2--14, 1994.

\bibitem[Zhu et~al.(2022)Zhu, Jiao, and Steinhardt]{zhu2022generalized}
Banghua Zhu, Jiantao Jiao, and Jacob Steinhardt.
\newblock Generalized resilience and robust statistics.
\newblock \emph{The Annals of Statistics}, 50\penalty0 (4):\penalty0
  2256--2283, 2022.

\bibitem[Blanc et~al.(2022)Blanc, Lange, Malik, and Tan]{blanc2022power}
Guy Blanc, Jane Lange, Ali Malik, and Li-Yang Tan.
\newblock On the power of adaptivity in statistical adversaries.
\newblock In \emph{Conference on Learning Theory}, pages 5030--5061. PMLR,
  2022.

\bibitem[Cortes and Vapnik(1995)]{cortes1995support}
Corinna Cortes and Vladimir Vapnik.
\newblock Support-vector networks.
\newblock \emph{Machine learning}, 20\penalty0 (3):\penalty0 273--297, 1995.

\bibitem[Sch{\"o}lkopf et~al.(2000)Sch{\"o}lkopf, Smola, Williamson, and
  Bartlett]{scholkopf2000new}
Bernhard Sch{\"o}lkopf, Alex~J. Smola, Robert~C. Williamson, and Peter~L.
  Bartlett.
\newblock New support vector algorithms.
\newblock \emph{Neural computation}, 12\penalty0 (5):\penalty0 1207--1245,
  2000.

\bibitem[Wang et~al.(2016)Wang, Glynn, and Ye]{wang2016likelihood}
Zizhuo Wang, Peter~W Glynn, and Yinyu Ye.
\newblock Likelihood robust optimization for data-driven problems.
\newblock \emph{Computational Management Science}, 13\penalty0 (2):\penalty0
  241--261, 2016.

\bibitem[Nickparvar(2021)]{msoud_nickparvar_2021}
Msoud Nickparvar.
\newblock Brain tumor mri dataset, 2021.
\newblock URL \url{https://www.kaggle.com/dsv/2645886}.

\bibitem[Guo et~al.(2017)Guo, Pleiss, Sun, and Weinberger]{pmlr-v70-guo17a}
Chuan Guo, Geoff Pleiss, Yu~Sun, and Kilian~Q. Weinberger.
\newblock On calibration of modern neural networks.
\newblock In Doina Precup and Yee~Whye Teh, editors, \emph{Proceedings of the
  34th International Conference on Machine Learning}, volume~70 of
  \emph{Proceedings of Machine Learning Research}, pages 1321--1330. PMLR,
  06--11 Aug 2017.
\newblock URL \url{https://proceedings.mlr.press/v70/guo17a.html}.

\bibitem[Rockafellar et~al.(2000)Rockafellar, Uryasev,
  et~al.]{rockafellar2000optimization}
R~Tyrrell Rockafellar, Stanislav Uryasev, et~al.
\newblock Optimization of conditional value-at-risk.
\newblock \emph{Journal of risk}, 2:\penalty0 21--42, 2000.

\bibitem[Markowitz(1952)]{markowitz1952portfolio}
Harry Markowitz.
\newblock Portfolio selection.
\newblock \emph{The Journal of Finance}, 7:\penalty0 77--91, 1952.

\bibitem[Bennouna et~al.(2023)Bennouna, Lucas, and
  Van~Parys]{pmlr-v202-bennouna23a}
M.~Amine Bennouna, R.~Lucas, and Bart~P.G. Van~Parys.
\newblock Certified robust neural networks: Generalization and corruption
  resistance.
\newblock International Conference on Machine Learning (ICML), pages
  2092--2112, Honolulu, Hawaii, USA, 2023.
\newblock URL \url{https://proceedings.mlr.press/v202/bennouna23a.html}.

\bibitem[Van~Parys(2024)]{van2024efficient}
Bart~PG Van~Parys.
\newblock Efficient data-driven optimization with noisy data.
\newblock \emph{Operations Research Letters}, 54:\penalty0 107089, 2024.

\bibitem[Farokhi(2023)]{farokhi2023distributionally}
Farhad Farokhi.
\newblock Distributionally robust optimization with noisy data for discrete
  uncertainties using total variation distance.
\newblock \emph{IEEE Control Systems Letters}, 7:\penalty0 1494--1499, 2023.

\bibitem[Chan et~al.(2024{\natexlab{b}})Chan, Van~Parys, and
  Bennouna]{chan2024distributional}
Gabriel Chan, Bart Van~Parys, and Amine Bennouna.
\newblock From distributional robustness to robust statistics: A confidence
  sets perspective.
\newblock \emph{arXiv preprint arXiv:2410.14008}, 2024{\natexlab{b}}.

\bibitem[Eichhorn and R{\"o}misch(2005)]{eichhorn2005polyhedral}
Andreas Eichhorn and Werner R{\"o}misch.
\newblock Polyhedral risk measures in stochastic programming.
\newblock \emph{SIAM Journal on Optimization}, 16\penalty0 (1):\penalty0
  69--95, 2005.

\bibitem[Cover and Thomas(1991)]{cover1991information}
Thomas~M. Cover and Joy~A. Thomas.
\newblock Information theory and statistics.
\newblock \emph{Elements of information theory}, 1\penalty0 (1):\penalty0
  279--335, 1991.

\bibitem[Bertsekas(2009)]{bertsekas2009convex}
Dimitri Bertsekas.
\newblock \emph{Convex optimization theory}, volume~1.
\newblock Athena Scientific, 2009.

\bibitem[Sion(1958)]{sion1958general}
Maurice Sion.
\newblock On general minimax theorems.
\newblock \emph{Pacific Journal of mathematics}, 8\penalty0 (1):\penalty0
  171--176, 1958.

\bibitem[Csisz{\'a}r(1998)]{csiszar1998method}
Imre Csisz{\'a}r.
\newblock The method of types [information theory].
\newblock \emph{IEEE Transactions on Information Theory}, 44\penalty0
  (6):\penalty0 2505--2523, 1998.

\bibitem[Agrawal(2020)]{agrawal2020finite}
Rohit Agrawal.
\newblock Finite-sample concentration of the multinomial in relative entropy.
\newblock \emph{IEEE Transactions on Information Theory}, 66\penalty0
  (10):\penalty0 6297--6302, 2020.

\bibitem[Rockafellar(2015)]{rockafellar2015convex}
Ralph~T. Rockafellar.
\newblock \emph{Convex analysis}.
\newblock Princeton university press, 2015.

\bibitem[Berge(1997)]{berge1997topological}
Claude Berge.
\newblock \emph{Topological Spaces: including a treatment of multi-valued
  functions, vector spaces, and convexity}.
\newblock Courier Corporation, 1997.

\bibitem[Namkoong and Duchi(2017)]{namkoong2017variance}
Hongseok Namkoong and John~C. Duchi.
\newblock Variance-based regularization with convex objectives.
\newblock \emph{Advances in neural information processing systems}, 30, 2017.

\bibitem[Bertsimas et~al.(2019)Bertsimas, Dunn, Pawlowski, and
  Zhuo]{bertsimas2019robust}
Dimitris Bertsimas, Jack Dunn, Colin Pawlowski, and Ying~Daisy Zhuo.
\newblock Robust classification.
\newblock \emph{INFORMS Journal on Optimization}, 1\penalty0 (1):\penalty0
  2--34, 2019.

\bibitem[Chawla et~al.(2002)Chawla, Bowyer, Hall, and
  Kegelmeyer]{chawla2002smote}
Nitesh~V. Chawla, Kevin~W. Bowyer, Lawrence~O. Hall, and W.~Philip Kegelmeyer.
\newblock Smote: synthetic minority over-sampling technique.
\newblock \emph{Journal of artificial intelligence research}, 16:\penalty0
  321--357, 2002.

\bibitem[Owen(2001)]{owen2001empirical}
Art~B. Owen.
\newblock \emph{Empirical likelihood}.
\newblock Chapman and Hall/CRC, 2001.

\bibitem[Akansu et~al.(2016)Akansu, Kulkarni, and
  Malioutov]{akansu2016financial}
Ali~N. Akansu, Sanjeev~R. Kulkarni, and Dmitry~M. Malioutov.
\newblock \emph{Financial signal processing and machine learning}.
\newblock John Wiley \& Sons, 2016.

\end{thebibliography}
